\documentclass{article}

\usepackage[final]{neurips_2023}

\usepackage[utf8]{inputenc} 
\usepackage[T1]{fontenc}    
\usepackage{hyperref}       
\usepackage{url}            
\usepackage{booktabs}       
\usepackage{amsfonts}       
\usepackage{nicefrac}       
\usepackage{microtype}      
\usepackage{xcolor}         

\usepackage[utf8]{inputenc} 
\usepackage[T1]{fontenc} 
\usepackage{hyperref} 
\usepackage[ruled]{algorithm2e}

\usepackage{amsmath}
\usepackage{amsthm}
\usepackage{amssymb}
\usepackage{bbm}
\usepackage{paralist}
\usepackage{enumitem}
\usepackage{xcolor}
\usepackage{mathtools}
\usepackage{dsfont}
\usepackage{subcaption}
\usepackage{yfonts}

\title{Best Arm Identification with Fixed Budget: 
	\\A Large Deviation Perspective}
\author{
  Po-An Wang\\
  EECS\\
  KTH, Stockholm, Sweden\\
  \texttt{wang9@kth.se} \\
  \And
   Ruo-Chun Tzeng\\
  EECS\\
  KTH, Stockholm, Sweden\\
  \texttt{rctzeng@kth.se} \\
  \AND
  Alexandre Proutiere\\
  EECS and Digital Futures\\
    KTH, Stockholm, Sweden\\
  \texttt{alepro@kth.se}
}
\usepackage{times}

\usepackage[utf8]{inputenc}
\usepackage{dsfont}
\usepackage{bbm}
\usepackage{caption,subcaption}
\usepackage{capt-of} 
\usepackage{tcolorbox} 

\tcbset{colframe=gray,colback=white,center title}

\newcommand{\skl}{\textnormal{kl}}
\newcommand{\E}{\mathbb{E}}
\newcommand{\alt}{\textnormal{Alt}}
\newcommand{\cl}{\textnormal{cl}}

\newcommand{\PP}{\mathbb{P}}
\newcommand{\QQ}{\mathbb{Q}}
\newcommand{\NN}{\mathbb{N}}
\newcommand{\RR}{\mathbb{R}}
\newcommand{\olog}{\overline{\log}}
\newcommand{\om}{\boldsymbol{\omega}}

\newcommand{\setmap}{\rightrightarrows}
\newcommand{\bm}{\boldsymbol{\mu}}
\newcommand{\bl}{\boldsymbol{\lambda}}
\newcommand{\bp}{\boldsymbol{\pi}}
\newcommand{\be}{\boldsymbol{\eta}}
\newcommand{\zz}{\boldsymbol{z}}
\newcommand{\yy}{\boldsymbol{y}}
\newcommand{\xx}{\boldsymbol{x}}

\newcommand{\bxi}{\bar{\xi}}
\newcommand{\bpsi}{\bar{\psi}}

\newtheorem{theorem}{Theorem}
\newtheorem{definition}{Definition}
\newtheorem{lemma}{Lemma}
 
\newtheorem{assumption}{Assumption} 
\newtheorem{proposition}{Proposition}
\newtheorem{corollary}{Corollary}

\newtheorem{conjecture}{Conjecture}

\DeclareMathOperator*\lowlim{\underline{lim}}
\DeclareMathOperator*\uplim{\overline{lim}}

\newcommand{\eproof}{\hfill $\Box$}

\makeatletter
\newtheorem*{rep@theorem}{\rep@title}
\newcommand{\newreptheorem}[2]{%
	\newenvironment{rep#1}[1]{%
		\def\rep@title{#2 \ref{##1}}%
		\begin{rep@theorem}}%
		{\end{rep@theorem}}}
\makeatother

\newreptheorem{theorem}{Theorem}
\newreptheorem{lemma}{Lemma}
\newreptheorem{proposition}{Proposition}
\newreptheorem{corollary}{Corollary}
\newreptheorem{claim}{Claim}
\newreptheorem{assumption}{Assumption}
\newreptheorem{conjecture}{Conjecture}

\newcommand{\sr}{\ensuremath{\text{\tt SR}}\xspace}
\newcommand{\cra}{\ensuremath{\text{\tt CR-A}}\xspace}
\newcommand{\crc}{\ensuremath{\text{\tt CR-C}}\xspace}
\newcommand{\sred}{\ensuremath{\text{\tt CR}}\xspace}
\newcommand{\sh}{\ensuremath{\text{\tt SH}}\xspace}  
\newcommand{\ucbe}{\ensuremath{\text{\tt UCB-E}}\xspace}    
\newcommand{\ugape}{\ensuremath{\text{\tt UGapE}}\xspace}    
\newcommand{\DOT}{\ensuremath{\text{\tt DOT}}\xspace}

\DeclareMathOperator*{\argmax}{argmax} 
\DeclareMathOperator*{\argmin}{argmin}

\begin{document}

\maketitle
\begin{abstract}
We consider the problem of identifying the best arm in stochastic Multi-Armed Bandits (MABs) using a fixed sampling budget. Characterizing the minimal instance-specific error probability for this problem constitutes one of the important remaining open problems in MABs. When arms are selected using a static sampling strategy, the error probability decays exponentially with the number of samples at a rate that can be explicitly derived via Large Deviation techniques. Analyzing the performance of algorithms with adaptive sampling strategies is however much more challenging. In this paper, we establish a connection between the Large Deviation Principle (LDP) satisfied by the empirical proportions of arm draws and that satisfied by the empirical arm rewards. This connection holds for any adaptive algorithm, and is leveraged (i) to improve error probability upper bounds of some existing algorithms, such as the celebrated \sr (Successive Rejects) algorithm \citep{audibert2010best}, and (ii) to devise and analyze new algorithms. In particular, we present \sred (Continuous Rejects), a truly adaptive algorithm that can reject arms in {\it any} round based on the observed empirical gaps between the rewards of various arms. Applying our Large Deviation results, we prove that \sred enjoys better performance guarantees than existing algorithms, including \sr. Extensive numerical experiments confirm this observation.   
\end{abstract}


\section{Introduction}
 
We study the problem of best-arm identification in stochastic bandits in the fixed budget setting. In this problem, abbreviated by BAI-FB, a learner faces $K$ distributions or arms $\nu_1,\ldots,\nu_K$ characterized by their unknown means $\bm=(\mu_1,\ldots,\mu_K)$ (we restrict our attention to distributions taken from a one-parameter exponential family). She sequentially pulls arms and observes samples of the corresponding distributions. More precisely, in round $t\ge 1$, she pulls an arm $A_t=k$ selected depending on previous observations and observes $X_{k}(t)$ a sample of a $\nu_k$-distributed random variable. $(X_k(t),t\ge 1, k\in [K])$ are assumed to be independent over rounds and arms. After $T$ arm draws, the learner returns $\hat{\imath}$, an estimate of the best arm $1(\bm) :=\arg\max_k\mu_k$. We assume that the best arm is unique, and denote by $\Lambda$ the set of parameters $\bm$ such that this assumption holds. The objective is to devise an adaptive sampling algorithm minimizing the error probability $\mathbb{P}_{\bm}[\hat{\imath}\neq 1(\bm)]$. This learning task is one of the most important problems in stochastic bandits, and despite recent research efforts, it remains largely open \citep{qin2022open}. In particular, researchers have so far failed at characterizing the minimal instance-specific error probability. This contrasts with other basic learning tasks in stochastic bandits such as regret minimization \citep{lai1985asymptotically} and BAI with fixed confidence \citep{garivier2016optimal}, for which indeed, asymptotic instance-specific  performance limits and matching algorithms have been derived. In BAI-FB, the error probability typically decreases exponentially with the sample budget $T$, i.e., it scales as $\exp(-R(\bm)T)$ where the instance-specific rate $R(\bm)$ depends on the sampling algorithm. Maximizing this rate over the set of adaptive algorithms is an open problem. 

\medskip
\noindent
{\bf Instance-specific error probability lower bound.} To guess the maximal rate at which the error probability decays, one may apply the same strategy as that used in regret minimization or BAI in the fixed confidence setting: (i) derive instance-specific lower bound for the error probability for some notion of {\it uniformly good} algorithms; (ii) devise a sampling strategy mimicking the optimal proportions of arm draws identified in the lower bound. Here the notion of uniformly good algorithms is that of {\it consistent} algorithms. Under such an algorithm, for any $\bm\in \Lambda$, $\PP_{\bm}\left[\hat{\imath}= 1(\bm) \right] \rightarrow 1$ as $T\rightarrow \infty$. \citep{garivier2016optimal} conjectures the following asymptotic lower bound satisfied by any consistent algorithm (refer to Appendix \ref{app:discussion} for details): as $T\to\infty$,
\begin{equation}\label{eq:lb}
\frac{1}{T} \log \frac{1}{\PP_{\bm}\left[ \hat{\imath} \neq 1(\bm) \right] } \le  \max_{\om \in \Sigma } \inf_{ \bl\in \alt (\bm) }\Psi(\bl,\om) ,
\end{equation}
where $\Sigma$ is the $(K-1)$-dimensional simplex, $\Psi(\bl,\om)= \sum_{k=1}^K\omega_k d(\lambda_k ,\mu_k)$, $\alt (\bm)=\{\bl\in\Lambda: 1(\bm)\neq 1(\bl)\}$ is the set of confusing parameters (those for which $1(\bm)$ is not the best arm), and $d(x,y)$ denotes the KL divergence between two distributions of parameters $x$ and $y$. Interestingly, the solution $\om^\star\in \Sigma$ of the optimization problem $\max_{\om \in \Sigma }\inf_{ \bl\in \alt (\bm) } \Psi(\bl,\om)$ provides the best static proportions of arm draws. More precisely, an algorithm selecting arms according to the allocation $\om^\star$, i.e., selecting arm $k$ $\omega_k^\star T$ times and returning the best empirical arm after $T$ samples, has an error rate matching the lower bound (\ref{eq:lb}). This is a direct consequence of the fact that, under a static algorithm with allocation $\om$, the empirical reward process $\{\hat{\bm}(t)\}_{t\ge 1}$ satisfies a LDP with rate function $\bl\mapsto \Psi(\bl,\om)$, see \citep{glynn2004large} and refer to Section \ref{sec:ldp} for more details. 

\medskip
\noindent
{\bf Adaptive sampling algorithms and their analysis.} The optimal allocation $\om^\star$ depends on the instance $\bm$ and is initially unknown. We may devise an adaptive sampling algorithm that (i) estimates $\om^\star$ and (ii) tracks this estimated optimal allocation. In the BAI with fixed confidence, such tracking scheme exhibits asymptotically optimal performance \citep{garivier2016optimal}. Here however, the error made estimating $\om^\star$ would inevitably impact the overall error probability of the algorithm. To quantify this impact or more generally to analyze the performance of adaptive algorithms, one would need to understand the connection between the statistical properties of the arm selection process and the asymptotic statistics of the estimated expected rewards. 

To be more specific, any adaptive algorithm generates a stochastic process $\{Z(t)\}_{t\ge 1}=\{(\om(t), \hat{\bm}(t))\}_{t\ge 1}$. $\om(t)=(\omega_1(t),\ldots,\omega_K(t))$ represents the allocation realized by the algorithm up to round $t$ ($\omega_k(t)=N_k(t)/t$ and $N_k(t)$ denotes the number of times arm $k$ has been selected up to round $t$). $\hat{\bm}(t)=(\hat\mu_1(t),\ldots,\hat\mu_K(t))$ denotes the empirical average rewards of the various arms up to round $t$. Now assuming that at the end of round $T$, the algorithm returns the arm with the highest empirical reward, the error probability is $\mathbb{P}_{\bm}[\hat{\imath}\neq 1(\bm)]=\mathbb{P}_{\bm}[\hat{\bm}(T)\in \alt (\bm)]$. Assessing the error probability at least asymptotically requires understanding the asymptotic behavior of $\hat{\bm}(t)$ as $t$ grows large. Ideally, one would wish to establish the Large Deviation properties of the process $\{Z(t)\}_{t\ge 1}$. This task is easy for algorithms using static allocations \citep{glynn2004large}, but becomes challenging and open for adaptive algorithms. Addressing this challenge is the main objective of this paper. 

\noindent
{\bf Contributions.} In this paper, we develop and leverage tools towards the analysis of adaptive sampling algorithms for the BAI-FB problem. More precisely, our contributions are as follows.

\noindent
(a) We establish a connection between the LDP satisfied by the empirical proportions of arm draws $\{\om(t)\}_{t\ge 1}$ and that satisfied by the empirical arm rewards. This connection holds for any adaptive algorithm. Specifically, we show that if the rate function of $\{\om(t)\}_{t\ge 1}$ is lower bounded by $\om\mapsto I(\om)$, then that of $(\hat{\bm}(t))_{t\ge 1}$ is also lower bounded by $\bl\mapsto \min_{\om\in \Sigma}\max\{\Psi(\bl,\om),I(\om)\}$. This result has interesting interpretations and implies the following asymptotic upper bound on the error probability of the algorithm considered: as $T\to\infty$,
\begin{equation}\label{eq:up}
\frac{1}{T} \log \frac{1}{\PP_{\bm}\left[ \hat{\imath} \neq 1(\bm) \right] } \ge  \inf_{\om \in \Sigma, \bl\in \alt (\bm)} \max\{\Psi(\bl,\om),I(\om)\}.
\end{equation}
The above formula, when compared to the lower bound (\ref{eq:lb}), quantifies the price of not knowing $\om^\star$ initially, and relates the error probability to the asymptotic statistics of the sampling process used by the algorithm. 

\noindent
(b) We show that by simply applying our generic Large Deviation result, we may improve the error probability upper bounds of some existing algorithms, such as the celebrated \sr algorithm \citep{audibert2010best}. Our result further opens up opportunities to devise and analyze new algorithms with a higher level of adaptiveness. In particular, we present \sred (Continuous Rejects), an algorithm that, unlike \sr, can eliminate arms in {\it each} round. This sequential elimination process is performed by comparing the empirical rewards of the various candidate arms using continuously updated thresholds. Leveraging the LDP tools developed in (a), we establish that \sred enjoys better performance guarantees than \sr. Hence \sred becomes the algorithm with the lowest instance-specific and guaranteed error probability. We illustrate our results via numerical experiments, and compare \sred to other BAI algorithms.

\section{Related Work}\label{sec:related}

We distinguish two main classes of algorithms to solve the best arm identification problem in the fixed budget setting. Algorithms from the first class, e.g. Successive Rejects (\sr) \citep{audibert2010best} and Sequential Halving (\sh) \citep{karnin2013almost}, split the sampling budget into phases of fixed durations, and discard arms at the end of each phase. Algorithms from the second class, e.g. \ucbe \citep{audibert2010best} and \ugape \citep{gabillon2012best} sequentially sample arms based on confidence bounds of their empirical rewards. It is worth mentioning that algorithms from the second class usually require some prior knowledge about the problem, for example, an upper bound of $H= \sum_{k\neq 1(\bm) } \frac{1}{(\mu_{1(\bm)}-\mu_k)^2}$. Without this knowledge, the parameters can be chosen in a heuristic way, but the performance gets worse.

Algorithms from the first class exhibit better performance numerically and are also those with the best instance-specific error probability guarantees. \sr had actually the best performance guarantees so far: for example, when the reward distributions are supported on $[0,1]$, the error probability of \sr satisfies: $\liminf_{T\rightarrow \infty} \frac{1}{T}\log \frac{1}{\PP_{\bm}\left[\hat{ \imath}\neq 1(\bm)\right]} \ge \frac{1}{H_2\log K}$, where $H_2= \max_{k\neq 1(\bm)} \frac{k}{(\mu_{1(\bm)}-\mu_k)^2}$. In this paper, we strictly improve this guarantee (see Section \ref{sec:sr}). Recently, \citep{barrier2022best} also refined and extended the analysis of \citep{audibert2010best} by replacing, in the analysis, Hoeffding's inequality by a large deviation result involving KL-divergences. Again here, we further improve this new guarantee. We also devise \sred, an algorithm with error probability provably lower than our improved guarantees for \sr.

Fundamental limits on the error probability have also been investigated. In the minimax setting, \citep{carpentier2016tight} established that for any algorithm, there exists a problem within the class of instances with given complexity $H$ such that the error probability is greater than $\exp(-\frac{400 T}{H\log K}) \ge \exp(-\frac{400 T}{H_2\log K}) $. Up to a universal constant (here 400), \sr is hence minimax optimal. This lower bound was also recently revisited in \citep{ariu2021,degenne2023,wang2023uniformly} to prove that the instance-specific lower bound (\ref{eq:lb}) cannot be achieved on all instances by a single algorithm 
. Deriving tight instance-specific lower bounds remains open \citep{qin2022open}. 

We conclude this section by mentioning two interesting algorithms. In \citep{komiyama2022minimax}, the authors propose \DOT, an algorithm trying to match minimax error probability lower bounds. To this aim, the algorithm requires to periodically call an oracle able to determine an optimal allocation, solution of an optimization problem with high and unknown complexity. \DOT has minimax guarantees but is computationally challenging if not infeasible (numerically, the authors cannot go beyond simple instances with 3 arms). Finally, researchers have also looked at the best arm identification problem from a Bayesian perspective. For example, \citep{russo2016} devise variants of the celebrated Thompson Sampling algorithm, that could potentially work well in practice. Nevertheless, as discussed in \citep{komiyama2022bayes}, Bayesian algorithms cannot be analyzed nor provably perform well in the frequentist setting.

\section{Large Deviation Analysis of Adaptive Sampling Algorithms}\label{sec:ldp}

In this section, we first recall key concepts in Large Deviations (refer to the classical textbooks \citep{budhiraja2019analysis, dembo2009large,dupuis2011weak,varadhan2016large} for a more detailed exposition). We then apply these concepts to the performance analysis of adaptive sampling algorithms. Finally, we exemplify the analysis and apply it to improve existing performance guarantees for the \sr algorithm \citep{audibert2010best}. 

\subsection{Large Deviation Principles}

Consider the stochastic process $\{Y(t)\}_{t\ge 1}$ with values in a separable complete metric space (i.e., a Polish space) $\mathcal{Y}$. Large Deviations are concerned with the probabilities of rare events related to $\{Y(t)\}_{t\ge 1}$ that decay exponentially in the parameter $t$. The asymptotic decay rate is characterized by the {\it rate function} $I:{\cal Y}\to \mathbb{R}_+$ defined so that essentially $-\frac{1}{t}\log \PP\left[Y(t)\in B\right]$ converges to $\min_{x\in B}I(x)$ for any Borel set $B$. We provide a more rigorous definition below.

\begin{definition}\label{def:rate}[Large Deviation Principle (LDP)]
The stochastic process  $\{Y(t)\}_{t\ge 1}$ satisfies a LDP with rate function $I$ if:\\
(i) $I$ is lower semicontinuous, and $\forall s\in [0,\infty]$, the set $\mathcal{K}_{s}=\{y\in \mathcal{Y}:I(y)\le s \}$ is compact;\\
(ii) for every closed (resp. open) set $C\subset \mathcal{Y}$ (resp. $O\subset \mathcal{Y}$),  
\begin{align}
\lowlim_{t\rightarrow \infty}\frac{1}{t}\log \frac{1}{\PP\left[ Y(t)\in C\right]} & \ge  \inf_{y\in C} I(y), \label{eq:rate UB}\\
\uplim_{t\rightarrow \infty}\frac{1}{t}\log \frac{1}{\PP\left[ Y(t)\in O\right]} & \le  \inf_{y\in O} I(y).\label{eq:rate LB}
\end{align}
\end{definition}

LDPs have been derived earlier in stochastic bandit literature. \citep{glynn2004large} have used G{\"a}rtner-Ellis Theorem \citep{ellis1984large,gartner1977large} to establish that under a static sampling algorithm with allocation $\om\in \Sigma$ (i.e., each arm $k$ is selected $\omega_k T$ times up to round $T$), the process $\{ \hat{\bm}(T)\}_{T \ge 1}$ satisfies an LDP with rate function $\bl\mapsto \Psi(\bl,\om)= \sum_{k=1}^K\omega_k d(\lambda_k ,\mu_k)$. Our objective in the next subsection is to investigate how to extend this result to the case of adaptive sampling algorithms.

\subsection{Analysis of adaptive sampling algorithms}\label{sec:varadaham}
An adaptive sampling algorithm generates a stochastic process $\{Z(t)\}_{t\ge 1}=\{(\om(t), \hat{\bm}(t))\}_{t\ge 1}$. When the sampling budget $T$ is exhausted, should the algorithm returns the arm with the highest empirical reward, the error probability is $\mathbb{P}_{\bm}[\hat{\imath}\neq 1(\bm)]=\mathbb{P}_{\bm}[\hat{\bm}(T)\in \alt (\bm)]$. To assess the rate at which this probability decays with the budget, we may try to establish a LDP for the empirical reward process $\{\hat{\bm}(t)\}_{t\ge 1}$. Due to the intricate dependence between the sampling and the empirical reward processes, deriving such an LDP is very challenging. Instead, we establish a connection between the LDPs satisfied by these processes. This connection will be enough for us to derive tight upper bounds on the error probability. We present our main result in the following theorem.

\begin{theorem}\label{thm:Varadaham bandits}
Assume that under some adaptive sampling algorithm, $\{ \om (t)\}_{t\ge 1}$ satisfies the LDP upper bound (\ref{eq:rate UB}) with rate function $I$. Then $\{\hat{\bm}(t)\}_{t\ge 1}$ satisfies the LDP upper bound (\ref{eq:rate UB}) with rate function $\bl\mapsto \min_{\om\in \Sigma}\max\{ \Psi(\bl,\om),I(\om)\}$. Moreover, we have: for any bounded Borel subset $\mathcal{S}$ of $\RR^K$ and any Borel subset $W$ of $\Sigma$,  
$$
\lowlim_{t\rightarrow \infty }   \frac{1}{t} \log \frac{1}{\PP_{\bm}\left[  \hat{\bm}(t)\in  \mathcal{S},\om(t)\in W \right] }\ge \inf_{\om\in \cl(W)} \max\left\{ F_{\mathcal{S}}(\om),I(\om)\right\},
$$
where $F_{ \mathcal{S}}(\om):=\inf_{\bl \in \cl (\mathcal{S}) }\Psi(\bl,\om)$, and $\cl (\mathcal{S})$ denotes the closure of ${\cal S}$.
\end{theorem}
\noindent
Before proving the above theorem, we make the following remarks and provide a simple corollary that will lead to improved upper bound on the error probability of the \sr algorithm.

\medskip
\noindent
{\it (a) Not a complete LDP.} To upper bound the error probability of a given algorithm, we do not actually need to establish that $\{\hat{\bm}(t)\}_{t\ge 1}$ satisfies a complete LDP. Instead, deriving a LDP upper bound is enough. Theorem \ref{thm:Varadaham bandits} provides such an upper bound, but does not yield a complete LDP. We conjecture if $\{ {\om}(t)\}_{t\ge 1}$ satisfies an LDP with rate function $I$, $\{\hat{\bm}(t)\}_{t\ge 1}$ satisfies an LDP with rate function $\bl\mapsto \inf_{\om\in W}\max\{ \Psi(\bl,\om),I(\om)\}$. The conjecture holds for static sampling algorithms as shown below.  If it holds for adaptive algorithms, we show, in Appendix \ref{app:conjecture}, that when $K>3$, no algorithm can attain the instance-specific lower bound (\ref{eq:lb}) for all parameters.

\noindent
{\it (b) Theorem \ref{thm:Varadaham bandits} is tight for static sampling algorithms.} When the sampling rule is static, namely $\om (t)=\om\in \Sigma$, then $\{\om (t)\}_{t\ge 1}$ satisfies a LDP with rate function $I$  defined as $I(\om)=0$ and $\infty$ elsewhere. Theorem \ref{thm:Varadaham bandits} with $W=\Sigma$ states that $\{\hat{\bm}(t)\}_{t\ge 1}$ satisfies the LDP upper bound (\ref{eq:rate UB}) with rate function $F_{\mathcal{S}}$. In fact, as shown by \citep{glynn2004large}, $\{\hat{\bm}(t)\}_{t\ge 1}$ satisfies a complete LDP with this rate function. 

\noindent
{\it (c) A useful corollary.} From Theorem \ref{thm:Varadaham bandits}, we have: 
$$
\lowlim_{t\rightarrow \infty }  \frac{1}{t} \log \frac{1}{\PP_{\bm}\left[  \hat{\bm}(t)\in \mathcal{S},\om (t)\in W \right] }\ge \inf_{\om\in \cl(W)} F_{\mathcal{S}}(\om).
$$ From there, we will be able to improve the performance guarantee for \sr.

\subsection{Proof of Theorem \ref{thm:Varadaham bandits}}
\begin{proof}
Observe that when ${\cal S}$ or $W$ is empty, the result holds. Now recall that $F_{\mathcal{S}}(\cdot)=\inf_{\bl\in \cl(\mathcal{S})}\Psi (\bl,\cdot)$ is the infimum of a family of linear functions on a compact set, $\Sigma$, hence it is upper bounded. Denote $u>0$ such an upper bound. $F_{\mathcal{S}}(\cdot)$ is also continuous (see Appendix \ref{app:continuity} for details). For each integer $N\in \NN$, we define a collection of closed sets:
    \begin{equation}\label{eq:W}
    W_{n}^N=\left\{ \om \in \cl (W): \frac{u(n-1)}{N}\le F_{\mathcal{S}}(\om )\le \frac{un}{N}\right\},\quad \forall n\in [N].     
    \end{equation}
We observe that:
    \begin{align*}
        \PP_{\bm}[\hat{\bm}(t)\in \mathcal{S},\om (t)\in W]&\le \sum_{n=1}^N \PP_{\bm}[\hat{\bm}(t)\in \mathcal{S},\om (t)\in W_{n}^N]\\
        &\le N \max_{n\in [N]}\PP_{\bm}[\hat{\bm}(t)\in \mathcal{S},\om (t)\in W_{n}^N].
    \end{align*}
Taking the logarithm on both sides and dividing them by $-t$ yields that
\begin{align}\label{neq:varadam1}
   \nonumber \lowlim_{t\rightarrow \infty} \frac{1}{t}\log \frac{1}{\PP_{\bm} \left[ \hat{\bm}(t)\in \mathcal{S},\om (t)\in W \right]  }&\ge 	\lowlim_{t\rightarrow  \infty} \min_{n \in [N]} \frac{1}{t}\log \frac{1}{\PP_{\bm} \left[ \hat{\bm}(t)\in \mathcal{S},\om (t)\in W_n^N\right]  }\\
 \nonumber    &=\min_{n\in [N]} \lowlim_{t\rightarrow \infty} \frac{1}{t}\log \frac{1}{\PP_{\bm} \left[ \hat{\bm}(t)\in \mathcal{S},\om (t)\in W_n^N\right]  }\\
    &\ge \min_{n\in [N]}\max\left\{\frac{u(n-1)}{N},\inf_{\om \in W_n^N }I(\om)\right\},
\end{align}
where the last inequality follows from Lemma \ref{lem:Wn}. Since for all $n\in [N]$, 
$$
\max\left\{\frac{u(n-1)}{N},\inf_{\om\in W_n^N}I(\om)\right\}=\inf_{\om\in W_n^N}\max\left\{\frac{u(n-1)}{N},I(\om)\right\},
$$
the r.h.s. of (\ref{neq:varadam1}) is equal to 
\begin{align*}
    \min_{n\in [N]}\inf_{\om \in W_n^N }\max\left\{\frac{u(n-1)}{N},I(\om)\right\}&\ge  \min_{n\in [N]}\inf_{\om \in W_n^N }\max\left\{F_{\mathcal{S}}(\om),I(\om)\right\}-u/N\\
    &=\inf_{\om \in \cl (W)} \max\{F_{\mathcal{S}}(\om),I(\om)\}-u/N,
\end{align*}
where the first inequality is due to (\ref{eq:W}). As $N$ can be taken arbitrarily large, we conclude this theorem.

\end{proof}
\medskip
\noindent
\begin{lemma}\label{lem:Wn}
For any $N\in \NN,\,n\in [N]$, 
$$
\lowlim_{t\to \infty} \frac{1}{t}\log \frac{1}{\PP_{\bm}[\hat{\bm}(t)\in \mathcal{S}, \om (t)\in W_{n}^N]} \ge \max\left\{ \frac{u(n-1)}{N}, \inf_{\om \in W_{n}^N} I(\om)\right\}
$$
\end{lemma}
\begin{proof}
Recall $F_{\mathcal{S}}(\cdot)=\inf_{\bl\in \cl(\mathcal{S}) } \Psi(\bl,\cdot)$. We deduce that
\begin{equation*}
\PP_{\bm}\left[  \hat{\bm}(t)\in \mathcal{S} ,\om (t)\in W_{n}^N \right]\allowbreak	\le \allowbreak \PP_{\bm}\left[  X  \ge tF_S(\om(t)), \om (t)\in W_{n}^N\right],
\end{equation*}
    where $X$ denotes $t\Psi(\hat{\bm}(t),\om (t))$ for short. Let $\alpha\in (0,1)$, Markov's inequality implies that
\begin{align}\label{neq:LDP2}
\nonumber	\PP_{\bm} \left[ X\ge tF_S(\om) , \om (t)\in W_{n}^N \right]&=\PP_{\bm}\left[ \mathbbm{1}\{ \om (t)\in W_{n}^N\} e^{\alpha (X-tF_S(\om(t)))}\ge 1 \right] \\
\nonumber &\le \E_{\bm} \left[\mathbbm{1}\{\om (t) \in W_{n}^N\} e^{\alpha (X-tF_S(\om(t) ))} \right]\\
&\le \E_{\bm}\left[\mathbbm{1}\{\om (t) \in W_{n}^N\} e^{\alpha X} \right] e^{-\frac{\alpha u(n-1)}{N} } ,
\end{align}
where the last inequality uses the definition of $W_n^N$ (see (\ref{eq:W})). By applying H{\"o}lder's inequality with $p,q$, where $p\in [1,1/\alpha)$ and $q=p/(p-1)$ on r.h.s. of (\ref{neq:LDP2}), we deduce that $\log \PP_{\bm}\left[  \hat{\bm}(t)\in  \mathcal{S} ,\om (t)\in W_{n}^N\right]	$ is at most
\[
 (\log \E_{\bm}\left[e^{\alpha p X}\right])/p+(\log \E_{\bm}[\mathbbm{1}\{\om (t)\in W_{n}^N\} ])/q -\frac{\alpha u(n-1)}{N}.
\]
As $\alpha p\in (0,1)$, Lemma \ref{lem:Xbd} in Appendix \ref{app:tec} shows that the first term above is $o(t)$. Using definition of the rate function, (\ref{eq:rate UB}) with $C = W_{n}^N$, on the second term yields that $\lowlim_{t\rightarrow \infty }   \frac{1}{t} \log \frac{1}{\PP_{\bm}\left[  \hat{\bm}(t)\in  \mathcal{S},\om (t)\in W_{n}^N  \right] }$ is lower bounded by
$$
(1/q)\inf_{\om\in W_{n}^N} I(\om) +\frac{\alpha u(n-1)}{N}= (1-1/p)\inf_{\om \in W_{n}^N}I(\om)+ \frac{\alpha u(n-1)}{N}.
$$
As $p$ can be arbitrarily close to $1/\alpha$, we get the lower bound $(1-\alpha)I(\om)+\frac{\alpha u(n-1)}{N}$. Further choosing $\alpha $ close to either $1$ or $0$, the proof is completed.

\end{proof}

\subsection{Improved analysis of the Successive Rejects algorithm}\label{sec:sr}

In \sr, the set of candidate arms is initialized as $\mathcal{C}_K=[K]$. The budget of samples is partitioned into $K-1$ phases, and at the end of each phase, \sr discards the empirical worst arm from the candidate set. In each phase, \sr uniformly samples the arms in candidate set. The lengths of phases are set as follows. Define $\olog K:=\frac{1}{2}+\sum_{k=2}^K\frac{1}{k}$. The candidate set is denoted by $\mathcal{C}_j$ when it has $j> 2$ arms. In the corresponding phase, (i) each arm in $\mathcal{C}_j$ is sampled until the round $t$ when $\min_{k\in \mathcal{C}_j}N_k(t)$ reaches $T/(j\olog K)$ (recall that $N_k(t)$ is the number of times arm $k$ has been sampled up to round $t$); (ii) the empirical worst arm, denoted by $\ell_j$, is then discarded, i.e., $\mathcal{C}_{j-1}=\mathcal{C}_j\setminus \{\ell_j\}$. During the last phase, the algorithm equally samples the two remaining arms and finally recommends $\hat{\imath}$, the arm with higher empirical mean in $\mathcal{C}_2$. The pseudo code is presented in Algorithm \ref{alg:SR}.
\begin{algorithm}
\SetKwProg{Init}{initialization}{ $\mathcal{C}_K\leftarrow [K], j\leftarrow K$;}{}
	\Init{}{}
	\For{$(t=1,\ldots, T)$}{
		\If{($j>2$ \textrm{and} $\min_{k\in \mathcal{C}_j} N_k(t)\ge \frac{T}{j\olog K}$)}{$\ell_j\leftarrow \argmin_{k\in \mathcal{C}_j} \hat{\mu}_k(t)$ (tie broken arbitrarily), $\mathcal{C}_{j-1}\leftarrow \mathcal{C}_j\setminus \{\ell_j\}$, and $j\leftarrow j-1$	;		
		}
		sample $A_t\leftarrow \argmin_{k\in \mathcal{C}_{j}} N_k(t)$ (tie broken arbitrarily), 
		update $\{N_k(t)\}_{k\in \mathcal{C}_j} $ and $\hat{\bm}(t)$;
	}
	$\ell_2\leftarrow \arg\min_{k\in \mathcal{C}_2}\hat{\mu}_k(T)$ and return $\hat{\imath}\leftarrow \arg\max_{k\in \mathcal{C}_2}\hat{\mu}_k(T)$ (tie broken arbitrarily).
	\captionof{algocf}{\sr}\label{alg:SR}
\end{algorithm}

We apply the corollary (c) in Section \ref{sec:varadaham} to improve the existing performance guarantees of \sr. To simplify the presentation, we assume wlog that $\mu_{1}>\mu_{2 }\ge \ldots\ge \mu_{K}$. For $j=2,\ldots,K$, define
\begin{equation}\label{eq:gamma}
\Gamma_j=\min_{J\in \mathcal{J}}\inf\left\{ \sum_{k\in J}d(\lambda_k,\mu_k):\bl\in \RR^K,\lambda_{1}\le \min_{k\in J} \lambda_k\right\},
\end{equation}
where $\mathcal{J}=\{J\subseteq [K]:\left|J\right|=j,1\in J\}$.

\begin{theorem}\label{thm:sr}
	Let $\bm\in \Lambda$. Under \sr, we have: for $j=2,\ldots,K$,
$\lowlim_{T\rightarrow \infty}\frac{1}{T}\log \frac{1}{\PP_{\bm}\left[ \ell_j=1\right]}\ge \frac{\Gamma_j}{j\olog K}$.\\
Hence, the error probability of \sr is upper bounded by $\lowlim_{T\rightarrow \infty}\frac{1}{T}\log \frac{1}{\PP_{\bm}\left[ \hat{\imath}\neq 1 \right]}\ge \min_{j\neq 1}  {\Gamma_j\over j\olog K}$.
\end{theorem}
\noindent
{\it Proof of Theorem \ref{thm:sr}.} 
	Fix $j\in \{2,\ldots,K\}$. Observe that 
	\begin{equation*}
	\PP_{\bm}\left[ \ell_j=1\right]=\sum_{J\in \mathcal{J}} \PP_{\bm}\left[ \ell_j=1,\mathcal{C}_j=J\right]\le \left|\mathcal{J}\right|\max_{J\in \mathcal{J}}\PP_{\bm}\left[ \ell_j=1,\mathcal{C}_j=J\right],
	\end{equation*}
	which implies that 
 \begin{equation}\label{eq:sr}
     \lowlim_{T\rightarrow \infty}\frac{1}{T}\log \frac{1}{\PP_{\bm}\left[ \ell_j=1\right]}\ge \min_{J\in \mathcal{J}}\lowlim_{T\rightarrow \infty}\frac{1}{T}\log \frac{1}{\PP_{\bm}\left[ \ell_j=1,\,\mathcal{C}_j=J\right]}
 \end{equation}
  as $\left|\mathcal{J}\right|<\infty$.\\
	\noindent
	Since $\ell_j$ is selected at the $\theta T$-th round\footnote{To simplify the presentation, we ignore cases where $\theta T$ is not an integer. Refer to Appendix \ref{app:allocation} for details.} where $\theta=(1+\sum_{k=j+1}^K\frac{1}{k})/\olog K$, the event $\{\ell_j=1,\,\mathcal{C}_j=J\}$ implies that  
	$\{\hat{\bm}(\theta T)\in \mathcal{S},\,\om (\theta T)\in W \}$, where 
	$$
	\mathcal{S}=\left\{\bl\in \RR^K: \lambda_{1}\le\lambda_k,\forall k\in J\right\}	 \hbox{ and } W=\left\{\om\in \Sigma:\omega_k=\frac{1}{\theta j\olog K},\forall k\in J \right\}.
	$$
	In other words, $\PP_{\bm}[\ell_j=1,\,\mathcal{C}_j=J]\le \PP_{\bm}[  \hat{\bm}(\theta T)\in \mathcal{S},\,\om(\theta T)\in W]$. Applying (c) in Section \ref{sec:varadaham} with the above $\mathcal{S}$ and $W$ yields that $\lowlim_{T\rightarrow \infty}\frac{1}{\theta T}\log \frac{1}{\PP_{\bm}\left[ \hat{\bm}(\theta T)\in \mathcal{S},\,\om(\theta T)\in W \right]}$ is larger than
	\begin{align}\label{eq:sr1}
	\inf_{\om \in W}\inf_{\bl\in \cl (\mathcal{S})}  \Psi(\bl,\om)  \ge \frac{1}{\theta j\olog K }\inf\left\{ \sum_{k\in J}d(\lambda_k,\mu_k):\bl\in\RR^K, \lambda_{1} \le\min_{k\in J}\lambda_k\right\} 	\ge  \frac{\Gamma_j}{\theta j\olog K},
	\end{align}
where the first inequality uses the fact that KL-divergences and the components of $\om$ are nonnegative, and the second one is due to the definition (\ref{eq:gamma}) of $\Gamma_j$. Combining (\ref{eq:sr}) and (\ref{eq:sr1}) completes the proof.\eproof

\medskip
\noindent
The upper bound derived in Theorem \ref{thm:sr} is tighter than those recently derived in \citep{barrier2022best}. Indeed, for any $J\in \mathcal{J}$, since $\left|J\right|=j$, one can find at least one index in $J$ at least larger than $j$, say $k_J$. Hence, 
$$
\Gamma_j\ge \min_{J\in \mathcal{J}}\inf_{\bl\in \RR^K,\lambda_1\le \lambda_{k_J}}  d(\lambda_1,\mu_1)+d(\lambda_{k_J},\mu_{k_J} )\ge  \inf_{\bl\in \RR^K,\lambda_1\le \lambda_{j}}  d(\lambda_1,\mu_1)+d(\lambda_{j},\mu_{j} ).
$$
The r.h.s. in the previous inequality corresponds to the upper bounds derived by \citep{barrier2022best}.

To simplify the presentation and avoid rather intricate computations involving the KL-divergences, in the remaining of the paper, we restrict our attention to specific classes of reward distributions.
\begin{assumption}\label{ass1}
The rewards are bounded with values in $(0,1)$. The reward distributions $\nu_1,\ldots,\nu_K$ are Bernoulli distributions such that $\nu_a$ is of mean $a$, and for any $a\neq b$, $d(a,b)\ge 2(a-b)^2$ (this is a consequence of Pinsker's inequality as rewards are in $(0,1)$).
\end{assumption}

Under Assumption \ref{ass1}, we have $\Gamma_j\ge 2\xi_j$ (a direct consequence of Proposition \ref{prop:larger C} in Appendix \ref{app:opt C}), where for $j=2,\ldots,K$,
$$
\xi_j=\inf\left\{ \sum_{k=1}^j(\lambda_k-\mu_k)^2:\bl\in [0,1]^j,\,\lambda_{1}\le \min_{k=1,\ldots,j} \lambda_k\right\}.
$$
We give an explicit expression of $\xi_j$
in Proposition \ref{prop:xi}, presented in Appendix \ref{app:opt C}. Moreover, $2\xi_j$ is clearly larger than 
$
2\inf_{\lambda_1\le \lambda_j}\left\{ (\lambda_1-\mu_1)^2+(\lambda_j-\mu_j)^2\right\}=(\mu_1-\mu_j)^2,
$
and hence $\min_{j\neq 1}  \Gamma_j/(j\olog K)\ge \min_{j\neq 1}  (\mu_1-\mu_j)^2/(j\olog K)$. This implies that our error probability upper bound is better than that derived in \citep{audibert2010best}.

\noindent
{\bf Example 1.} To illustrate the improvement brought by Theorem \ref{thm:sr} on the performance guarantees of \sr, consider the simple example with 3 Bernoulli arms and $\bm = (0.9,0.1,0.1)$. Then $\min_{j\neq 1}  (\mu_1-\mu_j)^2/(j\olog 3)=0.16$ for the upper bound presented in \citep{audibert2010best}. From Proposition \ref{prop:xi}, instead we get $ \min_{j\neq 1}2\xi_j/(j\olog 3)=0.21$.

\section{Continuous Rejects Algorithms}\label{sec:cr}

In this section, we present \sred, a truly adaptive algorithm that can discard an arm in {\it any} round. We propose two variants of the algorithm, \crc using a conservative criterion to discard arms and \cra discarding arms more aggressively. Using the Large Deviation results of Theorem \ref{thm:Varadaham bandits}, we establish error probability upper bounds for both \crc and \cra.


\subsection{The \crc and \cra algorithms}

As \sr, \sred initializes its candidate set as $\mathcal{C}_K=[K]$. For $j\ge 2$, ${\cal C}_j$ denotes the candidate set when it is reduced to $j$ arms. When $j>2$, the algorithm samples arms in the candidate set ${\cal C}_j$ uniformly until a {\it discarding condition} is met. The algorithm then discards the empirically worst arm $\ell_j\in {\cal C}_j$, i.e., $\mathcal{C}_{j-1}\leftarrow \mathcal{C}_j\setminus \{\ell_j\}$. More precisely, in round $t$, if there are $j$ candidate arms remaining and if $\ell(t)$ denotes the empirically worst candidate arm, the discarding condition is $N_{\ell(t)}(t)>\max_{k\notin \mathcal{C}_j}N_k(t),\, (\forall k\in \mathcal{C}_j,\,N_{\ell(t)}(t)=N_{k}(t))$\footnote{This condition is only imposed to simplify our analysis. It may be removed.}, and
\begin{align}
\hbox{for \crc:} \qquad &	\min_{k\in \mathcal{C}_j,k\neq \ell (t)}\hat{\mu}_k(t)-\hat{\mu}_{\ell(t)}(t)\ge G\left(\frac{\sum_{k\in \mathcal{C}_j}N_k(t)\olog j}{T-\sum_{k\notin \mathcal{C}_j}N_k(t) }\right), \label{C}\\
\hbox{for \cra:} \qquad & \frac{\sum_{k\in \mathcal{C}_j,k\neq \ell(t)}\hat{\mu}_k(t)}{j-1}-\hat{\mu}_{\ell(t)}(t)   \ge G\left(\frac{\sum_{k\in \mathcal{C}_j}N_k(t)\olog j}{T-\sum_{k\notin \mathcal{C}_j}N_k(t) }\right),\label{A}
\end{align} 
where $G(\beta)=1/\sqrt{\beta}-1$ for all $\beta> 0$. The idea behind (\ref{C}) is to keep the probability of discarding the best arm at most smaller than that of \sr while using less budget. Note that \eqref{A} is easier to achieve than \eqref{C}. \cra is hence more aggressive than \crc, and reduces the set of arms to $\mathcal{C}_2$ faster, but at the expense of a higher risk. After discarding $\ell_3$, \sred will sample the arms in $\mathcal{C}_2$ evenly, and recommend the empirical best arm in $\mathcal{C}_2$ when the budget is exhausted. The pseudo-code of \sred is presented in Algorithm \ref{alg:crc}.  

\begin{algorithm}
	\SetAlgoLined
{\bf Input:} $\theta_0\in (0,\frac{1}{\olog K})\cap \QQ$ independent of $T$ (can be chosen as small as one wishes)\\
 	\SetKwProg{Init}{initialization}{}{}
	\Init{}{
		$\mathcal{C}_K\leftarrow [K],\, j\leftarrow K$, sample each arm $k\in [K]$ once,  update $\{N_k(t)\}_{k\in \mathcal{C}_K} $ and $\hat{\bm}(t)$;
	}
        \For{$t=K+1,\ldots, \lfloor \theta_0 T\rfloor$}{
        sample $A_t\leftarrow \argmin_{k\in \mathcal{C}_{j}} N_k(t)$ {(tie broken arbitrarily)}, update $ \{N_k(t)\}_{k\in \mathcal{C}_j} $ and $\hat{\bm}(t)$;
        }
	\For{ $(t=\lfloor \theta_0T\rfloor+1,\ldots, T)$}{
		$\ell(t) \leftarrow  \argmin_{k\in \mathcal{C}_j}\hat{\mu}_k(t)$ (tie broken arbitrarily);\\
		\uIf{$j>2$, $N_{\ell(t)}(t)>\max_{k\notin\mathcal{C}_j}N_k(t)$, $(\forall k\in \mathcal{C}_j,\,N_{\ell(t)}(t)=N_{k}(t))$,\\
			 and \eqref{C} (resp. \eqref{A}) holds for \crc (resp. \cra)}{
			$\ell_j\leftarrow \ell(t),\,\mathcal{C}_{j-1}\leftarrow \mathcal{C}_j\setminus \{\ell_j\} ,\, j\leftarrow j-1$
		}
		sample $A_t\leftarrow \argmin_{k\in \mathcal{C}_{j}} N_k(t)$ {(tie broken arbitrarily)}, update $ \{N_k(t)\}_{k\in \mathcal{C}_j} $ and $\hat{\bm}(t)$;
	}
	$\ell_2\leftarrow \ell(T)$; return $\hat{\imath}\leftarrow \argmax_{k\in \mathcal{C}_2} \hat{\mu}_k(T) $ {(tie broken arbitrarily)}.
	\captionof{algocf}{\crc and \cra}\label{alg:crc}
\end{algorithm}

\subsection{Analysis of \crc and \cra}\label{sec:anal cr}
As in Section \ref{sec:sr}, $\mu_1>\mu_2\ge\ldots\ge \mu_K$ is assumed wlog and we further define $\mu_{K+1}=0$. We introduce the following instance-specific quantities needed to state our error probability upper bounds. For $ j\in \{2,\ldots, K\}$, define
$$
\psi_j=\frac{j-1}{j}\left(\mu_{1}-\frac{\sum_{k=2}^j\mu_{k}}{j-1}\right)^2,\quad \bpsi_j=\frac{j-1}{j}\left(\mu_1-\frac{\sum_{k=2}^{j-1}\mu_{k}+\mu_{j+1}}{j-1}\right)^2,\quad \zeta_j=\mu_{j}-\mu_{j+1},
$$
$$
\varphi_j=\frac{\sum_{k=1}^j\mu_{k}}{j}-\mu_{j+1},\quad \text{and} \quad \bxi_j=\inf \left\{   \sum_{k=1}^K(\lambda_k-\mu_k)^2: \bl\in [0,1]^K,\,   \lambda_1\le \min_{k=2,\ldots,j-1,j+1 }\lambda_k  \right\}.
$$
 Here we remark $\bxi_j\ge \xi_j$ and $\bpsi_j\ge \psi_j$. These inequalities are proven in Proposition \ref{prop:larger C} and Proposition \ref{prop:larger A} in Appendix \ref{app:opt}.

\begin{theorem}\label{thm:crc}
	Let $\bm\in  [0,1]^K$. Under \crc, $\lowlim_{T\rightarrow\infty}\frac{1}{T}\log \frac{1}{\PP_{\bm}\left[\hat{\imath}\neq 1\right]}$ is larger than
		$$
2\min_{j=2,\ldots,K}\left\{ \frac{\min\left\{\max \left\{ \frac{\xi_j\olog (j+1)(1-\alpha_{j})\mathbbm{1}_{\{j\neq K\}}}{\olog j},\xi_j\right\},  \bar{\xi}_j\right\}}{j\olog K}\right\},
	$$
	where $\alpha_{j}\in \RR$ is the real number such that 
	$	\frac{2\xi_j\left(1-   \alpha_j \right)}{j\olog j}= [(  (1+\zeta_j)\sqrt{\alpha_j }-\sqrt{\frac{1}{(j+1)\olog (j+1)}}    )_+]^2. 	
	$
\end{theorem}

\begin{theorem}\label{thm:cra}
Let $\bm\in  [0,1]^K$. Under \cra, $\lowlim_{T\rightarrow\infty}\frac{1}{T}\log \frac{1}{\PP_{\bm}\left[\hat{\imath}\neq 1\right]}$ is larger than
$$
2\min_{j=2,\ldots,K}\left\{ \frac{\min\{\max \{ \frac{\psi_j\olog (j+1)(1-\alpha_{j}) \mathbbm{1}_{\{j\neq K\}} }{\olog j},\psi_j\},  \bar{\psi}_j\}}{j\olog K}\right\}, 
$$
where $\alpha_{j}\in \RR$ is the real number such that 
$
\frac{\psi_j\left(1-   \alpha_j \right)}{j\olog j}= \frac{j}{j+1}[(  (1+\varphi_j)\sqrt{\alpha_j }-\sqrt{\frac{1}{(j+1)\olog (j+1)}}    )_+]^2.$
\end{theorem}

Note that Theorem \ref{thm:crc} implies that \crc enjoys better performance guarantees than \sr, and hence has for now the best known error probability upper bounds.
 
\medskip
\noindent
{\it Proof sketch.} The complete proof of Theorems \ref{thm:crc} and \ref{thm:cra} are given in Appendices \ref{app:crc} and \ref{app:cra}. We sketch that of Theorem \ref{thm:crc}. The proof consists in upper bounding $\mathbb{P}_{\bm}[\ell_j=1]$ for $j\in \{2,\ldots,K\}$. We focus here on the most challenging case where $j\in \{3,\ldots,K-1\}$ (the analysis is simpler when $j=K$, since the only possible allocation is uniform, and when $j=2$, since the only possible round deciding $\ell_2$ is the last round).  

To upper bound $\mathbb{P}_{\bm}[\ell_j=1]$ using Theorem \ref{thm:Varadaham bandits}, we will show that it is enough to study the large deviations of the process $\{\om(\theta T)\}_{T\ge 1}$ for any fixed $\theta \in [\theta_0,1]$  and to define a set $\mathcal{S}\subseteq [0,1]^K$ under which $\ell_j=\ell (\theta T)=1$. We first observe that $\ell_j=\ell(\theta T)$ restricts the possible values of $\om(\theta T)$: 
$
\om(\theta T)\in \mathcal{X}_j:=	\left\{ \xx\in \Sigma : \exists \sigma\in [K]^2\hbox{ s.t. } x_{\sigma (1)}=\ldots=x_{\sigma (j)}>x_{\sigma (j+1)}>\ldots > x_{\sigma (K)}> 0\right\}.
$
 We can hence just derive the LDP satisfied by $\{\om(\theta T)\}_{T\ge 1}$ on $\mathcal{X}_j$. This is done in Appendix \ref{app:allocation}, and we identify by $I_\theta$ a rate function leading to an LDP upper bound. By defining 
 \begin{equation*}
\mathcal{X}_{j,i}(\theta)=	\left\{\xx\in \mathcal{X}_j: 	\theta  x_{\sigma (i)}i\olog i> 1-\theta \sum_{k=i+1}^Kx_{\sigma(k)}	   \right\},\,\forall i\in \{j,\ldots,K\},
\end{equation*}
 As it is shown in Appendix \ref{subapp:local} that $I_{\theta}(\xx)=\infty$ if $\xx\in \mathcal{X}_{j,i}(\theta)$ for $i\ge j$,
 we may further restrict to  $\mathcal{X}_j\setminus \cup_{i= j}^K\mathcal{X}_{j,i}(\theta)$.

Next, we explain how to apply Theorem \ref{thm:Varadaham bandits} to upper bound $\mathbb{P}_{\bm}[\ell_j=1]$. Let $\mathcal{J}=\{J\subseteq [K]:\left|J\right|=j,1\in J\}$ as defined in Section \ref{sec:sr}. For all $\beta,\theta \in (0,1]$ and $J\in\mathcal{J}$, we introduce the sets 
\begin{align*}
	\mathcal{S}_{J}(\beta)&=\left\{\bl\in [0,1]^K:\min_{k\in J,k\neq 1}\lambda_k-\lambda_1\ge G(\beta)\right\},\\
	\mathcal{Z}_{J}(\theta,\beta)&=\left\{\zz\in \mathcal{X}_j\setminus  \cup_{i= j}^K\mathcal{X}_{j,i}(\theta):(\forall k\in J,\  z_k=\max_{k'\in [K]}z_{k'}),\ \frac{\theta \sum_{k\in J}z_k\olog j   }{1-\theta\sum_{k\notin J}z_{k}}=\beta\right\}.
\end{align*}
Assume that in round $t$, $ \ell_j=\ell(t)=1,\mathcal{C}_j=J$ and let $\tau=\sum_{k\notin J}N_k(t)\le t$ be the number of times arms outside $J$ are pulled. While $\om(t)\notin \mathcal{X}_{j,j}(t/T)$, we have $\beta = \frac{(t-\tau)\olog j}{T-\tau} \in (0,1]$. Using the criteria (\ref{C}), we observe that
\begin{multline*}
 \sum_{t=K+1}^T\sum_{J\in \mathcal{J}}\PP_{\bm}  \left[  \ell_j=\ell(t)=1,\mathcal{C}_j=J ,\om (t)\in \mathcal{X}_j\setminus \cup_{i=j}^K\mathcal{X}_{j,i}(\frac{t}{T})\right]\\
\le \sum_{t=K+1}^T\sum_{J\in \mathcal{J}}\sum_{\tau\le t,\tau\in \NN} \PP_{\bm}  \left[  \hat{\bm}(t)\in \mathcal{S}_{J}(\frac{(t-\tau)\olog j}{T-\tau}),\om (t)\in \mathcal{Z}_{J}(\frac{t}{T},\frac{(t-\tau)\olog j}{T-\tau}) \right].
\end{multline*}
To upper bound the r.h.s. in the above inequality, we combine the results of Theorem \ref{thm:Varadaham bandits} and a partitioning technique (presented in Appendix \ref{app:partition}). This gives:
$$
\lowlim_{T\rightarrow\infty}\frac{1}{T}\log \frac{1}{\PP_{\bm}\left[ \hat{\bm}( \theta T)\in \mathcal{S}_J(\beta) ,\om (\theta T)\in \mathcal{Z}_J(\theta,\beta)\right]}\ge \theta\inf_{\zz\in \cl(\mathcal{Z}_{J}(\theta,\beta))}\max\{F_{\mathcal{S}_{J}(\beta)}(\zz), I_\theta(\zz)\}.
$$
The proof is completed by providing lower bounds of $\theta F_{S_{J}(\beta)}(\zz)$ and $\theta I_{\theta} (\zz)$ for a fixed $\zz\in \mathcal{Z}_J(\theta,\beta)$ with various $J$. Such bounds are derived in Appendix \ref{app:opt C} and \ref{subapp:rate C}, respectively.
\eproof

\noindent
{\bf Example 2.} To conclude this section, we just illustrate through a simple example the gain in terms of performance guarantees brought by \sred compared to \sr. Assume we have 50 Bernoulli arms with $\mu_1=0.95,\mu_2=0.85,\mu_3=0.2$,  and $\mu_k=0$ for $k=4,\ldots,50$. For \sr, Theorem \ref{thm:sr} states that with a budget of 5000 samples, the error probability of \sr does not exceed $1.93\times 10^{-3}$. With the same budget, Theorems \ref{thm:crc} and \ref{thm:cra} state that the error probabilities of \crc and \cra do not exceed $6.40\times 10^{-4}$ and $6.36\times 10^{-4}$, respectively.

We note that in general, we cannot say that one of our two algorithms, \crc or \cra, has better guarantees than the other. This is demonstrated in the problem instances presented in Appendix \ref{num:1} and \ref{num:4}.

\section{Numerical Experiments}
We consider various problem instances to numerically evaluate the performance of \sred. In these instances, we vary the number of arms from 5 to 55; we use Bernoulli distributed rewards, and vary the shape of the arm-to-reward mapping. For each instance, we compare \sred to \sr, \sh, and \ugape. 

Most of our numerical experiments are presented in Appendix \ref{app:exp}. Due to space constraints, we just provide an example of these results below. In this example, we have 55 arms with {\it convex} arm-to-reward mapping. The mapping has 10 steps, and the $m$-th step consists of $m$ arms with same average reward, equal to $\frac{3}{4}\cdot 3^{-\frac{m}{10}}$. Table \ref{tab:stair-mu} presents the error probabilities averaged over $40,000$ independent runs. Observe that \cra performs better than \crc (being aggressive when discarding arms has some benefits), and both versions of \sred perform better than \sr and all other algorithms.



\begin{table}[htb!]
\caption{Error probability (in \%).
\label{tab:stair-mu}}
\centering
\begin{tabular}{llccc} \toprule
  &   & $T=3,000$ & $T=4,000$ & $T=5,000$\\\midrule
    \ugape & \citep{gabillon2012best} & $24.7$ & $21.3$ & $18.9$ \\\hline
    \sh &\cite{karnin2013almost} & $10.2$ & $5.9$ & $3.2$\\\hline
    \sr &\cite{audibert2010best}& $5.5$ & $2.8$ & $1.3$\\\hline
    \crc &(this paper) & $7.1$ & $2.6$ & $1.1$\\\hline
    \cra &(this paper) & $\mathbf{4.7}$ & $\mathbf{1.6}$ & $\mathbf{0.6}$\\\bottomrule
\end{tabular}%
\end{table}
\section{Conclusion}\label{sec:conclusion}

In this paper, we have established, in MAB problems, a connection between the LDP satisfied by the sampling process (under any adaptive algorithm) and that satisfied by the empirical average rewards of the various arms. This connection has allowed us to improve the performance analysis of existing best arm identification algorithms, and to devise and analyze new algorithms with an increased level of adaptiveness. We show that one of these algorithms \crc has better performance guarantees than existing algorithms and that it performs also better in practice in most cases.

Future research directions include: (i) developing algorithms with further improved performance guarantees -- can the discarding conditions of \sred be further optimized? (ii) Enhancing the Large Deviation analysis of adaptive algorithms -- under which conditions, can we establish a complete LDP of the process $\{Z(t)\}_{t\ge 1}=\{(\om(t), \hat{\bm}(t))\}_{t\ge 1}$? Answering this question would constitute a strong step towards characterizing the minimal instance-specific error probability for best arm identification with fixed budget. (iii) Extending our approach to other pure exploration tasks:  top-$m$ arm identification problems \citep{bubeck2013multiple}, best arm identification in structured bandits \citep{yang2022minimax,azizi2022fixed}, or best policy identification in reinforcement learning.

\section*{Acknowledgements}
The authors would like to express their gratitude to Guo-Jhen Wu for his invaluable discussion during the initial stages of this project. R.-C Tzeng's research is supported by the ERC Advanced Grant REBOUND (834862), A. Proutiere is supported by the Wallenberg AI, Autonomous Systems and Software Program (WASP) funded by the Knut and Alice Wallenberg Foundation, and Digital Futures.
\bibliographystyle{apalike}
\bibliography{ref}

\newpage
\appendix
\tableofcontents
\newpage

\section{Notation}

\begin{table}[htbp]
	\begin{center}
		\small 
		\begin{tabular}{c c p{10cm} }
			\hline
			\multicolumn{3}{l}{\bf Problem setting}\\
			\hline
			$K$ &   & Number of arms\\
			$[m]$ for any $ m \in\NN $& & The set $\{1,2\ldots,m\}$\\
			$\nu_k$ &  &  Reward distribution for arm $k$ \\
			$X_k(t)$& & Random reward received from pulling arm $k$ in round $t$\\
			$\bm \in \mathbb{R}^K$ & & Vector of the expected rewards of the various arms\\
			$\Lambda$ &   & Set of all possible parameters $\bm$\\ 
			$1 (\bm)$& & Best arm under parameter $\bm$\\
			$T$& & Given budget\\
			\noalign{\vskip 1mm} %
			\hline
			\multicolumn{3}{l}{\bf Quantities related to the error rate lower bound}\\
			\hline
			$\om$& & Vector of the proportions of arm draws \\
			$\Sigma$& & Simplex\\
			$\E_{\bm}$ and $\PP_{\bm}$& & The expectation and probability measure corresponding to $\bm$\\
			$\alt(\bm)$ & & Set of confusing parameters for $\bm$ (whose best arm is not $1(\bm)$)\\
			$d(\mu,\mu')$ & & KL divergence between the distributions parametrized by $\mu$ and $\mu'$\\
			$\skl (a,b)$& &  KL divergence between two Bernoulli distributions of means $a$ and $b$\\
			$\Psi(\bl,\om)$&& $\sum_{k=1}^K\omega_kd(\lambda_k,\mu_k)$\\
			\noalign{\vskip 1.5mm}
			\hline
			\multicolumn{3}{l}{\bf Notation used in large deviation theory}\\
			\hline
			$\cl (\mathcal{S})$&& The closure of $\mathcal{S}$\\
			$F_{\mathcal{S}}(\om)$&& $\inf_{\bl\in \cl (\mathcal{S})}\Psi(\bl,\om)$\\
			$I$ && Rate function for $\{\om(t)\}_{t\ge 1}$\\
			$B(y,\delta)$&& The open ball with center $y$ and radius $\delta$\\
			\noalign{\vskip 1mm} 
			\hline
			\multicolumn{3}{l}{\bf Notation used in the algorithms}\\
			\hline
			$N_k(t)$&& Number of pulls of arm $k$ up to $t$\\
			$\omega_k(t)$& & $N_k(t)/t$\\
			$A_t$& & The arm pulled in round $t$ \\
			$\hat{\mu}_k(t)$& &$\sum_{s=1}^t X_k(s)\mathbbm{1}\{A_s=k\}/N_k(t)$\\
			$\hat{\imath}$& & Recommended arm\\
			\hline
			\noalign{\vskip 1mm}
			\multicolumn{3}{l}{\bf Notation for \sr, \sred (assuming $\mu_1>\mu_2\ge\ldots\ge \mu_K$)}\\
			\hline
			$\mathcal{C}_j$& & Candidates set with size $j$\\
			$\ell_j$& & The arm discarded from $\mathcal{C}_j$\\
			$\ell(t)$&& Empirical worst arm at round $t$\\
			$\olog j$&& $\frac{1}{2}+\sum_{k=2}^j\frac{1}{k}$\\
			$G(\beta)$& &$\frac{1}{\sqrt{\beta}}-1$\\
			$\mathcal{J}$&& $\{ J\subseteq [K]:\left| J \right|=j,1\in J\}$ \\
                $I_\theta$&& Rate function for $\{\om (\theta T)\}_{T\ge 1}$ \\
			$\Gamma_j$&& $\min_{ J\in \mathcal{J}} \inf\left\{ \sum_{k\in J}d(\lambda_k,\mu_k):\bl\in \RR^K,\lambda_{1}\le \min_{k\in J} \lambda_k\right\}$\\
			$\xi_j$&& $\inf_{\bl\in [0,1]^j}\left\{ \sum_{k=1}^j(\lambda_k-\mu_k)^2:\lambda_{1}\le \min_{k=1,\ldots,j} \lambda_k\right\}$\\
			$\bxi_j$&&$\inf_{\bl \in [0,1]^K}\left\{   \sum_{k=1}^K(\lambda_k-\mu_k)^2:    \lambda_1\le \min_{k=2,\ldots,j-1,j+1 }\lambda_k  \right\}$\\
			$\psi_j$&&$\frac{j-1}{j}\left(\mu_{1}-\frac{\sum_{k=2}^j\mu_{k}}{j-1}\right)^2$\\
			$\bpsi_j$&&$\frac{j-1}{j}\left(\mu_1-\frac{\sum_{k=2}^{j-1}\mu_{k}+\mu_{j+1}}{j-1}\right)^2$\\
			$\zeta_j$&&$\mu_j-\mu_{j+1}$\\
			$\varphi_j$&& $\sum_{k=1}^j\mu_{k}/j-\mu_{j+1}$\\
			\hline
		\end{tabular}
		\normalsize
	\end{center}
	\label{tab:TableOfNotationsGeneral}
\end{table}

\newpage
\section{Technical lemmas towards the proof of Theorem \ref{thm:Varadaham bandits}}\label{app:tec}

\begin{lemma}\label{lem:Xbd}
	Let $\bm\in \Lambda,\, t>\max\{K,e\}$ and $\beta \in (0, 1)$. There is a constant $c>0$ (that depends on $K$ and $\beta$ only) s.t.
	\[
	\E_{\bm}\left[ e^{  \beta X} \right] \le c(\log t)^K,
	\]
	where $X=\sum_{k=1}^K N_{k}(t)d(\hat{\mu}_k(t),\mu_k) $.
\end{lemma}
\begin{proof}
	Let $M$ be the smallest positive integer s.t. (i) $\frac{\log M}{\beta }>K+1$ and (ii) $(\log M)^{2K}<M^{\frac{1}{\beta}}$. We have: 
	\begin{align}
	\nonumber	\E_{\bm}\left[e^{\beta X}\right] & \le \sum_{n=0}^\infty \PP_{\bm} \left[ e^{\beta X}\ge n\right]\\
	\nonumber &\le M +\sum_{n\ge M} \PP_{\bm}\left[ X\ge \frac{\log n}{\beta} \right]\\
	&    \le M +   \sum_{n\ge M}  \frac{\left( 2(\log n)^2\log t  \right)^K }{n^{\frac{1}{\beta}}} \frac{e^{K+1}}{\beta^{2K} K^K} ,\label{eq:Xbd1}
	\end{align}
	where the last inequality follows from repeatedly invoking Lemma \ref{lem:kl} with $\delta =\frac{\log n}{\beta} $ (notice that $n\ge M$ satisfies the condition on $\delta$ of Lemma \ref{lem:kl}). Observe that the r.h.s. of (\ref{eq:Xbd1}) is a Bertrand series and it is convergent since $\beta <1$. Unfamiliar reader can check the convergence analysis below. (ii) implies that the sequence will decrease after $n\ge M$, hence the sum in (\ref{eq:Xbd1}) is bounded as (up to a constant multiplicative factor):
	\begin{align}
	\nonumber    \sum_{n\ge M}  \frac{\left( \log n  \right)^{2K} }{n^{\frac{1}{\beta}}}  & \le      \int_{1}^{\infty}\frac{(\log x)^{2K}}{x^{\frac{1}{\beta}}}     dx\\
	\nonumber		&=      \int_{0}^{\infty} y^{2K}e^{-(\frac{1}{\beta}-1)y} dy\\
	&\le  \left(\frac{1}{\beta}-1\right)^{-2K-1} \Gamma (2K+1).\label{eq:Xbd2}
	\end{align}
	The constant $c$ can be deduced from (\ref{eq:Xbd1}) and (\ref{eq:Xbd2}).
\end{proof}

\medskip
\noindent
The following lemma is Theorem 2 in \cite{magureanu2014lipschitz}. It was originally stated for Bernoulli distributions, but as claimed in \cite{garivier2016optimal,kaufmann2018mixture}, it is straightforward to generalize it to one-parameter exponential distributions. 
\begin{lemma}[\cite{magureanu2014lipschitz}]\label{lem:kl}
	For all $\delta>(K+1)$ and $t\in \NN$, we have:
	\[
	\PP_{\bm}\left[ X\ge \delta \right]\le e^{-\delta}\left(\frac{\lceil\delta\log t\rceil\delta}{K}  \right)^K e^{K+1}.
	\]
\end{lemma}

\newpage
\section{Analysis of \sred}\label{app_cr}

In \ref{app:crc}, we give the proof for Theorem \ref{thm:crc} and in \ref{app:cra} that of Theorem \ref{thm:cra}. As mentioned in \textsection \ref{sec:sr} and \textsection \ref{sec:anal cr}, in the following analysis, we will assume that $1\ge \mu_1>\mu_2\ge\ldots\ge \mu_K\ge 0$.

\subsection{Performance analysis of \crc}\label{app:crc}
\begin{reptheorem}{thm:crc}
Let $\bm\in  [0,1]^K$. Under \crc, $\lowlim_{T\rightarrow\infty}\frac{1}{T}\log \frac{1}{\PP_{\bm}\left[\hat{\imath}\neq 1\right]}$ is larger than
$$
2\min_{j=2,\ldots,K}\left\{ \frac{\min\left\{\max \left\{ \frac{\xi_j\olog (j+1)(1-\alpha_{j})\mathbbm{1}_{\{j\neq K\}}}{\olog j},\xi_j\right\},  \bar{\xi}_j\right\}}{j\olog K}\right\},
$$
where $\alpha_{j}\in \RR$ is the real number such that 
$$
 \frac{2\xi_j\left(1-   \alpha_j \right)}{j\olog j}= \left[\left(  (1+\zeta_j)\sqrt{\alpha_j }-\sqrt{\frac{1}{(j+1)\olog (j+1)}}    \right)_+\right]^2.
$$

\end{reptheorem}
\noindent
We upper bound $\PP_{\bm}\left[\ell_j=1\right]$ for (i) $j=K$; (ii) $j=2$; (iii) $j\in \{3,\ldots,K\}$. The upper bound for (i), presented in \ref{subapp:cK}, is the easiest to derive as the only possible allocation before one discards the first arm is uniform among all arms. The bound for (ii), presented in \ref{subapp:c2}, is the second easiest to derive as $\ell_2$ is decided only in the end, namely, in the $T$-th round. The upper bound for (iii), presented in \ref{subapp:cj}, is more involved since we have to consider all possible allocations and rounds. 
\subsubsection{Upper bound of $\PP_{\bm}\left[\ell_K=1\right]$}\label{subapp:cK}
\begin{lemma}
Let $\bm\in  [0,1]^K$. Under \crc, 
	$$
\lowlim_{T\rightarrow\infty}\frac{1}{T}\log \frac{1}{\PP_{\bm}\left[\ell_K= 1\right]}\ge  \frac{2\xi_K}{K\olog K}.
	$$
\end{lemma}

\begin{proof}
Without loss of generality, let us assume $\theta_0T>K$. Observe that 
\begin{equation}\label{eq:c111}
	\PP_{\bm}\left[\ell_K=1\right]= \sum_{t\ge \theta_0 T}^T	\PP_{\bm}\left[\ell_K=\ell(t)=1\right].
\end{equation}
Since \crc discards $\ell_K=\ell(t)$ at the round $t$ only when $N_{\ell(t)}(t)=N_k(t)$ for all $k\in [K]$, it suffices to consider $\om(t)\in \mathcal{X}_K=\{(1/K,\ldots,1/K)\}$. We further introduce 
$$
\mathcal{S}_{\theta}=\left\{\bl\in \RR^K:\lambda_1\le \min_{k\neq 1}\lambda_k-G(\theta\olog K)\right\},\,\forall \theta\in [\theta_0,1]\cap \QQ.
$$
With this notation, we can use the criteria of discarding the arm $\ell_K=\ell(t)=1$ (see \eqref{C}) to get that:
\begin{equation} \label{eq:c112}
\sum_{t\ge \theta_0 T}^T	\PP_{\bm}\left[\ell_K=\ell(t)=1\right] \le  \sum_{t\ge \theta_0 T}^T	\PP_{\bm}\left[ \hat{\bm}(t) \in \mathcal{S}_{\frac{t}{T}},\om (t)\in \mathcal{X}_K  \right].
\end{equation}
Applying Theorem~\ref{thm:11} in Appendix~\ref{app:partition} with $\tilde{\theta}_0=\theta_0$
$$
\mathcal{E}=\{\ell_K=1\},\, 
\mathcal{S}_{\theta,\gamma}=\mathcal{S}_{\theta}
,\,
W_{\theta,\gamma}=  \mathcal{X}_{K},\forall \gamma, 
$$
yields that
\begin{align}\label{eq:c113}
 \nonumber   \lowlim_{T\to \infty} \frac{1}{T}\log \frac{1}{\PP_{\bm}[\ell_K=1]} &\ge \inf_{\theta,\gamma\in [\theta_0,1]\cap \QQ}\inf_{\om\in \cl(W_{\theta,\gamma })}\theta \max\{F_{\mathcal{S}_{\theta,\gamma}}(\om),I_\theta (\om)\}\\
    &=\inf_{\theta\in [\theta_0,1]\cap \QQ} \theta \{F_{\mathcal{S}_{\theta}}(1/K,\ldots,1/K),I_\theta  (1/K,\ldots,1/K)\}.
\end{align}
Theorem~\ref{thm:11} can be indeed applied since, in view of Theorem~\ref{thm:crc main rate}, $\{\om(\theta T)\}_{T\ge 1}$ satisfies LDP upper bound (\ref{eq:rate UB}) with rate function $I_\theta$.
In the above derivation, by convention, we let $\inf_{\bl\in \emptyset}f(\bl)=\infty$. Next, from Theorem \ref{thm:crc main rate} (a) in Appendix \ref{subapp:local}, we know that $I_\theta (1/K,\ldots,1/K)=\infty$ if $\theta>1/\olog K$. Thus, the minimization problem on r.h.s. of (\ref{eq:c113}) can be further lower bounded by $\inf_{\theta\in [\theta_0,1/\olog K]\cap \QQ} \theta F_{\mathcal{S}_{\theta}}(1/K,\ldots,1/K) $.
From the definition of $F_{\mathcal{S}_\theta}$, we have
\begin{align}\label{eq:c114}
	\nonumber	\theta F_{\mathcal{S}_\theta}(1/K,\ldots,1/K)&=\frac{\theta }{ K}\inf\left\{\sum_{k=1}^K d(\lambda_k,\mu_k):\lambda_1 \le \min_{k\neq 1}\lambda_k-G(\theta \olog K)\right\}\\
	&\ge\frac{2\theta }{K} \inf \left\{ \sum_{k=1}^K (\lambda_k-\mu_k)^2:\lambda_1 \le \min_{k\neq 1}\lambda_k-G(\theta \olog K)\right\} \ge  \frac{2\xi_K}{K\olog K},
\end{align} 
where the first inequality is from Assumption \ref{ass1}, and the last inequality follows from Proposition \ref{prop:crc G} in Appendix \ref{app:opt C} with $\beta=\theta \olog K$. The proof is completed combining (\ref{eq:c113}) and (\ref{eq:c114}).

\end{proof}

\subsubsection{Upper bound of $\PP_{\bm}\left[\ell_2=1\right]$}\label{subapp:c2}

\begin{lemma}\label{lem4}
	Let $\bm\in  [0,1]^K$. Under \crc, 
	$$
	\lowlim_{T\rightarrow \infty}\frac{1}{T}\log \frac{1}{\PP_{\bm}\left[\ell_2= 1\right]}\ge   \frac{\min\{\max \{ 4\xi_2(1-\alpha_2),3\xi_2 \},3\bar{\xi}_2\}}{3\olog K},
	$$
	where $\alpha_2\in \RR$ is the real number such that
	\begin{equation}\label{eq:C alpha 2}
	\xi_2\left(1-   \alpha_2 \right)= \left[\left(  (1+\zeta_2)\sqrt{\alpha_2 }-\frac{1}{2}    \right)_+\right]^2. 	
	\end{equation}
\end{lemma}

\begin{proof}
    There are two arms remaining in the last phase. Hence, it suffices to consider the estimate and allocation in the last round. The possible allocation $\om (T)$ belongs to the set
$$
\mathcal{X}_2=\left\{  \xx\in \Sigma: \exists \sigma:[K]  \mapsto [K] \hbox{ s.t. } x_{\sigma (1)}=x_{\sigma (2)}>x_{\sigma (3)}> \ldots > x_{\sigma (K)}> 0\right\}.
$$	
As we defined $\mathcal{J}$ in Section \ref{sec:sr}, we introduce $\mathcal{D}=\{D\subseteq [K]:\left|D\right|=2,1\in D\}$, and  
$$
\mathcal{X}_D= \left\{\xx\in {\cal X}_2: \min_{k\in D}x_k>\max_{k'\notin D}x_{k'}\right\}.
$$
The set $\cup_{D\in \mathcal{D}}{\cal X}_D$ is a subset of $\mathcal{X}_2$, and is relevant when we consider events where the best arm 1 is discarded in the last elimination phase. Since $\ell_2$ is decided as the empirical worst arm in $\mathcal{C}_2$, 
\begin{align}\label{eq:c121}
 \nonumber \PP_{\bm}[\ell_2=1]&\le \sum_{D\in {\cal D}}\PP_{\bm}[\hat{\bm} (T)\in {\cal S}_D,\om (T)\in {\cal X}_D]\\
 &\le (K-1)\max_{D\in {\cal D}}\PP_{\bm} [\hat{\bm} (T)\in {\cal S}_D,\om (T)\in {\cal X}_D],
\end{align}
where 
$$
\mathcal{S}_{D}=\left\{\bl\in [0,1]^K:\min_{k\in [D]}\lambda_k\ge \lambda_1\right\}.
$$
Rearranging (\ref{eq:c121}) yields that 
\begin{align}
    \nonumber \lowlim_{T\to \infty}\frac{1}{T}\log \frac{1}{\PP_{\bm}[\ell_2=1]} &\ge \lowlim_{T\to\infty}\min_{D\in {\cal D}} \frac{1}{T}\log\frac{1}{\PP_{\bm}[ \hat{\bm} (T)\in {\cal S}_D,\om (T)\in {\cal X}_D]}\\
\nonumber     &\ge \min_{D\in {\cal D}}  \lowlim_{T\to\infty}\frac{1}{T}\log\frac{1}{\PP_{\bm}[ \hat{\bm} (T)\in {\cal S}_D,\om (T)\in {\cal X}_D]} \\
\label{eq:c122}&\ge \min_{D\in {\cal D}} \inf_{\om \in \cl ({\cal X}_D)} \max\{ F_{{\cal S}_{D}}(\om ), I_1(\om )\},
\end{align}
where $I_1$ denotes the rate function for which $\{\om (T)\}_{T\ge 1}$ satisfies an LDP upper bound (\ref{eq:rate UB}), and the last inequality follows from Theorem \ref{thm:Varadaham bandits} with ${\cal S}={\cal S}_D, W={\cal X}_D $.

Further introduce for all $i\in \{2,\ldots, K\}$,	
$$
\mathcal{X}_{2,i}(1)=	\left\{\xx\in \mathcal{X}_2: 	x_{\sigma (i)}i\olog i> 1- \sum_{k=i+1}^Kx_{\sigma(k)}	   \right\},
$$	
where in this definition, $\sigma$ refers to the permutation used in the definition of $\mathcal{X}_2$.
In Theorem \ref{thm:crc main rate} (a) in Appendix \ref{subapp:local}, we show that $I_1(\om )=\infty$ for all $\om \in \cup_{i=2}^K{\cal X}_{2,i}(1)$. Hence any $\om \in \cup_{i=2}^K{\cal X}_{2,i}(1)$ cannot be the minimizer on the r.h.s. of (\ref{eq:c122}). In the following, we define 
$$
{\cal Z}_D (1,1)={\cal X}_D\setminus  \cup_{i=2}^K{\cal X}_{2,i}(1).
$$
Here the argument $(1,1)$ in $\mathcal{Z}_{D}(1,1)$ is to be consistent with our notation in Appendix \ref{app:allocation}, but we abbreviate it as $\mathcal{Z}_D$ for short below. To get a lower bound of (\ref{eq:c122}), we consider two cases: (a) $D\neq [2]$; (b) $D=[2]$.

\noindent
\underline{\bf(a) The case where $D\neq [2]$.} Using Corollary \ref{cor:crc G new} with $\beta =1,j=2$ in Appendix \ref{app:opt C} yields that:
\begin{align}
\nonumber   \inf_{\om \in \cl ({\cal X}_D)} \max\{ F_{{\cal S}_{D}}(\om ), I_1(\om )\}& \ge \inf_{z\in \cl ({\cal Z}_D )}  F_{{\cal S}_{D}}(\zz )\\
\label{eq:c123} &\ge  \left(1-\sum_{k\notin D}z_k\right)    \inf\left\{\sum_{k\in D} (\lambda_k-\mu_k)^2:\bl\in [0,1]^K,\,\min_{k\in D}\lambda_k= \lambda_1\right\}\\	
\nonumber &\ge\left(1-\sum_{k\notin D}z_k\right) \bar{\xi}_2\\
\label{eq:c124}&\ge \frac{\bar{\xi}_2}{\olog K},
\end{align}
where the last inequality directly comes from Proposition \ref{prop ybound} with $\theta =1,\,j=2,\,i=3$ in Appendix \ref{subapp:infinite I}.

\underline{\bf (b) The case where $D= [2]$.} Next, we will show that both $\frac{\xi_2}{\olog K}$ and $\frac{4\xi_2(1-\alpha_2)}{3\olog K}$ are lower bounds for $\inf_{\om \in \cl ({\cal X}_D)} \max\{ F_{{\cal S}_{D}}(\om ), I_1(\om )\}$. The maximum of these hence becomes our lower bound. Together with the conclusion obtained in the case (a), we complete the proof of Lemma \ref{lem4}.

\underline{Lower bounding by $\frac{\xi_2}{\olog K}$.} 
Observe that (\ref{eq:c123}) holds also for $D=[2]$, hence 
\begin{align}
\nonumber	\inf_{\om \in \cl ({\cal X}_{[2]})} \max\{ F_{{\cal S}_{[2]}}(\om ), I_1(\om )\}&\ge  \inf_{\om \in \cl ({\cal Z}_{[2]})} F_{{\cal S}_{[2]}}(\om ) \\
\nonumber& \ge  \left(1-\sum_{k=3}^K z_k\right)    \inf_{\bl\in \RR^K,\, \lambda_2> \lambda_1}\left\{\sum_{k\in [2]} (\lambda_k-\mu_k)^2\right\}\\
\label{eq:c125}& \ge \left(1-\sum_{k=3}^Kz_{k}\right) \xi_2\\
\label{eq:c126}&\ge \frac{\xi_2}{\olog K}, 
\end{align}
where the second inequality is derived by Proposition \ref{prop:crc G} with $\theta =1, \beta=1, j=2$ in Appendix \ref{app:opt C}, and the last inequality follows from Proposition \ref{prop ybound} with $j=2,i=3$ in Appendix \ref{subapp:finite I}.\\

\underline{Lower bounding by $\frac{4\xi_2(1-\alpha_2)}{3\olog K} $}. One can derive another lower bound by using $I_1$. In Theorem \ref{thm:crc main rate} (b), $j=2,\theta=\beta=1$ in Appendix~\ref{subapp:local}, we show that $I_1$ is a valid lower semi-continuous rate function for an LDP upper bound (\ref{eq:rate UB}) for the process $\{\om (T)\}_{T\ge 1}$. In Corollary~\ref{cor:cj rate} in Appendix~\ref{subapp:rate C}, we further show that $I_1(\zz)\ge \underline{I}_{1}(\zz)$ for $\zz\in {\cal Z}_{[2]}$, where
\begin{equation}\label{eq:c127}
\underline{I}_{1}(\zz)=\frac{4}{3\olog K}   \left[\left(  (1+\zeta_2)\sqrt{\frac{ z_{\sigma(3)}}{ 1-\sum_{k=4}^Kz_{\sigma(k)}  }}-\frac{1}{2}    \right)_+\right]^2,\,\forall \zz\in \mathcal{Z}_{[2]}.
\end{equation}

\noindent
Instead of using (\ref{eq:c126}), we lower bound $ F_{S_{[2]}}(\zz)$ as:
\begin{align}\label{eq:c128}
	\nonumber 	 F_{S_{[2]}}(\zz)&\ge 	 \left(1-\sum_{k=3}^Kz_{\sigma (k)}\right) \xi_2 \\
	\nonumber	&=\xi_2 \left(1- \sum_{k=4}^Kz_{\sigma(k)}  \right) \left(1-   \frac{z_{\sigma(3)}}{1- \sum_{k=4}^Kz_{\sigma(k)} }\right)\\
	&\ge \frac{4\xi_2}{3\olog K} \left(1-   \frac{z_{\sigma(3)}}{1- \sum_{k=4}^Kz_{\sigma(k)} }\right),
\end{align}
where the first inequality follows the derivation of (\ref{eq:c125}) and the last inequality stems from Proposition~\ref{prop ybound} with $\theta=1,j=2,i=4$ in Appendix~\ref{subapp:finite I}. Since $F_{S_{[2]}}(\zz)$ and $I_{1}(\zz)$ are lower bounded by the functions of $\alpha =  \frac{  z_{\sigma(3)}}{1- \sum_{k=4}^K z_{\sigma (k)}} $ given in (\ref{eq:c127}) and (\ref{eq:c128}), we have:
\begin{align}\label{eq:c129}
\nonumber\inf_{\om \in \cl ({\cal X}_{[2]})} \max\{ F_{{\cal S}_{[2]}}(\om ), \underline{I}_1(\om ) \}&\ge \frac{4}{3\olog K}  \inf_{\alpha\in \RR} \max\{ \xi_2\left(1-   \alpha \right) ,     [(  (1+\zeta_2)\sqrt{\alpha }-\frac{1}{2}    )_+]^2 \}\\
&\ge  \frac{4\xi_2(1-\alpha_2)}{3\olog K},
\end{align}
where the last inequality is due to Lemma \ref{lem:simple sol new1} in Appendix \ref{app:maxmin} and $\alpha_2$ is defined in (\ref{eq:C alpha 2}). Hence, the maximum of the r.h.s. of (\ref{eq:c126}) and (\ref{eq:c129}) is a lower bound for $\inf_{\om \in \cl ({\cal X}_{[2]})} \max\{ F_{{\cal S}_{[2]}}(\om ), \underline{I}_1(\om )\}$

\end{proof}
\subsubsection{Upper bound for $\PP_{\bm}\left[\ell_j=1\right]$ for $j\in \{3,\ldots,K-1\}$}\label{subapp:cj}
\begin{lemma}\label{lem5}
	Let $\bm\in  [0,1]^K,\,j\in \{3,\ldots,K-1\}  $. Under \crc, 
	$$
	\lowlim_{T\rightarrow \infty}\frac{1}{T}\log \frac{1}{\PP_{\bm}\left[\ell_j= 1\right]}\ge    \frac{2\min\left\{\max \left\{ \frac{\xi_j\olog (j+1)(1-\alpha_{j})}{\olog j},\xi_j\right\},  \bar{\xi}_j\right\}}{j\olog K},
$$
	where $\alpha_j\in \RR$ is the real number such that
\begin{equation}\label{eq:C alpha j}
	 \frac{2\xi_j\left(1-   \alpha_j \right)}{j\olog j}= \left[\left(  (1+\zeta_j)\sqrt{\alpha_j }-\sqrt{\frac{1}{(j+1)\olog (j+1)}}    \right)_+\right]^2. 	
\end{equation}
\end{lemma}

\begin{proof}
Without loss of generality, we assume $\theta_0T>K$ and $\frac{\theta_0T}{K}\in \NN$.\\
\noindent
Observe that $\PP_{\bm}[\ell_j=1]= \sum_{J\in {\cal J}} \PP_{\bm} [\ell_j=1,{\cal C}_j=J]$, which directly implies 
\begin{equation}\label{eq:c131}
    \lowlim_{T\rightarrow \infty}\frac{1}{T}\log \frac{1}{\PP_{\bm}\left[\ell_j= 1\right]}\ge\min_{J\in {\cal J}}   \lowlim_{T\rightarrow \infty}\frac{1}{T}\log \frac{1}{\PP_{\bm}\left[\ell_j= 1,{\cal C}_j=J\right]}.
\end{equation}

 There are $j$ arms remaining while discarding $\ell_j$. The possible allocation $\om (t)$ belongs to the set
$$
\mathcal{X}_j=\left\{ \xx\in \Sigma : \exists \sigma: [K] \mapsto [K] \hbox{ s.t. } x_{\sigma (1)}=\ldots=x_{\sigma (j)}>x_{\sigma (j+1)}> \ldots > x_{\sigma (K)}> 0\right\}.
$$
Suppose that in round $t$, $\mathcal{C}_j=J$ for some $J\in {\cal J}$ and let $\tau=\sum_{k\notin J}N_k(t)\ge \tilde{\theta}_0T$, where $\tilde{\theta}_0=(K-j)\theta_0/K$, be the number of times arms outside $J$ are pulled. While $\ell_j=\ell(t)=1$, we must have $\hat{\bm}(t)\in \mathcal{S}_{J}(\frac{(t-\tau)\olog j}{T-\tau})$, where 
 $$
 \mathcal{S}_J(\beta)=\left\{ \bl\in [0,1]^K:\min_{k\in J,k\neq 1}\lambda_k-\lambda_1\ge G(\beta)   \right\},\,\forall \beta >0,
 $$
 because $\frac{(t-\tau)\olog j}{T-\tau}=\frac{\sum_{k\in\mathcal{C}_j}N_k(t)\olog j  }{T-\sum_{k\notin \mathcal{C}_j}N_k(t) }$ (recall the discarding condition (\ref{C})). 
Now further introduce $\forall \theta,\beta\in (0,1]$,
 $$
 \mathcal{X}_{J}(\theta,\beta)=\left\{\xx\in \mathcal{X}_j :(\forall k\in J,\  x_k=\max_{k'\in [K]}x_{k'}),\ \frac{\theta \sum_{k\in J}x_k\olog j   }{1-\theta\sum_{k\notin J}x_{k}}=\beta\right\}.
 $$
 
We then have: 
\begin{equation*} \PP_{\bm}\left[\ell_j= 1,{\cal C}_j=J\right]\le \sum_{t\ge \theta_0 T}^T\sum_{\tau \ge \tilde{\theta}_0 T}^{t}\PP_{\bm}\left[\hat{\bm}(t)\in \mathcal{S}_J(\frac{(t-\tau)\olog j}{T-\tau}),\om (t)\in \mathcal{X}_J(\frac{t}{T}, \frac{(t-\tau)\olog j}{T-\tau} )\right].
 \end{equation*}
 Applying Theorem~\ref{thm:11} in Appendix~\ref{app:partition} with $
\mathcal{E}=\{\ell_j=1,{\cal C}_j=J \},$
$$ 
\mathcal{S}_{\theta,\gamma}=\left\{\begin{array}{cc}
    \mathcal{S}_{J}(\frac{(\theta-\gamma)\olog j}{1-\gamma}), & \text{if } G(\frac{(\theta-\gamma)\olog j}{1-\gamma})\le 1, \\
   \mathcal{S}_{J}(G^{-1}(1)),  & \text{otherwise},
\end{array}\right. \text{, and }W_{\theta,\gamma}=\mathcal{X}_{J}(\theta,\frac{(\theta-\gamma)\olog j}{1-\gamma}),
$$
(notice that $\beta=\frac{(t-\tau)\olog j}{T-\tau}=\frac{(\theta-\gamma)\olog j}{1-\gamma}$ and ${\cal S}_{J}(\beta)=\emptyset$ if $G(\beta)> 1$) yields that
\begin{multline}\label{eq:c132-}
\lowlim_{T\rightarrow\infty}\frac{1}{T}\log \frac{1}{\PP_{\bm}\left[\ell_j= 1,{\cal C}_j=J \right]}\ge\\
\inf_{\theta\in [\theta_0,1]\cap \QQ}\inf_{\gamma\in [\tilde{\theta}_0,1]\cap \QQ} \inf_{\xx\in \cl(\mathcal{X}_{J}(\theta,\frac{(\theta-\gamma)\olog j}{1-\gamma}))}\theta\max\{F_{\mathcal{S}_{J}( \frac{(\theta-\gamma)\olog j}{1-\gamma} )}(\xx), I_\theta(\xx)\}.
\end{multline}

Further introduce for all $i\in \{j,\ldots, K\}$,	
$$
\mathcal{X}_{j,i}(\theta)=	\left\{\xx\in \mathcal{X}_j: 	\theta  x_{\sigma (i)}i\olog i> 1-\theta \sum_{k=i+1}^Kx_{\sigma(k)}	   \right\}.
$$	
In Theorem \ref{thm:crc main rate} with (a) in Appendix~\ref{subapp:local}, we show that $I_\theta(\om )=\infty$ for all $\om \in \cup_{i=j}^K{\cal X}_{2,i}(\theta)$, hence any $\om \in \cup_{i=j}^K{\cal X}_{j,i}(\theta)$ will not be the minimizer on the r.h.s. of (\ref{eq:c132-}). In the following, we define 
$$
{\cal Z}_J (\theta,\beta)={\cal X}_J(\theta,\beta)\setminus  \cup_{i=j}^K{\cal X}_{j,i}(\theta).
$$
Consider any $\zz\in \mathcal{Z}_J(\theta,  \beta )$, as $\zz\notin {\cal X}_{j,j}(\theta)$, we have 
$$
\frac{(\theta-\gamma)\olog j}{1-\gamma}=\frac{\theta z_{\sigma(j)}j\olog j}{1-\theta\sum_{k>j}z_{\sigma (k)}}<1.
$$
Therefore, after excluding the points in $\cup_{i=j}^K{\cal X}_{j,i}(\theta)$ and setting $\beta= \frac{(\theta-\gamma)\olog j}{1-\gamma}$, the r.h.s of (\ref{eq:c132-}) can be lower bounded by
\begin{equation}\label{eq:c132}
    \inf_{\theta,\beta\in (0,1]\cap \QQ} \inf_{\zz\in \cl({\cal Z}_{J}(\theta,\beta))}\theta  \max\{F_{{\cal S}_J(\beta)}(\zz), I_{\theta}(\zz)\}.
\end{equation}

For lower bounding (\ref{eq:c132}), we consider two cases: (a) $J\neq [j]$; (b) $J=[2]$.

\noindent
\underline{\bf (a) The case where $J\neq [j]$.} 
 If $\zz\in\mathcal{Z}_{J}(\theta,\beta)$,
 \begin{align}
 \nonumber	\theta F_{S_{J}(\beta)}(\zz) &=\theta\inf_{\bl\in \cl(S_J(\beta))} \Psi(\bl,\zz)\\
 \nonumber			                     &\ge 2\inf\left\{\sum_{k\in J}\theta z_k(\lambda_k-\mu_k)^2:\min_{k\in J,k\neq 1}\lambda_k-\lambda_1\ge G(\beta)\right\}\\
\label{eq:c133} 							&= \frac{2}{j\olog j}\left(1-\theta \sum_{k\notin J}z_k\right)\beta\inf\left\{\sum_{k\in J}(\lambda_k-\mu_k)^2:\min_{k\in J,k\neq 1}\lambda_k-\lambda_1\ge G(\beta)\right\}\\
\nonumber&\ge  \frac{2\bar{\xi}_j}{j\olog j}(1-\theta \sum_{k\notin J}z_k)  \\
\label{eq:c134}&\ge \frac{2\bar{\xi}_j}{j\olog K},
 \end{align}
where the first inequality is due to Assumption~\ref{ass1}; the second equality uses the fact that $\forall k\in J,\, \theta z_k =\frac{(1-\theta\sum_{k'\notin J}z_{k'})\beta}{j\olog j}$ (as $\zz\in\mathcal{Z}_J(\theta,\beta)$); the third follows from Corollary~\ref{cor:crc G new} in Appendix~\ref{app:opt C}; the last one is a consequence of Proposition~\ref{prop ybound} with $i=j+1$ in Appendix~\ref{subapp:finite I}.

\underline{\bf (b) The case where $J=[j]$.} We will show that both $\frac{2\xi_j}{j\olog K}$ and $\frac{2\xi_j\olog (j+1)(1-\alpha_j)}{j\olog j\olog K}$ are lower bounds for the r.h.s. of (\ref{eq:c132}). The maximum of these becomes our lower bound. Together with the conclusion obtained in the case (a), we complete the proof of Lemma \ref{lem5}.

\underline{Lower bounding by $\frac{2\xi_j}{j\olog K}$}.
By Proposition \ref{prop:crc G} and Proposition \ref{prop ybound}, we can further lower bound the r.h.s. of (\ref{eq:c132}) as
\begin{align}
\nonumber	\theta F_{S_{[j]}(\beta)}(\zz) & \ge \frac{2}{j\olog j}\left(1-\theta \sum_{k>j}z_k\right)\beta\inf\left\{\sum_{k\in [j]}(\lambda_k-\mu_k)^2:\min_{k\in [j],k\neq 1}\lambda_k-\lambda_1\ge G(\beta)\right\}\\
\label{eq:c135}	&\ge 	\frac{2\xi_j}{j\olog j}\left(1-\theta \sum_{k>j}z_k\right)\\
\label{eq:c136}	& \ge \frac{2\xi_j}{j\olog K},
\end{align}
where the first inequality corresponds to (\ref{eq:c133}); the second inequality is due to Proposition~\ref{prop:crc G} in Appendix~\ref{app:opt C}; the last inequality uses Proposition~\ref{prop ybound} with $i=j+1$ in Appendix~\ref{subapp:finite I}.

\noindent
\underline{Lower bounding by $\frac{2\xi_j\olog (j+1)(1-\alpha_j)}{j\olog j\olog K}$}. One can derive another lower bound by using $I_\theta$. In Theorem~\ref{thm:crc main rate} with (b) in Appendix~\ref{subapp:local}, we show that $I_{\theta}$ is a valid lower semi-continuous rate function for an LDP upper bound (\ref{eq:rate UB}) for the process $\{\om (\theta T)\}_{T\ge 1}$. And in Corollary~\ref{cor:cj rate} in Appendix~\ref{subapp:rate C}, $I_\theta (\zz)\ge \underline{I}_\theta (\zz)$ for $\zz\in {\cal Z}_{[j]}(\theta,\beta)$, where
 \begin{equation}\label{eq:c137}
 \underline{I}_{\theta}(\zz)=\frac{\olog (j+1)}{\theta\olog K}   \left[\left(  (1+\zeta_j)\sqrt{\frac{\theta z_{\sigma(j+1)}}{ 1-\theta\sum_{k=j+2}^Kz_{\sigma(k)}  }}-\sqrt{\frac{1}{(j+1)\olog (j+1)}}    \right)_+\right]^2.
 \end{equation}

\noindent
Instead of using (\ref{eq:c136}), we lower bound $ F_{S_{[J]}(\beta)}(\zz)$ as:
\begin{align}\label{eq:c138}
\nonumber 	\theta F_{S_{[j]}(\beta)}(\zz)\ge \frac{2\xi_j}{j\olog j}\left(1-\theta \sum_{k\notin [j]}z_k\right)	 &=\frac{2\xi_j}{j\olog j} \left(1-\theta \sum_{k=j+2}^Kz_{\sigma(k)}  \right) \left(1-   \frac{\theta z_{\sigma(j+1)}}{1-\theta \sum_{k=j+2}^Kz_{\sigma(k)} }\right)\\
 	 &\ge \frac{2\xi_j\olog (j+1)}{j\olog j\olog K} \left(1-   \frac{\theta z_{\sigma(j+1)}}{1-\theta \sum_{k=j+2}^Kz_{\sigma(k)} }\right),
 \end{align}
 where the first inequality is from (\ref{eq:c135}) and the last inequality is due to Proposition~\ref{prop ybound} with $i=j+2$ in Appendix~\ref{subapp:finite I}. Since $\theta F_{S_{[j]}(\beta)}(\zz)$ and $\theta I_{\theta}(\zz)$ are lower bounded by the functions of $\alpha =  \frac{ \theta z_{\sigma(j+1)}}{1-\theta \sum_{k=j+2}^K z_{\sigma (k)}} $ given in (\ref{eq:c137}) and (\ref{eq:c138}), we have:
\begin{multline}\label{eq:c139}
     \inf_{\theta,\beta\in [0,1]\cap \QQ} \theta \inf_{\zz\in \cl({\cal Z}_{[j]}(\theta,\beta))} \max\{F_{{\cal S}_{[j]}(\beta)}(\zz), \underline{I}_{\theta}(\zz)\}\ge\\
     \frac{\olog (j+1)}{\olog K} \inf_{\alpha\in \RR} \max\left\{ \frac{2\xi_j\left(1-   \alpha \right)}{j\olog j} ,     \left[\left(  (1+\zeta_j)\sqrt{\alpha }-\sqrt{\frac{1}{(j+1)\olog (j+1)}}    \right)_+\right]^2 \right\}.
\end{multline}
 
By Lemma~\ref{lem:simple sol new1} in Appendix~\ref{app:maxmin}, (\ref{eq:c139}) is lower bounded by $\frac{2\xi_j\olog (j+1)(1-\alpha_j)}{j\olog j\olog K}$, where $\alpha_j$ is described in (\ref{eq:C alpha j}). 

\end{proof}



\subsection{Performance analysis of \cra}\label{app:cra}
\begin{reptheorem}{thm:cra}
	Let $\bm\in  [0,1]^K$. Under \cra, $\lowlim_{T\rightarrow \infty}\frac{1}{T}\log \frac{1}{\PP_{\bm}\left[\hat{\imath}\neq 1\right]}$ is larger than
	$$
	2\min_{j=2,\ldots,K}\left\{ \frac{\min\{\max \{ \frac{\psi_j\olog (j+1)(1-\alpha_{j}) \mathbbm{1}_{\{j\neq K\}} }{\olog j},\psi_j\},  \bar{\psi}_j\}}{j\olog K}\right\},  
	$$
	where $\alpha_{j}\in \RR$ is the real number such that 
	$$
		\frac{\psi_j\left(1-   \alpha_j \right)}{j\olog j}= \frac{j}{j+1}\left[\left(  (1+\varphi_j)\sqrt{\alpha_j }-\sqrt{\frac{1}{(j+1)\olog (j+1)}}    \right)_+\right]^2.
	$$
\end{reptheorem}
We upper bound $\PP_{\bm}\left[\ell_j=1\right]$ for (i) $j=K$; (ii) $j=2$; (iii) $j\in \{3,\ldots,K\}$. The upper bound for (i), presented in \ref{subapp:aK}, is the easiest to derive as the only possible allocation before one discards the first arm is uniform among all arms. The bound for (ii), presented in \ref{subapp:a2}, is the second easiest to derive as $\ell_2$ is decided only in the end, namely, in the $T$-th round. The upper bound for (iii), presented in \ref{subapp:aj}, is more involved since we have to consider all possible allocations and rounds. Overall, the analysis is very similar to that of \crc, and we just sketch the arguments below.

\subsubsection{Upper bound of $\PP_{\bm}\left[\ell_K=1\right]$}\label{subapp:aK}
\begin{lemma}
	Let $\bm\in  [0,1]^K$. Under \cra, 
	$$
	\lowlim_{T\rightarrow \infty}\frac{1}{T}\log \frac{1}{\PP_{\bm}\left[\ell_K= 1\right]}\ge  \frac{2\psi_K}{K\olog K}.
	$$
\end{lemma}

\begin{proof}
The proof follows the steps similar to those in the proof in Appendix~\ref{subapp:cK}. \\
\noindent
Applying Theorem (\ref{thm:11}) to 
\begin{equation} \label{eq:c212}
\sum_{t\ge \theta_0 T}^T	\PP_{\bm}\left[\ell_K=\ell(t)=1\right] \le  \sum_{t\ge \theta_0 T}^T	\PP_{\bm}\left[ \hat{\bm}(t) \in \mathcal{S}_{\frac{t}{T}},\om (t)\in \mathcal{X}_K  \right],
\end{equation}
with
$$
\mathcal{S}_{\theta}=\left\{\bl\in  [0,1]^K:\lambda_1\le \frac{\sum_{k=2}^K\lambda_k}{K-1}-G(\theta\olog K)\right\},\,\forall \theta\in (0,\frac{1}{\olog K}],
$$
we obtain
\begin{align}\label{eq:c213}  \lowlim_{T\to \infty} \frac{1}{T}\log \frac{1}{\PP_{\bm}[\ell_K=1]}
    \ge \min_{\theta\in [\theta_0,1]\cap \QQ} \theta \{F_{\mathcal{S}_{\theta}}(1/K,\ldots,1/K),I_\theta  (1/K,\ldots,1/K)\}.
\end{align}
\noindent
Next, from Theorem~\ref{thm:cra main rate} (a) in Appendix \ref{subapp:local}, we know that $I_\theta (1/K,\ldots,1/K)=\infty$ if $\theta>1/\olog K$. Finally, Assumption~\ref{ass1} and Proposition~\ref{prop:cra G} in Appendix \ref{app:opt A} yields that 
$$
\min_{\theta\in [\theta_0,1/\olog(K)]\cap \QQ} \theta \{F_{\mathcal{S}_{\theta}}(1/K,\ldots,1/K),I_\theta  (1/K,\ldots,1/K)\}\ge  \frac{2\psi_K}{K\olog K}.
$$
\end{proof}


\subsubsection{Upper bound of $\PP_{\bm}\left[\ell_2=1\right]$}\label{subapp:a2}
\begin{lemma}\label{lem7}
	Let $\bm\in  [0,1]^K$. Under \cra, 
	$$
	\lowlim_{T\rightarrow \infty}\frac{1}{T}\log \frac{1}{\PP_{\bm}\left[\ell_2= 1\right]}\ge   \frac{\min\{\max \{ 4\psi_2(1-\alpha_2),3\psi_2 \},3\bar{\psi}_2\}}{3\olog K},
	$$
	where $\alpha_2\in \RR$ is the real number such that
	\begin{equation}\label{eq:A alpha 2}
		3\psi_2\left(1-   \alpha_2 \right)= 4\left[\left(  (1+\varphi_2)\sqrt{\alpha_2 }-\frac{1}{2}    \right)_+\right]^2. 	
	\end{equation}
\end{lemma}

\begin{proof}

The proof follows the same steps as those of the proof in Appendix~\ref{subapp:c2}. \\
Applying Theorem~\ref{thm:Varadaham bandits} with ${\cal S}={\cal S}_D, W={\cal X}_D $, where
$$
\mathcal{S}_{D}=\left\{\bl\in  [0,1]^K:\min_{k\in [D]}\lambda_k\ge \lambda_1\right\},
$$
we get
\begin{equation}\label{eq:c222}
    \lowlim_{T\to \infty}\frac{1}{T}\log \frac{1}{\PP_{\bm}[\ell_2=1]}\ge \min_{D\in {\cal D}} \inf_{\om \in \cl ({\cal X}_D)} \max\{ F_{{\cal S}_{D}}(\om ), I_1(\om )\},
\end{equation}

\noindent
Then we exclude the points in $\cup_{i=2}^K{\cal X}_{2,i}(1)$ using Theorem~\ref{thm:cra main rate} (a) in Appendix~\ref{subapp:local} and we define 
$$
{\cal Z}_D (1,1)={\cal X}_D\setminus  \cup_{i=2}^K{\cal X}_{2,i}(1).
$$
Consider two cases: (a) $D\neq [2]$; (b) $D=[2]$.\\

\noindent
\underline{\bf(a) The case where $D\neq [2]$.} Using Corollary \ref{cor:cra G new} with $\beta =1,j=2$ in Appendix \ref{app:opt A} and Proposition \ref{prop ybound} with $\theta =1,\,j=2,\,i=3$ in Appendix \ref{subapp:infinite I} yields that 
$$
\inf_{\om \in \cl ({\cal X}_D)} \max\{ F_{{\cal S}_{D}}(\om ), I_1(\om )\}\ge\frac{\bar{\psi}_2}{\olog K}.
$$
\noindent
\underline{\bf (b) The case where $D= [2]$.} We show that both $\frac{\psi_2}{\olog K}$ and $ \frac{4\psi_2(1-\alpha_2)}{3\olog K}$ are lower bounds for $\lowlim_{T\rightarrow \infty}\frac{1}{T}\log \frac{1}{\PP_{\bm}\left[ \hat{\bm}(T)\in \mathcal{S}_{[2]},\om (T)\in   \mathcal{Z}_{[2]} \right]}$. The maximum of these hence becomes our lower bound. Together with the conclusion obtained in the case (a), we complete the proof of Lemma \ref{lem7}.

\noindent
\underline{Lower bounding by $ \frac{\psi_2}{\olog K}$.} Applying Proposition~\ref{prop:cra G} with $\theta =1, \beta=1, j=2$ in Appendix~\ref{app:opt A}, and Proposition~\ref{prop ybound} with $j=2,i=3$ in Appendix~\ref{subapp:finite I}, we can obtain:
$$
\inf_{\om \in \cl ({\cal X}_{[2]})} \max\{ F_{{\cal S}_{[2]}}(\om ), I_1(\om )\} \ge \frac{\psi_2}{\olog K}. 
$$
\noindent
\underline{Lower bounding by $\frac{4\psi_2(1-\alpha_2)}{3\olog K}$}.
In Theorem \ref{thm:cra main rate} (b), $j=2,\theta=\beta=1$ in Appendix \ref{subapp:local}, we show that $I_1$ is a valid rate function for an LDP upper bound (\ref{eq:rate UB}) for the process $\{\om (T)\}_{T\ge 1}$. And Corollary~\ref{cor:aj rate} in Appendix~\ref{subapp:rate A} show $I_1(\zz)>\underline{I}_1(\zz)$ for $\zz\in \mathcal{Z}_{[2]}$, where
\begin{equation}\label{eq:c227}
 \underline{I}_{1}(\zz)=\frac{16}{9\olog K}   \left[\left(  (1+\varphi_2)\sqrt{\frac{ z_{\sigma(3)}}{ 1-\sum_{k=4}^Kz_{\sigma(k)}  }}-\frac{1}{2}    \right)_+\right]^2,\,\forall \zz\in \mathcal{Z}_{[2]}.
\end{equation}

Also, we lower bound $ F_{S_{[2]}}(\zz)$ by Proposition~\ref{prop ybound} with $\theta=1,j=2,i=4$ in Appendix~\ref{subapp:finite I}:
\begin{align}\label{eq:c228}
	 	 F_{S_{[2]}}(\zz)\ge \frac{4\psi_2}{3\olog K} \left(1-   \frac{z_{\sigma(3)}}{1- \sum_{k=4}^Kz_{\sigma(k)} }\right).
\end{align}
 Observe that $F_{S_{[2]}}(\zz)$ and $I_{1}(\zz)$ are lower bounded by the functions of $\alpha =  \frac{  z_{\sigma(3)}}{1- \sum_{k=4}^K z_{\sigma (k)}} $ given in (\ref{eq:c227}) and (\ref{eq:c228}). Applying Lemma \ref{lem:simple sol new1} in Appendix~\ref{app:maxmin} yields that
\begin{align}\label{eq:c229}
\inf_{\om \in \cl ({\cal X}_{[2]})} \max\{ F_{{\cal S}_{[2]}}(\om ), \underline{I}_1(\om ) \}
&\ge  \frac{4\psi_2(1-\alpha_2)}{3\olog K}.
\end{align}
\end{proof}


\subsubsection{Upper bound for $\PP_{\bm}\left[\ell_j=1\right]$ for $j\in \{3,\ldots,K-1\}$}\label{subapp:aj}
\begin{lemma}\label{lem8}
	Let $\bm\in  [0,1]^K,\,j\in \{3,\ldots,K-1\}  $. Under \cra, 
	$$
	\lowlim_{T\rightarrow \infty}\frac{1}{T}\log \frac{1}{\PP_{\bm}\left[\ell_j= 1\right]}\ge     \frac{2\min\left\{\max \left\{ \frac{\psi_j\olog (j+1)(1-\alpha_{j})}{\olog j},\psi_j\right\},  \bar{\psi}_j\right\}}{j\olog K},	
	$$
	where $\alpha_j\in \RR$ is the real number such that
	\begin{equation}\label{eq:A alpha j}
		\frac{\psi_j\left(1-   \alpha_j \right)}{j\olog j}= \frac{j}{j+1}\left[\left(  (1+\varphi_j)\sqrt{\alpha_j }-\sqrt{\frac{1}{(j+1)\olog (j+1)}}    \right)_+\right]^2. 	
	\end{equation}
\end{lemma}

\begin{proof}
We proceed as in Appendix~\ref{subapp:cj}. We have:
\begin{equation}\label{eq:c231}
    \lowlim_{T\rightarrow \infty}\frac{1}{T}\log \frac{1}{\PP_{\bm}\left[\ell_j= 1\right]}\ge\min_{J\in {\cal J}}   \lowlim_{T\rightarrow \infty}\frac{1}{T}\log \frac{1}{\PP_{\bm}\left[\ell_j= 1,{\cal C}_j=J\right]}.
\end{equation}
We then introduce 
\begin{align*}
      \mathcal{S}_J(\beta)&=\left\{ \bl\in  [0,1]^K: \frac{\sum_{k\in J,k\neq 1}\lambda_k}{j-1}-\lambda_1\ge G(\beta)   \right\}\\
       \mathcal{X}_{J}(\theta,\beta)&=\left\{\zz\in \mathcal{X}_j :(\forall k\in J,\  x_k=\max_{k'\in [K]}x_{k'}),\ \frac{\theta \sum_{x\in J}z_k\olog j   }{1-\theta\sum_{k\notin J}x_{k}}=\beta\right\}.
\end{align*}
Theorem~\ref{thm:11} yields that for each $J\in {\cal J}$,
\begin{equation}\label{eq:c232}
\lowlim_{T\rightarrow\infty}\frac{1}{T}\log \frac{1}{\PP_{\bm}\left[\ell_j= 1,{\cal C}_j=J \right]}\ge  \inf_{\theta,\beta\in (0,1]\cap \QQ} \theta \inf_{\zz\in \cl({\cal Z}_{J}(\theta,\beta))} \max\{F_{{\cal S}_J(\beta)}(\zz), I_{\theta}(\zz)\},
\end{equation}
where ${\cal Z}_J (\theta,\beta)={\cal X}_J(\theta,\beta)\setminus  \cup_{i=j}^K{\cal X}_{j,i}(\theta).$

\noindent
\underline{\bf (a) The case where $J\neq [j]$.} Corollary~\ref{cor:cra G new} in Appendix \ref{app:opt A} and Proposition~\ref{prop ybound} with $i=j+1$ in Appendix~\ref{subapp:finite I} yields:
$$
		 \theta F_{{\cal S}_J(\beta)}(\zz)   \ge \frac{2\bar{\psi}_j}{j\olog K}.
$$
\noindent
\underline{\bf (b) The case where $J=[j]$.} We show both $\frac{2\psi_j}{j\olog K}$ and $\frac{2\psi_j\olog (j+1)(1-\alpha_j)}{j\olog j\olog K}$ are lower bounds of (\ref{eq:c232}). The maximum of these becomes our lower bound. 

\noindent
\underline{Lower bounding $\frac{2\psi_j}{j\olog K}$}. By Proposition~\ref{prop:cra G} and Proposition~\ref{prop ybound}, Proposition~\ref{prop:cra G} in Appendix~\ref{app:opt A} and Proposition~\ref{prop ybound} with $i=j+1$ in Appendix~\ref{subapp:finite I}, we have
	\begin{align}
		\theta F_{S_{J}(\beta)}(\zz) \ge \frac{2\psi_j}{j\olog K}. 
	\end{align}
	
\noindent
\underline{Lower bounding by $\frac{2\psi_j\olog (j+1)(1-\alpha_j)}{j\olog j\olog K}$}. A similar argument as above implies
\begin{equation}\label{eq:lll}
  (\ref{eq:c232})  \ge \frac{2\olog (j+1)}{\olog K} \inf_{\alpha\in \RR} \max\left\{ \frac{\psi_j\left(1-   \alpha \right)}{j\olog j} ,    \frac{j}{j+1} \left[\left(  (1+\varphi_j)\sqrt{\alpha }-\sqrt{\frac{1}{(j+1)\olog (j+1)}}    \right)_+\right]^2 \right\}.
\end{equation}
	By Lemma \ref{lem:simple sol new1} in Appendix \ref{app:maxmin}, (\ref{eq:c232}) is lower bounded by $\frac{2\psi_j\olog (j+1)(1-\alpha_j)}{j\olog j\olog K}$.

\end{proof}

\newpage
\section{Optimization Problems}\label{app:opt}

This section provides results related to the various optimization problems we encounter in the paper. In \ref{app:opt C}, we compute the $\xi_j$'s appearing in the performance guarantees of \sr and \crc, and prove other useful results. In \ref{app:opt A}, we  focus on computing the $\psi_j$'s, useful for the performance analysis of \cra. 

\subsection{Optimization problems for \sr and \crc}\label{app:opt C}
 Let $j\in \{2,\ldots, K\}$ and let $\bm\in \RR^K$ such that $\mu_1>\mu_2\ge\ldots\ge \mu_K$. Denote by $\xi_j$ the optimal value of the following optimization problem:
 \begin{align}\label{opt:1}
	\inf\left\{ \sum_{k=1}^j (\lambda_k-\mu_k)^2:\bl\in [0,1]^K, \lambda_1\le \min_{k\neq 1}\lambda_k\right\}.
\end{align}

\medskip
\noindent
We first show Proposition \ref{prop:xi}, restated below for convenience, and deduce some related results.
 
\begin{proposition}\label{prop:xi}
	We have:
	$$
	\xi_j=\left\{
	\begin{array}{ll}
	\sum_{k=1,j} \left(\mu_k-\frac{\mu_1+\mu_j}{2}\right)^2,&\hbox{ if }\mu_{j-1}\ge \frac{\mu_1+\mu_j}{2},\\
	\sum_{k=1,j-1,j} \left(\mu_k-  \frac{\mu_1+\mu_{j-1}+\mu_j}{3} \right)^2,&\hbox{ if }\mu_{j-1}< \frac{\mu_1+\mu_j}{2},\mu_{j-2}\ge \frac{\mu_1+\mu_{j-1}+\mu_j}{3}, \\
	\vdots&\vdots\\
	\sum_{k=1}^j \left(\mu_k-  \frac{\sum_{i=1}^j\mu_i}{j} \right)^2,&\hbox{ if }\mu_{j-1}< \frac{\mu_1+\mu_j}{2},\ldots, \mu_2< \frac{\mu_1+\mu_3+\ldots+\mu_j}{j-1}.
	\end{array}
	\right.
	$$
\end{proposition}
\begin{proof}
The objective function and the functions defining the constraints in (\ref{opt:1}) are all convex. There exists $\bl\in \RR^K$ s.t. all the constraints are strict (Slater condition). Hence we can identify the solution of (\ref{opt:1}) by just verifying the KKT conditions. The Lagrangian of the problem is
$$
\mathcal{L}_{\bm}(\bl,\eta_2,\ldots, \eta_j)=\frac{1}{2}\sum_{k=1}^j (\lambda_k-\mu_k)^2+\sum_{k=2}^j\eta_k(\lambda_1-\lambda_k),\,\hbox{for } (\bl,\be)\in \RR^K\times \RR^{j-1}_{\ge 0}.
$$
Let $(\bl^\star, \be^\star)$ be a saddle point of $\mathcal{L}$. It satisfies KKT conditions:
\begin{align}
\lambda_1^\star\le\lambda^\star_k,\,\,\hbox{for }k=2,\ldots,j,\tag{Primal Feasibility}\\
\eta^\star_k\ge 0,\,\,\hbox{for }k=2,\ldots,j,\tag{Dual Feasibility}\\
\lambda^\star_1-\mu_1+\sum_{k=2}^j\eta^\star_k=0;  \lambda^\star_k-\mu_k-\eta^\star_k=0,\,\hbox{for }k=2,\ldots,j,   \tag{Stationarity}\\
\eta^\star_k(\lambda_1^\star-\lambda_k^\star)=0,\,\hbox{for }k=2,\ldots,j.\tag{Complementarity}
\end{align} 
Let $i\in \{2,\ldots,j\}$ be the smallest index such that
$$
\mu_i<\frac{\mu_1+\sum_{i<k\le j}\mu_k}{j-i+1}.
$$
One can easily see the point $(\bl^\star,\be^\star)$ defined in (\ref{eq:blstar}) satisfies the KKT conditions listed above.
\begin{equation}\label{eq:blstar}
    \lambda^\star_k=\left\{
    \begin{array}{ll}
      \frac{\mu_1+\sum_{k=i}^j\mu_k}{j-i+2} ,  &\text{if }k=1,i,\ldots,j,  \\
        \mu_k, & \text{otherwise,}
    \end{array}
    \right.
     \eta^\star_k=\left\{
    \begin{array}{ll}
     \frac{\mu_1+\sum_{k=i}^j\mu_k}{j-i+2}-\mu_k  &\text{if }k=1,i,\ldots,j,  \\
        0, & \text{otherwise.}
    \end{array}
    \right.
\end{equation}

\end{proof}
We are now interested in quantifying the impact of $\bm$ on the value of $\xi_j$. We investigate this impact in the following two propositions.  

\begin{proposition}\label{prop:pxi}
Assume that $\xi_j=\sum_{b\in B}(\mu_b-A)^2$ for some $B\subseteq [j]$ and for $A=\frac{\sum_{b\in B} \mu_b}{\left|B\right|}$. Let $S$ be such that $S_1=\sum_{b\in B,b\neq 1} S_b$, where $S_b\ge 0$ for all $b\neq 1$. Consider another parameter $\bm'$ defined as $\mu_1'=\mu_1+S_1$ and $\mu_b'=\mu_b-S_b$ for all $b\in B,b \neq 1$. Then (i) $\frac{\sum_{b\in B}\mu_b'}{\left|B\right|}=A$; (ii)
    $$
    \inf\left\{\sum_{k\in B}(\lambda'_k-\mu_k')^2:\bl'\in  [0,1]^K,\lambda_1'\le \min_{b\in B}\lambda_b'  \right\}=\sum_{k\in B}(\mu'_b-A)^2.
    $$   
\end{proposition}
\begin{proof}
(i) is trivial. We now prove (ii). Using Proposition \ref{prop:xi} and the fact that $\xi_j=\sum_{b\in B}(\mu_b-A)^2$, we get that 
    \begin{equation}\label{eq:pxi1}
        \forall b\neq 1,b\in B,\, \mu_b<\frac{\mu_1+\sum_{k>b}\mu_k}{j-b+1}.
    \end{equation}
Also, as $S_b\ge 0$, 
\begin{align}
\nonumber    \mu'_b&=\mu_b-S_b\le \mu_b< \frac{\mu_1+\sum_{k>b}\mu_k}{j-b+1}\\
\nonumber&=\frac{1}{j-b+1}\left(\mu_1-S_1+\sum_{k>b}\mu_k+\sum_{k>b}S_k   \right)\\
\nonumber&=  \frac{\mu'_1+\sum_{k>b}\mu'_k}{j-b+1},
\end{align}
where the second inequality is due to (\ref{eq:pxi1}). By (i) and Proposition \ref{prop:xi} again, we conclude the proof.
\end{proof}
\medskip
\noindent

\begin{proposition}\label{prop:larger C}
	Consider the optimization problem (\ref{opt:1}) instantiated with another $\bm'\in  [0,1]^K$ which satisfies that $\mu_1'\ge \mu_1$ and $\mu'_k\le \mu_k$ for all $k=2,\ldots,j$, and denote its value by $\xi'_j$. Then $\xi'_j\ge \xi_j$.
\end{proposition}
\begin{proof}
Consider the Lagrangians of the two optimization problems: $\mathcal{L}_{\bm}$ and $\mathcal{L}_{\bm'}$. The corresponding Lagrange dual functions are:
	$g_{\bm}(\be)=\min_{\bl\in \RR^K}\mathcal{L}_{\bm}(\bl,\be)$ and $g_{\bm'}(\be)=\min_{\bl\in \RR^K}\mathcal{L}_{\bm'}(\bl,\be)$ and one can easily verify that
	\begin{align*}
	g_{\bm}(\be)&=\frac{1}{2}\left[ (\sum_{k=2}^j\eta_k)^2+\sum_{k=2}^j\eta_k^2\right]+\sum_{k=2}^j\eta_k(\mu_1-\mu_k-\sum_{i\neq k}\eta_i),\\
	g_{\bm'}(\be)&=\frac{1}{2}\left[ (\sum_{k=2}^j\eta_k)^2+\sum_{k=2}^j\eta_k^2\right]+\sum_{k=2}^j\eta_k(\mu'_1-\mu'_k-\sum_{i\neq k}\eta_i).
	\end{align*}
	Recall $\mu_1-\mu_k\le \mu_1'-\mu_k'$ and $\be\in \RR^{j-1}_{\ge 0}$, hence $g_{\bm}(\be)\le g_{\bm'}(\be)$ for all $\be\in \RR^{j-1}_{\ge 0}$. For (\ref{opt:1}), Slater condition holds clearly, hence strong duality follows (see e.g. \citep{boyd2004convex} Chapter 5.5.3). Thus, $\xi_j=\max_{\be\in \RR^{j-1}_+} g_{\bm}(\be)\le \max_{\be\in \RR^{j-1}_+} g_{\bm'}(\be)=\xi'_j$.
\end{proof}

\medskip
\noindent
The following result relates the function $G$ to $\xi_j$ and is instrumental in the proof of Theorem \ref{thm:crc}.
\begin{proposition}\label{prop:crc G}
	$\forall \beta\in (0,1],\, \bm\in [0,1]^K$, and $2\le j\le K$, one has
	\begin{equation}\label{eq:crc G}
	\beta\inf  \left\{  \sum_{k=1}^j (\lambda_k-\mu_k)^2 : \bl\in  [0,1]^K, \lambda_1\le \min_{k=2,\ldots,j} \lambda_k-G(\beta)\right\}\ge \xi_j.
	\end{equation}
\end{proposition}
\begin{proof}
    Let $B\subseteq [j]$ s.t. $\xi_j=\sum_{b\in B}(\mu_b-A)^2$, where $A=\frac{\sum_{b\in B}\mu_b}{\left|B\right|}$. Using the fact that $\sum_{k\notin B}(\mu_k-\lambda_k)^2\ge 0$ for all $\bl\in \RR^K$, one can deduce that  
\begin{align}\label{eq:ncrc 1}
\nonumber & \text{l.h.s. of (\ref{eq:crc G})}\ge \beta\inf\left\{  \sum_{b\in B} (\lambda_b-\mu_b)^2 :  \lambda_1  \le  \lambda_b-G(\beta),\forall b\in B,b\neq 1\right\}\\
\nonumber & \ge \beta\inf\left\{  \sum_{b\in B} (\lambda_b-\mu_b)^2 :  \lambda_1  \le  \lambda_b-(\mu_1-\mu_b)G(\beta),\forall b\in B,b\neq 1\right\}\\
&=\beta\inf\left\{  \sum_{b\in B} (\lambda_b-\mu_b)^2 :  \lambda_1 +(\mu_1-A)G(\beta) \le  \lambda_k- (A-\mu_b)G(\beta),\forall b\in B,b\neq 1\right\}
\end{align}
where the second inequality comes from $1\ge \mu_1-\mu_k $. Now introduce $\bl'$ and $\bm'$ as 
\begin{align*}
\lambda_1'=\lambda_1+(\mu_1-A)G(\beta)&,\,\,\lambda'_b=\lambda_b-(A-\mu_b)G(\beta),\forall b\neq 1,b\in B;\\
\mu_1'=\mu_1-(\mu_1-A)G(\beta)&,\,\,\mu'_b=\mu_b+(A-\mu_b)G(\beta),\forall b\neq 1,b\in B.
\end{align*}    
These allow us to write the r.h.s. of (\ref{eq:ncrc 1}) as the value of the following optimization problem:
\begin{equation}\label{eq:ncrc 2}
    \beta\inf \left\{\sum_{b\in B}(\lambda'_b-\mu'_b)^2: \lambda_1'\le \min_{b\in B}\lambda_b' \right\}.
\end{equation}
Applying Proposition \ref{prop:pxi} with $S_b=(A-\mu_b)G(\beta)$ for all $b\in B,b\neq 1$ yields that the value of (\ref{eq:ncrc 2}) is 
\begin{equation}\label{eq:ncrc 3}
    \beta \sum_{b\in B}(\mu'_b-A)^2 =\beta \left( (\mu_1+(\mu_1-A)G(\beta)-A)^2+\sum_{b\in B,b\neq 1}(A-\mu_b+(A-\mu_b)G(\beta))^2 \right).
\end{equation}
Recall that $G(\beta)=1/\sqrt{\beta}-1$. Hence, (\ref{eq:ncrc 3}) is larger than $\sum_{b\in B}(\mu_b-A)^2=\xi_j$.
\end{proof}

In Proposition \ref{prop:crc G}, the top-$j$ arms only are considered. We can prove similar results for any $J\in \mathcal{J}$, by combining Proposition \ref{prop:larger C} to the arguments of the previous proof.

\begin{corollary}\label{cor:crc G new}
	$\forall \beta\in (0,1],\, \bm\in [0,1]^K$, $2\le j\le K$, and $J\in \mathcal{J},\, J\neq [j] $, one has
	\begin{equation}\label{eq:crc G new}
		\beta\inf   \left\{  \sum_{k=1}^K (\lambda_k-\mu_k)^2 : \bl\in [0,1]^K, \lambda_1\le \min_{k\in J,k\neq 1} \lambda_k-G(\beta)\right\}\ge \bxi_j.
	\end{equation}
\end{corollary}
\begin{proof}
Let $J\in \mathcal{J},\, J\neq [j]$ be fixed, we denote the indexes in $J$ by $\{\tilde{1},\tilde{2},\ldots, \tilde{j}\}$ such that $\tilde{1}<\tilde{2}<\ldots<\tilde{j}$. One can repeat the argument in the proof of Proposition \ref{prop:crc G} to obtain that the l.h.s. of (\ref{eq:crc G new}) is larger than  
\begin{equation}\label{eq:crc G new1}
\inf\left\{   \sum_{k=1}^K(\lambda_k-\mu_k)^2: \bl\in [0,1]^K,\,   \lambda_1\le \min_{k\in J,k\neq 1}\lambda_k  \right\}.	
\end{equation}
Since every $J\in \mathcal{J}$ includes $1$, $\tilde{1}=1$. Also, since $J\neq [j]$, we have $\tilde{2}\ge 2,\ldots, \tilde{j} \ge j+1$. Because we assume that $\mu_1>\mu_2\ge \ldots\ge \mu_K$, Proposition \ref{prop:larger C} yields that the value of (\ref{eq:crc G new1}) is larger than $\bxi_j$.
\end{proof}

\medskip
\subsection{Optimization problem for \cra}\label{app:opt A}
The following proposition is the analogue of Proposition \ref{prop:larger C} for $\psi_j$.

\begin{proposition}\label{prop:larger A}
	Let $\bm'\in [0,1]^K$ such that $\mu_1'\ge \mu_1$ and $\mu'_k\le \mu_k$ for all $k=2,\ldots,j$. Define $\psi'_j=\frac{j-1}{j}(\mu'_{1}-\frac{\sum_{k=2}^j\mu'_{k}}{j-1})^2$. Then, $\psi'_j\ge \psi_j$.
\end{proposition}
\begin{proof}
	The result simply follows from the following inequality:
	$$
	\psi'_j=\frac{j-1}{j}(\mu'_{1}-\frac{\sum_{k=2}^j\mu'_{k}}{j-1})^2\ge \frac{j-1}{j}(\mu_{1}-\frac{\sum_{k=2}^j\mu'_{k}}{j-1})^2\ge \frac{j-1}{j}(\mu_{1}-\frac{\sum_{k=2}^j\mu_{k}}{j-1})^2=\psi_j.
	$$
\end{proof}

\medskip
\noindent
We use the following result in the proof of Theorem \ref{thm:cra}.

\begin{proposition}\label{prop:cra G}
	$\forall \beta\in (0,1],\bm\in [0,1]^K \hbox{ with }\mu_1>\mu_2\ge \ldots \ge \mu_K $, and $2\le j\le K$, one has
	\begin{equation}\label{eq:cra G}
	\beta\inf   \left\{  \sum_{k=1}^j (\lambda_k-\mu_k)^2 : \bl\in [0,1]^K, \lambda_1\le \frac{\sum_{k=2,\ldots,j} \lambda_k}{j-1}-G(\beta)\right\}\ge \psi_j,
	\end{equation}
	where $\psi_j=	\frac{j-1}{j}(\mu_{1}-\frac{\sum_{k=2}^j\mu_{k}}{j-1})^2,\,\forall j\in \{2,\ldots, K\}$, as introduced in Section \ref{sec:anal cr}. 
\end{proposition}
\begin{proof}
The Lagrangian of the optimization problem (\ref{eq:cra G}) is:
$$
\mathcal{L}(\bl,\eta)=\frac{\beta}{2}\sum_{k=1}^j (\lambda_k-\mu_k)^2+\eta(\lambda_1-\frac{\sum_{k=2}^j \lambda_k}{j-1}+G(\beta)),\,\hbox{for } (\bl,\eta)\in \RR^j\times \RR_{\ge 0}.
$$
Denote the saddle point of $\mathcal{L}$ by $( \bl^\star,\eta^\star)$. The KKT conditions are satisfied: 
\begin{align}
\lambda^\star_1\le \frac{\sum_{k=2}^j\lambda^\star_k}{j-1}-G(\beta)\hbox{ and }\eta^\star\ge 0,\tag{Feasibility}\\
	\beta (\lambda_1^\star-\mu_1)+\eta^\star=0,\hbox{ and }	\beta (\lambda_k^\star-\mu_k)-\frac{\eta^\star}{j-1}=0,\forall k\neq 1,\tag{Stationarity}\\
\eta^\star \left(	\lambda^\star_1 - \frac{\sum_{k=2}^j\lambda_k^\star }{j-1}+G( \beta ) \right)=0. \tag{Complementarity}
\end{align}
One can simply verify that if $\eta^\star=0$, stationarity and feasibility cannot hold simultaneously. Thus $\eta^\star>0$ and complementarity yield that $\lambda^\star_1 - \sum_{k=2}^j\lambda_k^\star /(j-1)+G( \beta )=0$. In conjunction with stationarity, we have  
\[
\eta^\star=\frac{\beta(j-1)}{j}\left(\mu_1-\frac{\sum_{k=2}^j\mu_k}{j-1}+G(\beta)\right), 
\]
and hence the value of (\ref{eq:cra G}) is 
\begin{equation}\label{eq:opt4}
	\frac{(j-1)\beta}{j}\left(\mu_1-\frac{\sum_{k=2}^j\mu_k}{j-1}+G(\beta)\right)^2.
\end{equation}
Recall that $G(\beta)=1/\sqrt{\beta}-1$ and $\bm\in [0,1]^K$. We deduce that $G(\beta)\ge (1/\sqrt{\beta}-1)(\mu_1-\frac{\sum_{k=2}^j\mu_k}{j-1})$, which is equivalent to $\mu_1-\frac{\sum_{k=2}^j\mu_k}{j-1}+G(\beta)\ge \frac{1}{\sqrt{\beta}}(\mu_1-\frac{\sum_{k=2}^j\mu_k}{j-1} )$ and hence (\ref{eq:opt4}) is larger than $\frac{j-1}{j}(\mu_1-\frac{\sum_{k=2}^j\mu_k}{j-1})^2$. 
\end{proof}
\noindent
As we obtained Corollary \ref{cor:crc G new} in Appendix  \ref{app:opt C}, combining Proposition \ref{prop:larger A} and the proof of Proposition \ref{prop:cra G} yields the following corollary.

\begin{corollary}\label{cor:cra G new}
	$\forall \beta\in (0,1]$, $\bm\in [0,1]^K \hbox{ with }\mu_1>\mu_2\ge \ldots \ge \mu_K $, $2\le j\le K$, and $J\in \mathcal{J},\, J\neq [j] $, one has
	\begin{equation}\label{eq:cra G new}
		\beta\inf   \left\{  \sum_{k\in [K]} (\lambda_k-\mu_k)^2 : \bl\in  [0,1]^K, \lambda_1\le \frac{\sum_{k\in J,k\neq 1} \lambda_k}{j-1}-G(\beta)\right\}\ge \bpsi_j.
	\end{equation}
\end{corollary}
\begin{proof}
	Let $J\in \mathcal{J},\, J\neq [j]$ be fixed. We denote the indexes in $J$ by $\{\tilde{1},\tilde{2},\ldots, \tilde{j}\}$ such that $\tilde{1}<\tilde{2}<\ldots<\tilde{j}$. One can repeat the arguments of the proof of Proposition \ref{prop:cra G} to obtain that the l.h.s. of (\ref{eq:cra G new}) is larger than  
	\begin{equation}\label{eq:cra G new1}
				\frac{j-1}{j}(\mu_{\tilde{1}}-\frac{\sum_{k=2}^j\mu_{\tilde{k}}}{j-1})^2.
	\end{equation}
	Note that every $J\in \mathcal{J}$ containing $1$ satisfies $\tilde{1}=1$, and that since $J\neq [j]$, we have $\tilde{2}\ge 2,\ldots, \tilde{j} \ge j+1$. As we assume that $\mu_1>\mu_2\ge \ldots\ge \mu_K$, Proposition \ref{prop:larger A} yields that the value of (\ref{eq:cra G new1}) is larger than $\bxi_j$.
\end{proof}

\newpage
\section{LDP for the sampling process under \sred}\label{app:allocation}

In this section, we are interested in deriving an LDP for the process $\{\om(\theta T)\}_{T\ge 1}$ for a fixed $\theta\in (0,1]\cap \mathbb{Q}$ under \sred. More precisely, we look for a function $I_\theta (\cdot)$ which satisfies an LDP upper bound (\ref{eq:rate UB}) on $\mathcal{X}_j$ for some fixed $j\in \{1 ,\ldots,K\}$, where 
\begin{align}\label{eq:xj}
	\mathcal{X}_j =  \left\{ \xx\in \Sigma : \exists \sigma: [K] \mapsto [K] \hbox{ s.t. } x_{\sigma (1)}=\ldots=x_{\sigma (j)}>x_{\sigma (j+1)}> \ldots > x_{\sigma (K)}> 0\right\}.
\end{align}
For convenience, we define $x_{\sigma(K+1)}=0$. For any $j$, we also define
\begin{equation}\label{eq:yj2}
\mathcal{X}_{j,i}(\theta)=	\left\{\xx\in \mathcal{X}_j: 	\theta  x_{\sigma (i)}i\olog i> 1-\theta \sum_{k=i+1}^Kx_{\sigma(k)}	   \right\},\,\forall i\in \{j,\ldots,K\},
\end{equation}
where the permutation $\sigma$ depends on $\xx$ as in the definition of ${\cal X}_j$ (\ref{eq:xj}).

It is important to remark that when $\theta T$ is not an integer, $\om(\theta T)$ is not defined. Hence in the following, when we write $\lowlim_{T\to\infty} f(\mathbb{P}_{\bm}[\om(\theta T)\in F])$, we actually mean $\lowlim_{T\to\infty: \theta T\in \mathbb{N}} f(\mathbb{P}_{\bm}[\om(\theta T)\in F])$.

Deriving an LDP upper bound (\ref{eq:rate UB}) is not easy in general, and to this aim, we first introduce a useful sufficient condition in \ref{subapp:local}.

\subsection{A sufficient condition towards an LDP upper bound (\ref{eq:rate UB})}\label{subapp:local}
The following condition will be useful in our analysis, in particular in this section. This condition is similar to those presented in Chapter 2 in \cite{varadhan2016large}. We say that 
$\{Y(t)\}_{t\ge 1}$ satisfies an {\it LDP local upper bound with rate function $I$ at point $\yy\in {\cal Y}$} if:\\
	\begin{equation}
		\lowlim_{\delta\rightarrow 0}\lowlim_{t\rightarrow \infty}\frac{1}{t}\log \frac{1}{\PP[Y(t)\in B(\yy,\delta)]} \ge I(\yy), \label{eq:local UB}
	\end{equation}
	where $B(\yy,\delta)$ is the open ball with center $\yy$ and radius $\delta$.

\begin{lemma}\label{lem:local}
	Suppose $\mathcal{Y}$ is compact and $\{Y(t)\}_{t\ge 1}$ satisfies an LDP local upper bound (\ref{eq:local UB}) with a lower semi-continuous function $I$ at all $\yy\in \mathcal{Y}$, then  $\{Y(t)\}_{t\ge 1}$ satisfies an LDP upper bound (\ref{eq:rate UB}).
\end{lemma}
\begin{proof}
	Let $C\subseteq \mathcal{Y}$ be a closed (and hence compact) set, and $s=\inf_{\yy\in C} I(\yy)$. We prove $\lowlim_{t\rightarrow \infty }\frac{1}{t}\log \frac{1}{\PP\left[Y(t)\in C\right]  } \ge s$ if (i) $s=\infty$ and if (ii) $s<\infty$ separately.\\
	
	\medskip
	\noindent
	\underline{\bf (i) If $s=\infty$.} 
	Let $M>0$ and $\yy\in C$. As $I(\yy)=\infty$, and since $I$ is lower semi-continuous, there exists $\delta_{\yy}>0$ s.t. 
	\begin{equation}\label{eq:alt 1}
	\lowlim_{t\rightarrow \infty }\frac{1}{t}\log \frac{1}{\PP\left[Y(t)\in B(\yy,\delta_{\yy})\right]  } \ge M.		
	\end{equation}
	 Now observe that $C\subseteq \cup_{\yy\in C}B(\yy,\delta_{\yy})$. The compactness of $C$ implies that we can find $N\in \NN$, and $\{\yy_1,\ldots,\yy_N\}$ such that $C\subseteq \cup_{i=1}^N B(\yy_i,\delta_{\yy_i})$, which directly yields that 
	 \begin{equation}\label{eq:alt 2}
	 	 \PP\left[Y(t)\in C\right]\le \sum_{i=1}^{N}  \PP\left[Y(t)\in B(\yy_i,\delta_{\yy_i})\right]\le N\max_{i\in [N]} \PP\left[Y(t)\in B(\yy_i,\delta_{\yy_i})\right].
	 \end{equation} 
	Using a simple rearrangement in (\ref{eq:alt 1}) and (\ref{eq:alt 2}), we then have $\lowlim_{t\rightarrow \infty }\frac{1}{t}\log \frac{1}{\PP\left[Y(t)\in C\right]  } \ge M$. As $M$ can be taken arbitrarily large, the proof is completed.

	\medskip
	\noindent
	\underline{\bf (ii) If $s<\infty$.}
	Let $\epsilon\in (0,s/2)$ and $\yy\in C$. As $I(\yy)\ge s$ and since $I$ is lower semi-continuous,  there exists $\delta_{\yy}>0$ such that 
	\begin{equation}\label{eq:alt 3}
		\lowlim_{t\rightarrow \infty }\frac{1}{t}\log \frac{1}{\PP\left[Y(t)\in B(\yy,\delta_{\yy})\right]  } \ge s-\epsilon.		
	\end{equation}
	Now observe that $C\subseteq \cup_{\yy\in C}B(\yy,\delta_{\yy})$. The compactness of $C$ implies that we can find $N\in \NN$, $\{\yy_1,\ldots,\yy_N\}$ such that $C\subseteq \cup_{i=1}^NB(\yy_i,\delta_{\yy_i})$, which directly yields that 
	\begin{equation}\label{eq:alt 4}
		\PP\left[Y(t)\in C\right]\le \sum_{i=1}^{N}  \PP\left[Y(t)\in B(\yy_i,\delta_{\yy_i})\right]\le N\max_{i\in [N]} \PP\left[Y(t)\in B(\yy_i,\delta_{\yy_i})\right].
	\end{equation} 
	Using a simple rearrangement in (\ref{eq:alt 3}) and (\ref{eq:alt 4}), we then have $\lowlim_{t\rightarrow \infty }\frac{1}{t}\log \frac{1}{\PP\left[Y(t)\in C\right]  } \ge s-\epsilon$. As $\epsilon$ can be taken arbitrarily small, the proof is completed. 	
\end{proof}

\medskip
\noindent
We apply Lemma \ref{lem:local} to the process $\{\om (\theta T)\}_{T\ge 1}$. The latter has values in $\Sigma$, a compact set. To derive an LDP upper bound for this process (such an LDP upper bound is required to apply Theorem \ref{thm:Varadaham bandits}), we just need to establish at all points in $\Sigma$ a local LDP upper bound.


The following two theorems state that $\{\om (\theta T)\}_{T\ge 1}$ under \crc and \cra satisfies a local LDP upper bound. 

\begin{theorem}\label{thm:crc main rate}[Local LDP upper bound for \crc]
    For $\theta \in (0,1]\cap \QQ$, we define $I_{\theta}$ as follows.\\
(a) If $\xx\in {\cal X}_1\cup(\cup_{j=2}^K\cup_{i=j}^K\mathcal{X}_{j,i}(\theta))$, then $I_{\theta}(\xx)=\infty$;\\
(b) If $\exists j\in \{2,\ldots, K\}$ such that $\xx\in \mathcal{X}_{j}\setminus \cup_{j=2}^K\cup_{i=j}^K\mathcal{X}_{j,i}(\theta) $, then
$$
I_{\theta}(\xx)= \max_{p=j,\ldots,K-1}   2x_{\sigma (p+1)}\inf_{\bl\in \mathcal{S}_p(\xx)}\sum_{k=1}^{p+1}(\lambda_{\sigma(k)}-\mu_{\sigma (k)})^2,
$$
where $\mathcal{S}_p(\xx)$ is defined in (\ref{eq:Spc});\\
(c) If $\mathcal{V}= \cup_{k=1}^K{\cal X}_k$, and $\xx\in \cl ({\cal V})\setminus {\cal V}$, then 
$$
I_{\theta}(\xx)= \inf\{ \lowlim_{s\rightarrow\infty} I_{\theta}(\xx^{(s)}):\{\xx^{(s)}\}_{s\in \NN}\subset {\cal V}, \xx^{(s)} \to \xx \hbox{ as }s\to\infty\};
$$
Then the process $\{\om (\theta T)\}_{T\ge 1}$ under \crc satisfies an LDP upper bound (\ref{eq:rate UB}) with the rate function $I_{\theta}$, and $I_{\theta}$ is lower semi-continuous.
\end{theorem}
\begin{proof}
In view of Lemma \ref{lem:local}, the theorem holds if we are able to show that $\{\om (\theta T)\}_{T\ge 1}$ satisfies a local LDP upper bound with $I_\theta$ and if $I_\theta$ is lower semi-continuous. The first part is established below in Lemma~\ref{lem:x1}, Lemma~\ref{lem:yj2}, Lemma~\ref{lem:yj3}, and Theorem~\ref{thm:Cj}. \\
For the second part, we first verify the lower semi-continuity of $I_\theta$ restricted to $\cup_{j=1}^K{\cal X}_j$, and then apply Lemma~\ref{lem:V} with $f=I_\theta$ to establish the lower semi-continuity of $I_\theta$ in $\Sigma$. Let $\xx\in \cup_{j=1}^K{\cal X}_j$. If $\xx\in {\cal X}_1$, lower semi-continuity directly follows from the fact ${\cal X}_1$ is open and $I_{\theta}(\xx)=\infty$ for $\xx\in {\cal X}_1$. We then consider $\xx\in {\cal X}_{j}$ for some $j=2,\ldots, K$. By definition, there is $\sigma :[K]\mapsto[K]$ such that $x_{\sigma(1)}=\ldots =x_{\sigma (j)}>x_{\sigma (j+1)}>\ldots>x_{\sigma (K)}$. By taking $\delta<\min_{i\ge j}\{x_{\sigma(i)}-x_{\sigma (i+1)}\}/2$, we have $\xx'\in \cup_{q=1}^j{\cal X}_q$ if $\left\|\xx'-\xx\right\|_\infty<\delta$, and $I_\theta(\xx')\ge \max_{p=j,\ldots,K-1}   2x'_{\sigma (p+1)}\inf_{\bl\in \mathcal{S}_p(\xx')}\sum_{k=1}^{p+1}(\lambda_{\sigma(k)}-\mu_{\sigma (k)})^2$ as a consequence. Now as verified in Lemma~\ref{lem:verify c} in Appendix~\ref{app:continuity}, the mapping $\xx\mapsto 2x_{\sigma (p+1)}\inf_{\bl\in \mathcal{S}_p(\xx)}\sum_{k=1}^{p+1}(\lambda_{\sigma(k)}-\mu_{\sigma (k)})^2$ is continuous, we hence deduce that $I_\theta$ is lower semi-continuity at $\xx$.
\end{proof}
\medskip
\begin{theorem}\label{thm:cra main rate}[Local LDP upper bound for \cra]
    For $\theta \in (0,1]\cap \QQ$, we define $I_{\theta}$ as follows.\\
(a) If $\xx\in {\cal X}_1\cup(\cup_{j=2}^K\cup_{i=j}^K\mathcal{X}_{j,i}(\theta))$, then $I_{\theta}(\xx)=\infty$;\\
(b) If $\exists j\in \{2,\ldots, K\}$ such that $\xx\in \mathcal{X}_{j}\setminus \cup_{j=2}^K\cup_{i=j}^K\mathcal{X}_{j,i}(\theta) $, then
$$
I_{\theta}(\xx)= \max_{p=j,\ldots,K-1}   2x_{\sigma (p+1)}\inf_{\bl\in \mathcal{S}_p(\xx)}\sum_{k=1}^{p+1}(\lambda_{\sigma(k)}-\mu_{\sigma (k)})^2,
$$
where $\mathcal{S}_p(\xx)$ is defined in (\ref{eq:Spa});\\
(c) If $\mathcal{V}= \cup_{k=1}^K{\cal X}_k$, and $\xx\in \cl ({\cal V})\setminus {\cal V}$, then 
$$
I_{\theta}(\xx)= \inf\{ \lowlim_{s\rightarrow\infty} I_{\theta}(\xx^{(s)}):\{\xx^{(s)}\}_{s\in \NN}\subset {\cal V}, \xx^{(s)} \to \xx \hbox{ as }s\to\infty\};
$$
Then the process $\{\om (\theta T)\}_{T\ge 1}$ under \cra satisfies an LDP upper bound (\ref{eq:rate UB}) with the rate function $I_{\theta}$, and $I_{\theta}$ is lower semi-continuous.
\end{theorem}
\begin{proof}
     In view of Lemma \ref{lem:local}, the theorem is deduced if we are able to show that $\{\om (\theta T)\}_{T\ge 1}$ satisfies a local LDP upper bound with $I_\theta$. This is established below in Lemma~\ref{lem:x1}, Lemma~\ref{lem:yj2}, Lemma~\ref{lem:yj3}, and Theorem~\ref{thm:Aj}. \\
We then verify the lower semi-continuity of $I_\theta$ restricted to $\cup_{j=1}^K{\cal X}_j$, and then applying Lemma~\ref{lem:V} with $f=I_\theta$ yields the lower semi-continuity of $I_\theta$ in $\Sigma$. Let $\xx\in \cup_{j=1}^K{\cal X}_j$. If $\xx\in {\cal X}_1$, lower semi-continuity directly follows from the fact ${\cal X}_1$ is open and $I_{\theta}(\xx)=\infty$ for $\xx\in {\cal X}_1$. We then consider $\xx\in {\cal X}_{j}$ for some $j=2,\ldots, K$. By definition, there is $\sigma :[K]\mapsto[K]$ such that $x_{\sigma(1)}=\ldots =x_{\sigma (j)}>x_{\sigma (j+1)}>\ldots>x_{\sigma (K)}$. By taking $\delta<\min_{i\ge j}\{x_{\sigma(i)}-x_{\sigma (i+1)}\}/2$, we have $\xx'\in \cup_{q=1}^j{\cal X}_q$ if $\left\|\xx'-\xx\right\|_\infty<\delta$ and $I_\theta(\xx')\ge \max_{p=j,\ldots,K-1} 2x'_{\sigma (p+1)}\inf_{\bl\in \mathcal{S}_p(\xx')}\sum_{k=1}^{p+1}(\lambda_{\sigma(k)}-\mu_{\sigma (k)})^2, $ as a consequence. Now as verified in Lemma~\ref{lem:verify a} in Appendix~\ref{app:continuity}, the mapping $\xx\mapsto 2x_{\sigma (p+1)}\inf_{\bl\in \mathcal{S}_p(\xx)}\sum_{k=1}^{p+1}(\lambda_{\sigma(k)}-\mu_{\sigma (k)})^2$ is continuous, we hence deduce that $I_\theta$ is lower semi-continuity at $\xx$.
     
\end{proof}

\medskip
\begin{lemma}\label{lem:V}
    Suppose ${\cal V}\subseteq \Sigma$ is the set such that $\cl ({\cal V})=\Sigma$, and $f:{\cal V}\to \RR$ is a lower semi-continous function. If we extend $f$ to $\Sigma$ by defining 
    $$
    \bar{f}(\om)=\left\{\begin{array}{lc}
       f(\om),  & \text{ if }\om\in {\cal V},  \\
       \inf\{ \lowlim_{s\rightarrow\infty} f(\om^{(s)}):\{\om^{(s)}\}_{s\in \NN}\subset {\cal V}, \om^{(s)} \to \om \hbox{ as }s\to\infty\},
 & \text{ otherwise,}
    \end{array}\right.
    $$
    then $\bar{f}$ is a lower semi-continous function in $\Sigma$.
\end{lemma}
\begin{proof}
    By the definition of $\bar{f}$ and the fact $\cl ({\cal V})=\Sigma$, 
    \begin{equation}\label{eq:v1}
        \forall \varepsilon >0,\,\forall\delta>0,\,\forall\om\in \Sigma,\,\exists\xx\in {\cal V}\text{ such that }f(\xx)<\bar{f}(\om)+\epsilon\text{ and }\left\|\xx-\om\right\|_{\infty}<\delta.
    \end{equation}
    Next suppose on the contrary, $\bar{f}$ is not lower semi-continous at some $\om\in \Sigma$, that is, $\exists \{\om^{(s)}\}\subset \Sigma$ such that $\om^{(s)}\rightarrow\om$ as $s\rightarrow\infty$ and $\lim_{s\rightarrow\infty}\bar{f}(\om^{(s)})<\bar{f}(\om)$. Let $\eta=\bar{f}(\om)-\lim_{s\rightarrow\infty}\bar{f}(\om^{(s)})>0$. For each $s\in \NN$, (\ref{eq:v1}) implies that there is $\xx^{(s)}\in {\cal V}$ such that $\left\|\xx^{(s)}-\om^{(s)}\right\|_{\infty}<1/s$ and $f(\xx^{(s)})<\bar{f}(\om^{(s)})+\eta/2$. Hence,
$$
\lowlim_{s\rightarrow\infty}f(\xx^{(s)})\le\lowlim_{s\rightarrow\infty}\bar{f}(\om^{(s)})+\frac{\eta}{2}<\bar{f}(\om) ,
$$
which contradicts the lower semi-continuity of $f$ if $\om\in {\cal V}$ and the definition of $\bar{f}$ if $\om\notin {\cal V}$.

\end{proof}

\subsection{Local LDP upper bound on \texorpdfstring{$\cup_{i=j}^K\mathcal{X}_{j,i}(\theta)$}{}}\label{subapp:infinite I}
Let $\theta\in (0,1]\cap \mathbb{Q}$ and $j\in \{2,\ldots, K\}$. Here, we first prove the result on $\mathcal{X}_{j,j}(\theta)$ in Lemma \ref{lem:yj2} and that on $\mathcal{X}_{j,i}(\theta)$ for any $i>j$ in Lemma \ref{lem:yj3}. Note the results in this subsection are valid for both \crc and \cra.


\begin{lemma}\label{lem:x1}
Let $\theta\in (0,1]\cap \mathbb{Q},\,\xx\in {\cal X}_1$, the process $\{\om (\theta T)\}_{T\ge 1}$ satisfies an LDP local upper bound (\ref{eq:local UB}) with $I_{\theta} (\xx)=\infty$ at $\xx\in \mathcal{X}_{1}$.  
\end{lemma}
\begin{proof}
    For $\xx\in {\cal X}_1$, there exists $\sigma:[K]\mapsto [K]$ such that $x_{\sigma(1)}>x_{\sigma(2)}>\ldots>x_{\sigma(K)}$. Let $\delta<\min_{k=1,\ldots, K-1} \{x_{\sigma (k)}-x_{\sigma(k+1)}\} $ and $T>\frac{2}{\theta\delta}$ such that $\theta T\in \NN$, we show that $\PP_{\bm}[\om (\theta T)\in B(\xx,\delta)]=0$, which directly completes the proof. \\

    If $\om (\theta T)\in B(\xx,\delta)$, then we have $\omega_{\sigma (1)}(\theta T)>\omega_{\sigma (2)}(\theta T)>\ldots >\omega_{\sigma (K)}(\theta T)$ and 
    \begin{equation}\label{eq:x11}
         \min_{k=1,\ldots, K-1}\{N_{\sigma (k)}(\theta T)-N_{\sigma (k+1)}(\theta T)\}=\theta T\min_{k=1,\ldots, K-1}\{\omega_{\sigma (k)}(\theta T)-\omega_{\sigma (k+1)}(\theta T)\}>2.
    \end{equation}
    Because \sred always pulls arms in the candidate set in a round-robin manner, (\ref{eq:x11}) will never happen. Hence, $\PP_{\bm}[\om (\theta T)\in B(\xx,\delta)]=0$.
\end{proof}

\medskip
\noindent

\begin{lemma}\label{lem:yj2}
	Let $j\in \{2,\ldots, K\}$. The process $\{\om (\theta T)\}_{T\ge 1}$ satisfies an LDP local upper bound (\ref{eq:local UB}) with $I_{\theta} (\xx)=\infty$ at $\xx\in \mathcal{X}_{j,j}(\theta)$. 	
\end{lemma}
\begin{proof}
 Let $\xx\in \mathcal{X}_{j,j}(\theta)$. We show that there exists $\delta_{\theta ,\xx}>0$ and $ T_{\theta,\xx}\in \NN$ s.t. if $T\ge T_{\theta,\xx}$ and $\delta <\delta_{\theta,\xx}$, then $\om (\theta T)\notin B(\xx,\delta)$ almost surely. As a consequence, $\PP_{\bm}[\om (\theta T)\in B(\xx,\delta)]=0$, and $I_{\theta} (\xx)=\infty$. We decompose the proof into three steps.

\medskip 
\noindent
\underline{\it 1. Defining $\delta_{\theta,\xx}$ and $T_{\theta,\xx}$.}
We introduce the two functions, $f_1,f_2: [0,1]\times \Sigma\mapsto \RR $: 
	\begin{align*}
		f_1(\theta',\xx')&= \theta'\sum_{k=1}^j x'_{\sigma(k)}\olog j-\theta' \sum_{k=j+1}^K x'_{\sigma(k)}-1,\\
			f_2(\theta',\xx')&=\min_{k\le j} x'_{\sigma(k)}-\max_{k\ge j+1} x'_{\sigma(k)} .
	\end{align*}
Since $\xx\in \mathcal{X}_{j,j}(\theta)$, we have $c=\min \{f_1(\theta,\xx),f_2(\theta,\xx)\}>0$. Leveraging the fact that $f_1,f_2$ are continuous, we can find $\delta_{\theta,\xx}\in (0,\frac{1}{3j})$ s.t.
	\begin{equation}\label{eq:nyj2 1}
	 \text{if }\left|\theta'-\theta\right| < 3j\delta_{\theta,\xx}\text{ and }\left\|\xx'-\xx\right\|_{\infty}< 7j\delta_{\theta,\xx},\text{ then }\min \{f_1(\theta',\xx'), f_2(\theta',\xx')\}>\frac{c}{2}.	
	\end{equation}
 \noindent
 We also define $T_{\theta,\xx}=\max\{\lceil \frac{4}{\theta c}\rceil,\lceil \frac{1}{\delta_{\theta,\xx}}\rceil\}$.

 \medskip 
\noindent
 \underline{\it 2. We prove that for $\delta<\delta_{\theta,\xx}$ and $T>T_{\theta.\xx}$, $\om (\theta T)\notin B(\xx,\delta)$ a.s..} 
 We proceed by contradiction. Assume $\om(\theta T) \in  B(\xx,\delta)$. From (\ref{eq:nyj2 1}), we have $f_2(\theta,\om(\theta T))>c/2$. It then directly yields that
\begin{align}\label{eq:nyj2 2}
\nonumber	\min_{k\le j} N_{\sigma(k)}(\theta T)-\max_{k\ge j+1} N_{\sigma(k)}(\theta T) &=\theta T\left[ \min_{k\le j} \omega_{\sigma(k)}(\theta T)-\max_{k\ge j+1} \omega_{\sigma(k)}(\theta T)\right]\\
\nonumber&=\theta T f_2(\theta,\om (\theta T))\\	
	&>\frac{\theta cT}{2}>\frac{\theta cT_{\theta,\xx}}{2} \ge 2,
\end{align}
where the last inequality follows from $T_{\theta,\xx}>\frac{4}{\theta c}$. Observe that \sred always pulls the arms in the candidate set in a round-robin manner (the maximal difference of pulling amounts among the candidate set is at most $1$), and \sred stops pulling an arm $k$ after $k$ is removed from the candidate set. Thus, from (\ref{eq:nyj2 2}), we deduce several facts: (i) $\mathcal{C}_j = \{ \sigma(1),\ldots,\sigma (j)\}$; (ii) before the round $\tau=j\min_{k\le j}N_{\sigma(k)}(\theta T)+\sum_{k>j}N_{\sigma (k)} (\theta T)$, the arm $\ell_j$  to be removed has not yet been decided; (iii) 
\begin{equation*}
N_{\sigma (k)}(\tau-j)=\left\{
\begin{array}{ll}
    \min_{k\le j}N_{\sigma (k)}(\theta T)-1, &\text{ if }k\le j,\\
     N_{\sigma (k)}(\tau-j)=N_{\sigma (k)}(\theta T), &\text{ if }k>j.
\end{array}\right.
\end{equation*}
However, we will show in the next step that in the round $\tau-j $, the condition for discarding an arm in $\mathcal{C}_j$ is triggered. In other words, $\ell_j=\ell(\tau-j)$ is removed from $\mathcal{C}_j$ in the round $\tau-j$, which contradicts (ii). 

\medskip
\noindent
\underline{\it 3. The condition for discarding an arm in round $\tau-j$ is triggered.} 
First, from (iii) in Step 2,
\begin{equation}\label{eq:nyj2 7}
    N_{\sigma (1)}(\tau-j)= N_{\sigma (2)}(\tau-j)=\ldots= N_{\sigma (j)}(\tau-j).
\end{equation}

Then, using (iii) in Step 2 again yields that
\begin{equation}\label{eq:nyj2 5}
    \min_{k\le j} N_{\sigma(k)}(\tau-j)-\max_{k\ge j+1} N_{\sigma(k)}(\tau-j) =\min_{k\le j} N_{\sigma(k)}(\theta T)-1-\max_{k\ge j+1} N_{\sigma(k)}(\theta T)>1,
\end{equation}
where the last inequality comes from (\ref{eq:nyj2 2}).
Finally, observe that 
\begin{align}\label{eq:nyj2 3}
\nonumber  \theta -\frac{\tau-j}{T}&=\frac{\sum_{k\in [j]}N_{\sigma (k)}(\theta T) -j\min_{k\in [j]} N_{\sigma (k)}(\theta T)  }{T}+\frac{j}{T}\\  
 \nonumber &\le j\left[\max_{k\in [j]} \omega_{\sigma (k)}(\theta T) -x_{\sigma (j)}+ x_{\sigma (j)}-\min_{k\in [j]} \omega_{\sigma (k)}(\theta T) \right]+j\delta_{\theta,\xx}\\
  &\le j(2\delta+\delta_{\theta,\xx})< 3j \delta_{\theta,\xx},
\end{align}
where the first inequality is due to $T\ge 1/\delta_{\theta,\xx}$; the second inequality follows from $\om (\theta T)\in B(\xx,\delta) $; the last inequality is obtained using $\delta<\delta_{\theta,\xx}$. Combining Lemma \ref{lem:theta delta} (see below) with $\delta=3j\delta_{\theta,\xx}$ and (\ref{eq:nyj2 3}) yields that $\left\|\om (\tau-j)-\om (\theta T) \right\|_{\infty}\le 6j\delta_{\theta,\xx}$, hence
\begin{equation}\label{eq:nyj2 4}
    \left\|\om (\tau-j)-\xx\right\|_{\infty}\le   \left\|\om (\tau-j)-\om (\theta T)\right\|_{\infty}+ \left\|\om (\theta T)- \xx\right\|_{\infty} \le 7j \delta_{\theta,\xx}.
\end{equation}
From (\ref{eq:nyj2 1})-(\ref{eq:nyj2 3})-(\ref{eq:nyj2 4}), we get $f_1(\frac{\tau-j}{T},\om (\tau-j )) >\frac{c}{2}$. Thus,
	\begin{equation}\label{eq:nyj2 6}
		G\left(\frac{\sum_{k=1}^j N_{\sigma (k)}(\tau-j)\olog j}{ T-\sum_{k=j+1}^KN_{\sigma(k)}(\tau-j) }   \right)=G\left(\frac{\frac{\tau-j}{T}\sum_{k=1}^j \omega_{\sigma (k)}(\tau-j)\olog j}{ 1-\frac{\tau-j}{T}\sum_{k=j+1}^K\omega_{\sigma(k)} (\tau-j)}   \right)	< G(1)=0,
	\end{equation}
where the inequality is directly derived from $f_1(\frac{\tau-j}{T},\om(\tau-j))>0$ and the fact that $G(\beta)=1/\sqrt{\beta}-1$ is a strictly decreasing function. Note that (\ref{eq:nyj2 7})-(\ref{eq:nyj2 5})-(\ref{eq:nyj2 6}) trigger the condition of discarding $\ell(\tau-j)$ in the round $\tau-j$.
\end{proof}

\medskip
\noindent

\begin{lemma}\label{lem:yj3}
Let $j\in \{2,\ldots, K-1\}$ and $i>j$. The process $\{\om (\theta T)\}_{T\ge 1}$ satisfies an LDP local upper bound (\ref{eq:local UB}) with $I_{\theta} (\xx)=\infty$ at $\xx\in \mathcal{X}_{j,i}(\theta)$. 
\end{lemma}
\begin{proof}
Let $\xx\in \mathcal{X}_{j,i}(\theta)$  and let $\sigma$ be the permutation described in (\ref{eq:xj}) for $\xx$.  We show that there exists $\delta_{\theta ,\xx}>0$ s.t. for all $\delta <\delta_{\theta,\xx}$,
\begin{equation}\label{eq:yj3 0}
\lim_{T\rightarrow \infty}\frac{1}{T}\log \frac{1}{ \PP_{\bm}[\om (\theta T)\in B(\xx,\delta)]}=\infty.    
\end{equation}
We decompose the proof into three steps.

\medskip
\noindent
\underline{\it 1. Defining $\delta_{\theta,\xx}$, $T_{\xx}$, and a random time $\tau$.} We introduce two functions: for all $\xx'\in\Sigma$:
	\begin{align*}
	f_1(\xx')&= \theta\min_{k=1,\ldots,i} x'_{\sigma(k)}i\olog i-\theta \sum_{k=i+1}^K x'_{\sigma(k)}-1,\\
	f_2(\xx')&=\min_{k=1,\ldots,i} x'_{\sigma(k)}-\max_{k=i+1,\ldots ,K} x'_{\sigma(k)} .
\end{align*}
Because  $f_1(\xx)>0,f_2(\xx)>(x_{\sigma(i)}-x_{\sigma(i+1)})/2$, and both $f_1,f_2$ are continuous, we can find a positive $\delta_{\theta,\xx}>0$ s.t. 
\begin{equation}\label{eq:yj3 1}
    \forall \xx'\in B(\xx,\delta_{\theta,\xx}),\, f_1(\xx')>0,\,f_2(\xx')>\frac{x_{\sigma(i)}-x_{\sigma(i+1)}}{2}.
\end{equation}
In the following, we fix $\delta<\delta_{\theta,\xx}$. We define $T_{\xx}$ as: $T_{\xx}=\lceil \frac{2}{x_{\sigma(i)}-x_{\sigma(i+1)}}\rceil$. Finally, we introduce the random time $\tau= i\min_{k\le i}N_{\sigma(k)}(\theta T)+\sum_{k>i}N_{\sigma(k)}(\theta T)$ and two fixed rounds, $\tau_{\min}=\lfloor(ix_{\sigma (i)}+\sum_{k>i}x_{\sigma(k)}-K\delta)T\rfloor $ and $\tau_{\max}=\lceil(ix_{\sigma (i)}+\sum_{k>i}x_{\sigma(k)}+K\delta)T\rceil$. 

\medskip
 \noindent
 \underline{\it 2. If $T>T_{\theta,\xx}$ and $\om (\theta T )\in B(\xx,\delta)$, then (i) $\tau\in \{\tau_{\min},\ldots,\tau_{\max}\}$ and (ii) $\om (\tau)\in \mathcal{X}_{i,i}(\frac{\tau}{T})$.}
 (i) is trivial based on the definition of $B(\xx,\delta)$. To show (ii), we observe that 
\begin{equation}\label{eq:yj3 2}
\min_{k\le i}N_{\sigma(k)}(\theta T)-\max_{k\ge i+1}N_{\sigma(k)}(\theta T)=Tf_2(\om(\theta T))>\frac{2}{x_{\sigma(i)}-x_{\sigma(i+1)}} \frac{x_{\sigma(i)}-x_{\sigma(i+1)}}{2} = 1,	    
\end{equation}
where the inequality simply comes from (\ref{eq:yj3 1}) and $T>T_{\xx}$. Since \sred always pulls the arms in the candidate set in a round-robin manner (the maximal difference of pulling amounts among the candidate set is at most $1$), (\ref{eq:yj3 2}) implies \sred discards one arm in $\{\sigma(i+1),\ldots,\sigma(K)\}$ in the round 
$\tau$, and $\om (\tau)$ is:
\begin{equation}\label{eq:yj3 3}
\omega_{\sigma (k)}(\tau)=\left\{
\begin{array}{ll}
	\min_{k'\le i}N_{\sigma(k')}(\theta T)/\tau,	& \text{ if }k\le i,\\
	N_{\sigma(k)(\theta T)}/\tau,	&\text{ otherwise.}
\end{array}\right.	
\end{equation}
Note that (\ref{eq:yj3 3}) yields that $\om (\tau)\in \mathcal{X}_{i}$. Moreover,  
\begin{align*}
\frac{\tau}{T}\omega_{\sigma(i) }(\tau)i\olog i-\frac{\tau}{T}\sum_{k=i+1}^K\omega_{\sigma (k)}(\tau)&=\theta \min_{k\le i}\omega_{\sigma (k)}(\theta T)i\olog i- \theta\sum_{k=i+1}^K\omega_{\sigma (k)}(\theta T)\\
&=f_1(\om (\theta T))+1>1,    
\end{align*}
where the inequality directly follows from (\ref{eq:yj3 1}) and the fact that $\om (\theta T)\in B(\xx,\delta_{\theta,\xx})$. Hence $\om (\tau)\in \mathcal{X}_{i,i}(\frac{\tau}{T})$. 

\medskip
 \noindent
 \underline{\it 3. We show (\ref{eq:yj3 0}).}
Thanks to (i) and (ii) in Step 2, we have, for $T>T_{\xx}$, 
$$
\PP_{\bm}\left[ \om (\theta T)\in \mathcal{B}(\xx,\delta) \right] \le \sum_{\tau=\tau_{\min}}^{\tau_{\max}} \PP_{\bm}\left[ \om(\tau)\in \mathcal{X}_{i,i}(\frac{\tau}{T})  \right]\le 2K\delta T  \max_{\theta'\in (0,1]}\PP_{\bm}\left[ \om(\theta'T)\in \mathcal{X}_{i,i}(\theta')  \right].
$$ 
Thus, a simple rearrangement of the above inequality yields that
\begin{align*}
\lowlim_{T\rightarrow \infty}\frac{1}{T}\log\frac{1}{ \PP_{\bm}[\om (\theta T)\in B(\xx,\delta)]}&\ge\inf_{\theta'\in (0,1]}\lowlim_{T\rightarrow \infty}\frac{1}{T}\log\frac{1}{ \PP_{\bm}[\om (\theta' T)\in \mathcal{X}_{i,i}(\theta')]} \\
&\ge \inf_{\theta'\in (0,1]}  \inf_{\xx'\in \mathcal{X}_{i,i}(\theta')}   I_{\theta'}(\xx')= \infty,
\end{align*}
where the last inequality stems from Lemma \ref{lem:yj2}.

\end{proof}

\medskip

\begin{lemma}\label{lem:theta delta}
	Let $\theta \in (0,1]\cap \QQ$, and $\delta \in (0,1)$. If we consider $\theta'\in [\theta-\theta\delta, \theta]$ and $T\in \NN$ such that $\theta T,\theta'T\in \NN$, then
	\begin{equation}\label{eq:theta delta}
	\left\| \om(\theta T)-\om(\theta' T)\right\|_{\infty}\le 2\delta.
	\end{equation}
\end{lemma}
\begin{proof}
	Observe that for any $ k\in [K]$,
	\begin{align*}
		\left| \omega_k(\theta T)-\omega_k(\theta' T)\right|=	\left| \frac{N_k(\theta T)}{\theta T}-\frac{N_k(\theta' T)}{\theta' T}\right|&\le 	\left| \frac{N_k(\theta T)}{\theta T}-\frac{N_k(\theta' T)}{\theta T}\right|+	\left| \frac{N_k(\theta' T)}{\theta T}-\frac{N_k(\theta' T)}{\theta' T}\right|\\
		&	\le \frac{\theta-\theta'}{\theta}  +\theta'(\frac{1}{\theta'}-\frac{1}{\theta})\\
		&=1-\frac{\theta'}{\theta}+1-\frac{\theta'}{\theta}\le 2\delta,
	\end{align*}
	where the first inequality uses the triangle inequality; the second inequality simply comes from $N_k(\theta T)-N_k(\theta' T)\le (\theta-\theta' )T$ and $N_k(\theta' T)\le \theta'T$; the third inequality is a consequence of $\theta'\ge \theta(1-\delta)$. Hence (\ref{eq:theta delta}) follows.
\end{proof}

\begin{lemma}\label{lem:theta delta mu}
	Let $\theta \in (0,1]\cap \QQ$, and $\delta \in (0,1)$. If we consider $\theta'\in [\theta-\theta\delta, \theta]$ and $T\in \NN$ such that $\theta T,\theta'T\in \NN$, then 
	\begin{equation}\label{eq:theta delta mu}
	\left\| \hat{\bm} (\theta T)-\hat{\bm}(\theta' T)\right\|_{\infty}\le 2\delta.
	\end{equation}
\end{lemma}
\begin{proof}
	Observe that for any $ k\in [K]$,
	\begin{align*}
		\left|\hat{\mu}_k(\theta T)-\hat{\mu}_k(\theta' T)\right|&=	\left| \frac{\sum_{t\le \theta T}X_k(t) }{\theta T}-\frac{\sum_{t\le \theta' T}X_k(t)}{\theta' T}\right|\\
  &\le 	\left| \frac{\sum_{t\le \theta T}X_k(t)}{\theta T}-\frac{\sum_{t\le \theta' T}X_k(t)}{\theta T}\right|+	\left| \frac{\sum_{t\le \theta' T}X_k(t)}{\theta T}-\frac{\sum_{t\le \theta' T}X_k(t)}{\theta' T}\right|\\
		&	\le \frac{\theta-\theta'}{\theta}  +\theta'(\frac{1}{\theta'}-\frac{1}{\theta})\\
		&=1-\frac{\theta'}{\theta}+1-\frac{\theta'}{\theta}\le 2\delta,
	\end{align*}
	where the first inequality uses the triangle inequality; the second inequality simply comes from $\sum_{\theta'T<t\le \theta T}X_k(t)\le (\theta-\theta' )T$ and $\sum_{t\le \theta T}X_k(t)\le \theta'T$ (as Assumption~\ref{ass1} ensures that $X_k(t)\in [0,1]$); the third inequality is a consequence of $\theta'\ge \theta(1-\delta)$. Hence (\ref{eq:theta delta mu}) follows.
\end{proof}


\subsection{Local LDP upper bound on $\mathcal{X}_j\setminus \cup_{i= j}^K \mathcal{X}_{j,i}(\theta)$}\label{subapp:finite I}
Fix $\theta \in (0,1]\cap \mathbb{Q}$. We will establish local LDP upper bounds  on $\mathcal{X}_j\setminus \cup_{i= j}^K \mathcal{X}_{j,i}(\theta)$ for the process $\{\om (\theta T)\}_{T\ge 1}$ under \crc and \cra. The upper bound for \crc is presented in \ref{subapp:rate C} and that for \cra in \ref{subapp:rate A}. We first present a useful proposition repeatedly used in \ref{subapp:rate C}, \ref{subapp:rate A}, and the main proofs for Theorem \ref{thm:crc} and Theorem \ref{thm:cra} in \ref{app_cr}. 

\medskip
\noindent
One important property for $\xx\in\mathcal{X}_j\setminus  \cup_{i= j}^K \mathcal{X}_{j,i}(\theta) $ is that the remaining budget for the empirical top-$j$ arms is lower bounded by a constant. This observation is summarized in Proposition \ref{prop ybound}.

\begin{proposition}\label{prop ybound}
If $\xx\in \mathcal{X}_j\setminus \cup_{i= j}^K \mathcal{X}_{j,i}(\theta)$, then $ \forall i\in \{j+1,\ldots,K\}$,
	\begin{equation*}
	1-\theta\sum_{k=i}^Kx_{\sigma (k)}\ge \frac{\olog (i-1)}{\olog K}.		
	\end{equation*}

\end{proposition}
\begin{proof}
	Notice that the statement of the proposition is equivalent to (\ref{eq:ybound 1}): $ \forall i\in \{j+1,\ldots,K\}$,
	\begin{equation}\label{eq:ybound 1}
		\theta \sum_{k=i}^Kx_{\sigma(k)}\le \frac{1}{\olog K}\sum_{k=i}^K\frac{1}{k}.
	\end{equation}
	The inequalities (\ref{eq:ybound 1}) will be proved by induction. 
	
	\medskip
	\noindent
	\underline{\it 1. We show (\ref{eq:ybound 1}) for $i=K$}. Since $\xx\notin  \mathcal{X}_{j,K}(\theta)$, 
	\begin{equation}\label{eq:ybound 2}
		1\ge \theta K\olog Kx_{\sigma(K)}.
	\end{equation}
	Dividing by $K\olog K$ on both sides of (\ref{eq:ybound 2}) directly yields (\ref{eq:ybound 1}) with $i=K$. Now we assume (\ref{eq:ybound 1}) is valid for some $i+1\in \{j+2,\ldots, K\}$, and we show (\ref{eq:ybound 1}) for $i$.
	
	\medskip
	\noindent
	\underline{\it 2. We show (\ref{eq:ybound 1}) for $i$ assuming that (\ref{eq:ybound 1}) holds for $i+1$.} As $\xx\notin \mathcal{X}_{j,i}(\theta)$, we get:
	$$
	1-\theta \sum_{k=i+1}^Kx_{\sigma(k)}\ge  \theta  i\olog i x_{\sigma(i)}.
	$$ 
	Dividing the above inequality by $i\olog i$ and adding $\theta \sum_{k=i+1}^Kx_{\sigma (k)}$ to the both sides, we obtain that
	\begin{align*}
		\theta \sum_{k=i}^Kx_{\sigma(k)}&\le \frac{1}{i\olog i}+(1-\frac{1}{i\olog i})\theta \sum_{k=i+1}^Kx_{\sigma (k)}\\
		&\le \frac{1}{i\olog i}+(1-\frac{1}{i\olog i})\frac{1}{\olog K}\sum_{k=i+1}^K\frac{1}{k}\\
		&=\frac{1}{\olog K}\sum_{k=i+1}^K\frac{1}{k}+\frac{1}{i\olog i}\left[1-\frac{1}{\olog K}\sum_{k=i+1}^K\frac{1}{k}\right],
	\end{align*}
	where the second inequality stems from our inductive hypothesis (\ref{eq:ybound 1}) for $i+1$. As $\olog K-\olog i=\sum_{k=i+1}^K\frac{1}{k}$, the r.h.s on the above inequality is equal to
	$$
	\frac{1}{\olog K}\sum_{k=i+1}^K\frac{1}{k}+\frac{1}{i\olog i}\left[1-\frac{1}{\olog K} \left(\olog K-\olog i\right)  \right]=\frac{1}{\olog K}\sum_{k=i}^K\frac{1}{k},
	$$
	hence (\ref{eq:ybound 1}) is proved.
\end{proof}

\subsubsection{The allocation process under \crc}\label{subapp:rate C}
We show that ${I}_{\theta}$ presented in Theorem \ref{thm:Cj} below is a valid rate function for a local LDP upper bound (\ref{eq:local UB}) for the process $\{\om (\theta T)\}_{T\ge 1}$. This function is however too complicated to use, and we will present a simpler rate function in Corollary \ref{cor:cj rate}. 

\begin{theorem}\label{thm:Cj}
	Let $j\in \{2,\ldots, K-1\},\,\theta\in (0,1], \xx\in \mathcal{X}_j\setminus \cup_{i= j}^K \mathcal{X}_{j,i}(\theta) $, define
	\begin{equation}
	I_{\theta}(\xx )= \max_{p=j,\ldots,K-1}   2x_{\sigma (p+1)}\inf_{\bl\in \mathcal{S}_p(\xx)}\sum_{k=1}^{p+1}(\lambda_{\sigma(k)}-\mu_{\sigma (k)})^2,
	\end{equation}
	where $\sigma$ is the permutation described in (\ref{eq:xj}) for $\zz$ and  
	\begin{equation}\label{eq:Spc}
			\mathcal{S}_p(\xx) =\left\{\bl\in [0,1]^K: \lambda_{\sigma(p+1)} \le \min_{k\le p}\lambda_{\sigma (k)}-G\left(\frac{\theta x_{\sigma (p+1)}(p+1)\olog (p+1) }{1-\theta\sum_{k=p+2}^K x_{\sigma(k)}}\right) \right\}.
	\end{equation}
 Then the process $\{\om (\theta T)\}_{T\ge 1}$ under \crc satisfies a local LDP upper bound (\ref{eq:local UB}) with $I_{\theta}$ at $\xx$.
\end{theorem}
\begin{proof}

\noindent
For simplicity, $\sigma$ is assumed to be the identity map. Let $p\in \{j,\ldots,K-1\}$. As $\xx\in \mathcal{X}_j$, we have $x_{p}-x_{p+1}>0$. Let $T>K$, and $\delta<(x_{p}-x_{p+1})/2$ be some positive number. We will derive an upper bound for $\PP_{\bm}\left[\om (\theta T)\in B(\xx,\delta)\right]$, and then compute its rate by driving $T\rightarrow \infty$ and $\delta\rightarrow 0$.
	
	\medskip
        \noindent
Observe that of course:
\begin{align}\label{eq:cj new0}
\PP_{\bm}\left[ \om ( \theta T)\in B(\xx,\delta)\right] \le  \PP_{\bm}\left[ \cup_{\yy\in B(\xx,\delta)}\{ \om (\theta T)=\yy\} \right].
\end{align}

 \medskip
\noindent
\underline{\it 1. Obtaining a necessary condition for $ \om (\theta T)=\yy$.}
 For any $\yy\in B(\xx,\delta)$, we introduce $\theta(\yy)$ and $\zz (\yy)$ as: 
	$$
	 \theta(\yy)=\theta-\theta\sum_{k=1}^p (y_{k}-y_{p+1})\text{, and }
		z_{k}(\yy)= \left\{\begin{array}{ll}
			\theta y_{k}/\theta(\yy),&\hbox{ if } k\ge p+2,\\
			\theta   y_{p+1}/\theta(\yy),&\hbox{ if } k\le p+1.
		\end{array}\right. 
	$$
 Following directly from the above definitions, we obtain that 
\begin{equation}\label{eq:cj new1}	
\theta (\yy) \sum_{k=1}^{p+1}z_{k}(\yy) =\theta  (p+1)y_{p+1}\hbox{, and }\theta (\yy ) z_{k} (\yy )=\theta  y_{k}, \,\forall k=p+2,\ldots K.	
\end{equation}
 From the choice of $\delta$, $y_{p+1}$ is the smallest real number in $\{y_1,\ldots,y_{p+1}\}$, so $\ell_{p+1}=p+1$. Moreover, $\theta (\yy)T$ is the round \crc decides $\ell_{p+1}=\ell(\theta(\yy) T)=p+1$, and $\om (\theta (\yy)T)=\zz(\yy)$. Due to the condition for discarding $p+1$ (see (\ref{C})), we have $\hat{\bm}(\theta (\yy)T)\in \mathcal{S}_p(\yy)$.
Consequently, we have: 
\begin{equation}\label{eq:cj new2}
\left\{ \om (\theta T)=\yy\right\} \subset \left\{ \hat{\bm}(\theta (\yy) T)\in \mathcal{S}_p(\yy),\om (\theta (\yy) T)= z(\yy)\right\}.
\end{equation}

 \medskip
\noindent
\underline{\it 2. Reducing $\cup_{\yy\in B(\xx,\delta)}\{ \om (\theta T)=\yy\} $ to a single set.} To do this reduction, we use the results of Step 1 and Lemmas \ref{lem:theta delta} and \ref{lem:theta delta mu}. 

Let $\yy_0\in\arg\max_{\yy \in B(\xx,\delta)}\theta(\yy)$. Using the continuity of the function $\theta(\yy)$, there exists $\eta(\delta)$ such that $\eta(\delta)$ tends to 0 as $\delta\to 0$, and for all $\yy\in B(\xx,\delta)$, $\theta(\yy)\in [\theta(\yy_0)(1-\eta (\delta)),\theta(\yy_0)]$. By Lemmas \ref{lem:theta delta} and \ref{lem:theta delta mu}, we obtain that:
\begin{align*}
    \| \hat{\mu}(\theta(\yy))-\hat{\mu}(\theta(\yy_0))\|_\infty \le 2\eta(\delta),\\
    \| \om(\theta(\yy)T) - \om(\theta(\yy_0)T)\|_\infty \le 2\eta (\delta). 
\end{align*}
Now define the following sets:
\begin{align*}
    \bar S_{\delta,p} &= \cup_{\yy\in B(\xx,\delta)}\{\bar{s}: \exists s\in {\cal S}_p(\yy): \|s-\bar{s}\|_\infty \le 2\eta(\delta)\},\\
    \bar W_\delta &= \cup_{\yy\in B(\xx,\delta)}\{\bar{w}\in \Sigma: \|\bar{w}- z(\yy)\|_\infty\le 2\eta(\delta)\}.
\end{align*}
By construction, for all $\yy \in  B(\xx,\delta)$, we have that if $\hat{\bm}(\theta (\yy) T)\in \mathcal{S}_p(\yy)$, then $\hat{\bm}(\theta (\yy_0) T)\in \bar S_{\delta,p}$, and similarly $\om (\theta (\yy) T)= z(\yy)$ implies that $\om (\theta (\yy_0) T)\in \bar W_\delta$.

 \medskip
\noindent
\underline{\it 3. Using Theorem \ref{thm:Varadaham bandits}.} Putting the results of the first two steps together, we get:
$$
\PP_{\bm}\left[ \om ( \theta T)\in B(\xx,\delta)\right] \le \PP_{\bm}\left[ \hat{\bm}(\theta (\yy_0) T)\in \bar{S}_{\delta,p},\om (\theta (\yy_0) T)\in \bar{W}_\delta\right]
$$

By applying (c) in Section \ref{sec:varadaham} with  $\mathcal{S}=\bar{S}_{\delta,p}$ and $W=\bar{W}_\delta$, it follows that 
\begin{align}
\nonumber	\lowlim_{T\to\infty}\frac{1}{T}\log \frac{1}{\PP_{\bm}\left[ \om ( \theta T)\in B(\xx,\delta)\right] } &\ge \theta (\yy_0)\inf_{\om\in \bar{W}_\delta} F_{\bar{S}_{\delta,p}}(\om).
\end{align}
When $\delta$ tends to 0, the r.h.s. simply converges to $\theta(\xx)F_{{\cal S}_p(\xx)}(z(\xx))$. The latter is $\inf_{\bl\in \mathcal{S}_{p}(\xx)}  2 \sum_{k=1}^{p+1}x_{p+1}(\lambda_k-\mu_k)^2$ (see Lemma \ref{lem:verify c} in Appendix \ref{app:verify c} for continuity arguments). As the above proof holds for any $p\in \{j,\ldots,K-1\}$, we complete the proof.
\end{proof}

\medskip
\noindent
Next, as in Appendix~\ref{app_cr}, we define the subset of $\mathcal{X}_j\setminus \cup_{i= j}^K\mathcal{X}_{j,i}(\theta)$:
$$
\mathcal{Z}_{[j]}(\theta,\beta)=\left\{\zz\in\mathcal{X}_j\setminus \cup_{i= j}^K\mathcal{X}_{j,i}(\theta) :\sigma(k)=k,\,\forall k\le j\text{ and }\frac{\theta z_j j\olog j}{1-\theta\sum_{k>j}z_k}=\beta\right\}
$$ for all $\beta\in (0,1]$. Note that the permutation $\sigma$ in the above definition corresponds to that used for point $\zz$ as in the definition of ${\cal X}_j$ (\ref{eq:xj}): it is such that $z_{\sigma (1)}=\ldots=z_{\sigma (j)}>z_{\sigma (j+1)}> \ldots > z_{\sigma (K)}> 0$.

\medskip
\noindent
\begin{corollary}\label{cor:cj rate}
Let $j\in \{2,\ldots, K-1\},\,\theta,\beta\in (0,1],\, \zz\in \mathcal{Z}_{[j]}(\theta,\beta)$. Define
$$
\underline{I}_{\theta}(\zz)=\frac{\olog (j+1)}{\theta\olog K}   \left[\left(  (1+\zeta_j)\sqrt{\frac{\theta z_{\sigma(j+1)}}{ 1-\theta\sum_{k=j+2}^Kz_{\sigma (k)}  }}-\sqrt{\frac{1}{(j+1)\olog (j+1)}}    \right)_+\right]^2,
$$
where $\sigma$ is the permutation described in (\ref{eq:xj}) for $\zz$, then $\underline{I}_{\theta}(\zz )\le I_{\theta} (\zz)$. 
\end{corollary}
\begin{proof}
Observe that $I_{\theta} (\zz)$ is larger than the value of the following optimization problem:
\begin{align}\label{opt:5}
	&\min_{\bl\in \RR^K}2z_{\sigma(j+1)} \left( (\lambda_j-\mu_j)^2+(\lambda_{\sigma(j+1) }-\mu_{\sigma(j+1)})^2\right),\\
	\nonumber	&\hbox{subject to }\lambda_{\sigma(j+1)}\le \lambda_j-G(\tilde{\beta}),
\end{align}
where 
$$
\tilde{\beta}=\frac{\theta z_{\sigma(j+1)} (j+1)\olog (j+1) }{1-\theta \sum_{k=j+2}^Kz_{\sigma(k)}}.
$$
One can simply verify that the optimal value of (\ref{opt:5}) is 
\begin{equation}\label{eq:JC rate 1}
		z_{\sigma(j+1)}[(\mu_j-\mu_{\sigma(j+1)}-G(\tilde{\beta}))_+]^2\ge  z_{\sigma(j+1)}[(1+\zeta_j-\frac{1}{\sqrt{\tilde{\beta}}})_+]^2,
\end{equation}
where the inequality follows from the fact that $\mu_{j+1}\ge \mu_{\sigma(j+1)}$ and $G(\tilde{\beta})=1/\sqrt{\tilde{\beta}}-1$. Substituting the value of $\tilde{\beta}$ yields that the r.h.s. of (\ref{eq:JC rate 1}) is equal to
$$
\frac{1}{\theta}\left(1-\theta\sum_{k=j+2}^Kz_{\sigma(k)}\right)\left[\left(  (1+\zeta_j)\sqrt{\frac{\theta z_{\sigma(j+1)}}{ 1-\theta\sum_{k=j+2}^Kz_{\sigma (k)} }}-\sqrt{\frac{1}{(j+1)\olog (j+1)}}    \right)_+\right]^2.
$$
As $\zz\in \mathcal{X}_j\setminus \cup_{i= j}^K\mathcal{X}_{j,i}(\theta)$, Proposition \ref{prop ybound} with $i=j+2$ directly completes the proof. 
\end{proof}


\subsubsection{The allocation process under \cra} \label{subapp:rate A}
One can use similar arguments as those used in the proof of Theorem \ref{thm:Cj} to derive the analogous rate function for the allocation process under \cra.

\begin{theorem}\label{thm:Aj}
	Let $j\in \{2,\ldots, K-1\},\,\theta\in (0,1], \xx\in \mathcal{X}_j\setminus \cup_{i= j}^K \mathcal{X}_{j,i}(\theta) $, define
	\begin{equation}
		I_{\theta}(\xx )=  \max_{p=j,\ldots, K-1}  2x_{\sigma (p+1)}\inf_{\bl\in \mathcal{S}_p(\xx)}\sum_{k=1}^{p+1}(\lambda_{\sigma(k)}-\mu_{\sigma (k)})^2,
	\end{equation}
	where $\sigma$ is the permutation described in (\ref{eq:xj}) for $\zz$ and  
	\begin{equation}\label{eq:Spa}
		\mathcal{S}_p(\xx) =\left\{\bl\in [0,1]^K: \lambda_{\sigma(p+1)}\le \frac{\sum_{k=1}^{p}\lambda_{\sigma (k)}}{p}-G\left(\frac{\theta x_{\sigma (p+1)}(p+1)\olog (p+1) }{1-\theta\sum_{k=p+2}^K x_{\sigma(k)}}\right)\right\}.
	\end{equation} 
	Then the process $\{\om (\theta T)\}_{T\ge 1}$ under \cra satisfies a local LDP upper bound (\ref{eq:local UB}) with $\bar{I}_{\theta} $ at $\xx$.
\end{theorem}
The proof is omitted as it is almost the same as that of Theorem \ref{thm:Cj}. We can also obtain the analogous version of Corollary \ref{cor:cj rate} as shown below:

\medskip
\noindent
\begin{corollary}\label{cor:aj rate}
	Let $j\in \{2,\ldots, K\},\,\theta,\beta\in (0,1],\, \zz\in \mathcal{Z}_{[j]}(\theta,\beta)$, we define
	\begin{equation}\label{eq:JA rate}
		\underline{I}_{\theta}(\zz )=\frac{2j\olog (j+1)}{\theta(j+1)\olog K} \left[\left(  (1+\varphi_j)\sqrt{\frac{\theta z_{\sigma(j+1)}}{ 1-\theta\sum_{k=j+2}^Kz_{\sigma(k)}  }}-\sqrt{\frac{1}{(j+1)\olog (j+1)}}    \right)_+\right]^2,
	\end{equation}
where $\sigma$ is the permutation described in (\ref{eq:xj}) for $\zz$, then $\underline{I}_{\theta}(\zz )\le I_{\theta} (\zz)$. 
\end{corollary}
\begin{proof}
	We can simply solve the optimization problem similar to (\ref{eq:cra G}) as in the proof Proposition \ref{prop:cra G} in Appendix \ref{app:opt A} and get that $I_{\theta}(\zz)$ is greater than the l.h.s. of the following inequality.
 \begin{equation}\label{eq:aj rate 1}
	\frac{2jz_{\sigma (j+1)}}{j+1}   \left[  \left(\frac{\sum_{k=1}^j \mu_k}{j}-\mu_{\sigma(j+1)}-G(\tilde{\beta})\right)_+   \right]^2\ge\frac{2jz_{\sigma (j+1)}}{j+1}   \left[  \left( (1+\varphi_j)-\frac{1}{\sqrt{\tilde{\beta}}}\right)_+   \right]^2 	,
	\end{equation}
	where 
	$$
	\tilde{\beta}=\frac{\theta z_{\sigma(j+1)} (j+1)\olog (j+1) }{1-\theta \sum_{k=j+2}^Kz_{\sigma(k)}}.
	$$
	(\ref{eq:aj rate 1}) is due to $\mu_{j+1}\ge \mu_{\sigma(j+1)}$ (hence$ \sum_{k=1}^j \mu_k/ j-\mu_{\sigma(j+1)}\ge \varphi_j$) and $G(\tilde{\beta})=\frac{1}{\sqrt{\tilde{\beta}}}-1$. Substituting the value of $\tilde{\beta}$ yields that the r.h.s. of (\ref{eq:aj rate 1}) equals to
	$$
	\frac{2j}{\theta(j+1)}\left(1-\theta\sum_{k=j+2}^Kz_{\sigma(k)}\right) \left[\left(  (1+\varphi_j)\sqrt{\frac{\theta z_{\sigma(j+1)}}{ 1-\theta\sum_{k=j+2}^Kz_{\sigma(k)}  }}-\sqrt{\frac{1}{(j+1)\olog (j+1)}}    \right)_+\right]^2.
	$$
	As $\zz\in \mathcal{X}_j\setminus \cup_{i= j}^K\mathcal{X}_{j,i}(\theta)$, Proposition \ref{prop ybound} with $i=j+2$ directly completes the proof. 
	
\end{proof}

\newpage
\section{Continuity arguments}\label{app:continuity}
We introduce some definitions and results taken from \citep{berge1997topological}, and also used recently in \citep{wang2021fast,degenne2019pure,combes2017} in the bandit literature. 

\medskip
\noindent
Suppose $\mathbb{X}$ and $\mathbb{Y}$ are Hausdorff topological spaces. Let $u:\mathbb{X}\times \mathbb{Y}\to \RR$ be a function and $\Phi:\mathbb{X}\setmap \mathbb{S}(\mathbb{Y})$ be a set-valued function, where $\mathbb{S}(\mathbb{Y})$ is the set of non-empty subsets of $\mathbb{Y}$. Besides, we introduce $\mathbb{K}(\mathbb{X})=\{F\in \mathbb{S}(\mathbb{X}):F\text{ is compact}\}$. We are interested in a minimization problem of the form: for $x\in \mathbb{X}$, 
\begin{align*}
v(x) &= \inf_{y\in \Phi(x)} u(x, y).
\end{align*}
We define the set of solutions of this problem as $\Phi^*(x) = \{y\in\Phi(x)\: : \: u(x, y) = v(x)\}$.

\begin{theorem}[{\citep{berge1997topological}}]\label{thm:berge}
Assume that
	\begin{itemize}
		\item $\Phi: \mathbb{X} \rightrightarrows \mathbb{K}(\mathbb{X})$ is continuous (i.e., both lower and upper hemicontinous),
		\item $u: \mathbb{X} \times \mathbb{Y} \to \RR$ is continuous.
	\end{itemize}
	Then the function $v:\mathbb{X} \to \RR$ is continuous and the solution multifunction $\Phi^*:\mathbb{X}\to\mathbb{S}(\mathbb{Y})$ is upper hemicontinuous, with values that are nonempty and compact.
\end{theorem}



\subsection{Verifying the continuity in Theorem \ref{thm:Cj}}\label{app:verify c}
We verify the continuity argument used in the proof of Theorem \ref{thm:Cj}. 
\begin{lemma}\label{lem:verify c}
The function 
\begin{align*}
&	\inf_{\bl\in \mathcal{S}_p(\xx)} \sum_{k=1}^{p+1}x_{p+1}(\lambda_k-\mu_k)^2,\\
&\text{where }\mathcal{S}_p(\xx) =\left\{\bl\in [0,1]^K: \lambda_{p+1} \le \min_{k\le p}\lambda_{k}-G\left(\frac{\theta x_{p+1}(p+1)\olog (p+1) }{1-\theta\sum_{k=j+2}^K x_{k}}\right) \right\},
\end{align*}
is continuous for all $\xx\in \Sigma$.
\end{lemma}
\begin{proof}
		We apply Theorem \ref{thm:berge} with:
	\begin{itemize}
		\begin{minipage}{0.4\linewidth}
			\item $\mathbb{X} = \Sigma$,
			\item $\mathbb{Y} =[0,1]^K$,
		\end{minipage}
		\begin{minipage}{0.6\linewidth}
			\item $\Phi(\xx ) = \mathcal{S}_p(\xx)$,
			\item $u(\xx, \bl) =\sum_{k=1}^Kx_{p+1}(\lambda_k-\mu_k)^{2}$.
		\end{minipage}
	\end{itemize}
As the objective function is obviously continuous, it remains to show that $\mathcal{S}_p(\cdot)$ is a continuous correspondence. By simply using Lemma \ref{lem:hemi S} with $f(\bl)=\lambda_{p+1}- \min_{k\le p}\lambda_{k}$ and $g(\xx)=-G\left(\frac{\theta x_{p+1}(p+1)\olog (p+1) }{1-\theta\sum_{k=p+2}^K x_{k}}\right) $, we can complete the proof.
\end{proof}

\medskip
\noindent
It is straightforward to get a similar guarantee for the function involved in \cra: we hence omit the proof of the following lemma. 

\begin{lemma}\label{lem:verify a}
The function 
\begin{align*}
&	\inf_{\bl\in \mathcal{S}_p(\xx)} \sum_{k=1}^{p+1}x_{p+1}(\lambda_k-\mu_k)^2,\\
&\text{where }\mathcal{S}_p(\xx) =\left\{\bl\in [0,1]^K: \lambda_{p+1} \le \frac{\sum_{k\le p}\lambda_{k}}{p}-G\left(\frac{\theta x_{p+1}(p+1)\olog (p+1) }{1-\theta\sum_{k=j+2}^K x_{k}}\right) \right\},
\end{align*}
is continuous for all $\xx\in \Sigma$.
\end{lemma}

\medskip
\noindent

\begin{lemma}\label{lem:hemi S}
	Let $g:\Sigma\mapsto \RR$ be a continuous mapping and $f:[0,1]^K\mapsto \RR$ be a lower semicontinuous mapping which further satisfies that $\forall \bl\in \RR^K,\delta>0$, there exists $\bl'\in \RR^K$ s.t. $\left\|\bl-\bl'\right\|_{\infty}\le \delta$ and $f(\bl')<f(\bl)$. Then $\forall \xx\in \Sigma$,
	$$
	\mathcal{S}(\xx)=\left\{\bl\in [0,1]^K:f(\bl)\le g(\xx)\right\},
	$$
	is a continuous correspondence.
	\begin{proof}
		{\it (i) Upper hemicontinuity}: Suppose $\{\xx_n\}_{n\in \NN}\subset \Sigma$ converges to $\xx^\star\in \Sigma$ and $\{\bl_n\}_{n\in \NN}\subset \RR^K$ converges to $\bl^\star$ s.t. $\bl_n\in \mathcal{S}(\xx_n)$ for all $n\in\NN$. We will show that $\bl^\star\in \mathcal{S}( \xx^\star)$. Since $g$ is upper semicontinuous, and $\xx_n\rightarrow \xx^\star$ as $n\rightarrow \infty$, for any $\epsilon>0$, $\exists N_\epsilon\in\NN$ s.t. $g(\xx_n)\le g(\xx^\star)+\epsilon $ for all $n\ge N_\epsilon$. As $\bl_n\in\mathcal{S}(\xx_n)$, we deduce that 
		$$
		f( \bl_n )\le g(\xx_n)\le g( \xx^\star )+\epsilon,\, \,\forall n\ge N_\epsilon.  
		$$
		Now the lower semicontinuity of $f$ implies $f(\bl^\star)\le \lowlim_{n\rightarrow \infty}f(\bl_n)\le g( \xx^\star )+\epsilon$. As $\epsilon$ can be taken arbitrarily, $\bl^\star \in \mathcal{S}(\xx^\star)$.
		
		\medskip
		\noindent
		{\it (ii) Lower hemicontinuity}: Suppose $\{\xx_n\}_{n\in \NN}\subset \Sigma$ converges to $\xx^\star\in \Sigma$, we aim to show that for all $\bl^\star\in \mathcal{S}(\xx^\star)$ (or equivalently $f(\bl^\star)\le g(\xx^\star)$), there exist a subsequence $\{\xx_{n_k}\}_{k\in \NN}\subseteq \{\xx_n\}_{n\in \NN}$ and a sequence $\{\bl_k\}_{k\in \NN}$ s.t. $\bl_k\in \mathcal{S}(\xx_{n_k})$ and $\bl_k \rightarrow \bl^\star$ as $k\rightarrow \infty$. By assumption on $f$, for any integer $k$, there exists $\bl_k$ s.t. $\left\|\bl_k-\bl^\star\right\|_{\infty}<1/k$ and $f(\bl_k)<f(\bl^\star)$. Also, $g(\xx_n)\rightarrow g(\xx^\star)$ as $n\rightarrow \infty$ implies there exists a finite $n$ s.t. $g(\xx_{n})\ge f(\bl_k)$ due to the lower semicontinouity of $g$. Hence we can always find a subsequence $\{n_k\}$ s.t. $g(\xx_{n_k})\ge  f(\bl_k)$, which is equivalent to $\bl_k\in \mathcal{S}(\xx_{n_k})$.  
	\end{proof}
\end{lemma}
\begin{lemma}\label{lem:Fcont}
    When $\mathcal{S}$ is a bounded set in $\RR^K$, $F_{\mathcal{S}}(\cdot)$ is a continuous function.
\end{lemma}
\begin{proof} This is a direct application of Berge's maximum theorem. 
\end{proof}


\newpage
\section{A partitioning technique for large deviations }\label{app:partition}

In this section, we establish a theorem that is instrumental in the large deviation analysis of our algorithms.

\begin{theorem}\label{thm:11}
Let $\theta_0,\tilde{\theta}_0\in (0,1)$. Let $({\cal S}_{\theta,\gamma})_{\theta\in [\theta_0,1],\gamma\in [\tilde{\theta}_0,1]}$ and $(W_{\theta,\gamma})_{\theta\in [\theta_0,1],\gamma\in [\tilde{\theta}_0,1]}$ two collections of Borel sets in $[0,1]^K$ and $\Sigma$, respectively. 
We assume that

Suppose (i) for any $T\in \NN$, 
$$
\PP_{\bm}[\mathcal{E}]\le \sum_{t\ge\theta_0 T}^T\sum_{\tau\ge\tilde{\theta}_0 T}^T \PP_{\bm}[\hat{\bm}(t)\in \mathcal{S}_{\frac{t}{T},\frac{\tau}{T}},\om (t)\in W_{\frac{t}{T},\frac{\tau}{T}}],
$$
(ii) for any $\theta\in [\theta_0,1]\cap \QQ$, $\{\om (\theta T)\}_{T\ge 1}$ satisfies LDP upper bound (\ref{eq:rate UB}) with $I_\theta$, where $I_\theta$ is lower semi-continuous in $\Sigma$. \\
(iii) $\forall \theta,\gamma,\,{\cal S}_{\theta,\gamma}\neq \emptyset$. For all $\delta>0$, there exists $\eta>0$ such that if $\max\{\left|\theta'-\theta\right|,\left|\gamma'-\gamma\right|\}<\eta$, then
    $$
      \mathrm{dist}(\mathcal{S}_{\theta,\gamma},\mathcal{S}_{\theta',\gamma'})=\max\left\{\sup_{s\in {\cal S}_{\theta,\gamma}} \inf_{s'\in {\cal S}_{\theta',\gamma'}}\left\|s-s'\right\|_{\infty},\sup_{s'\in {\cal S}_{\theta',\gamma'}}\inf_{s\in {\cal S}_{\theta,\gamma}}\left\|s'-s\right\|_{\infty}\right\}.
    $$
Under Assumption \ref{ass1}, we have
\begin{equation}
    \lowlim_{T\to \infty} \frac{1}{T}\log \frac{1}{\PP_{\bm}[\mathcal{E}]}\ge \inf_{\theta\in [\theta_0,1]\cap \QQ}\inf_{\gamma\in [\tilde{\theta}_0,1]\cap \QQ}\inf_{\om\in \cl(W_{\theta,\gamma })}\theta \max\{F_{\mathcal{S}_{\theta,\gamma}}(\om),I_\theta (\om)\}.
\end{equation}
\end{theorem}

\begin{proof}
Without loss of generality, we can assume $\theta_0=\tilde{\theta}_0$. If $\theta_0<\tilde{\theta}_0$, we further define ${\cal S}_{\theta,\gamma}={\cal S}_{\theta,\tilde{\theta}_0}$ and $W_{\theta,\gamma}=W_{\theta,\tilde{\theta}_0}$ for $\theta_0\le\gamma<\tilde{\theta}_0$. And we handle the case $\theta_0>\tilde{\theta}_0$ similarly.

The main idea behind the proof is to partition the set of instants $t\in \{\theta_0 T,\ldots,T\}$ into a finite collection of instant sets such that if $t$ lies within one of these sets then we may bound $\PP_{\bm}[\hat{\bm}(t)\in \mathcal{S}_{\frac{t}{T},\frac{\tau}{T}},\om (t)\in W_{\frac{t}{T},\frac{\tau}{T}}]$ uniformly. From this partition, we can rewrite the upper bound $\PP_{\bm}[\mathcal{E}]$ as a finite sum. This sum is further upper bounded by a maximum over each of its terms. We conclude by taking the limit as $T$ grows large -- since we deal with the maximum over a finite number of terms, the limit and the maximum can be exchanged.   

\underline{Step 1. Partition of $[\theta_0,1]$.} To build this partition, we leverage the results of Lemmas \ref{lem:theta delta} and \ref{lem:theta delta mu}. Let $\delta>0$. We construct $N_\delta$ points $\theta_1,\ldots,\theta_{N_\delta}$ as follows: $\theta_{N_\delta}=1$ and
$$
\theta_n = (1-{\delta\over 2})^{-n} \theta_0, \qquad \forall n=1,\ldots, N_{\delta}-1.
$$
$N_\delta$ is chosen as $\min\{ p\in \mathbb{N}: \theta_0(1-{\delta\over 2})^{-p}\ge 1 \}$. Now observe by construction that:
\begin{align}
&\cup_{n=1}^{N_\delta} [\theta_{n-1},\theta_n] = [\theta_0,1],\\
&\forall n, (\theta \in [\theta_{n-1},\theta_n]) \Longrightarrow  (\theta \in [\theta_n(1-{\delta\over 2}),\theta_n]).\label{eq:thththg}
\end{align}

\underline{Step 2. Uniform upper bounds of $\PP_{\bm}[\hat{\bm}(t)\in \mathcal{S}_{\frac{t}{T},\frac{\tau}{T}},\om (t)\in W_{\frac{t}{T},\frac{\tau}{T}}]$.} We define the following sets: for all $n,m=1,\ldots,N_\delta$,
\begin{align*}
\bar{S}_{n,m}^\delta &= \cup_{\theta \in [\theta_{n-1},\theta_n]\cap\mathbb{Q}}\cup_{\gamma \in [\theta_{m-1},\theta_m]\cap\mathbb{Q}} \left\{ \bar{s}\in [0,1]^K: \exists s\in {\cal S}_{\theta,\gamma}: \|\bar{s}-s\|_\infty\le \delta\right\},\\
\bar{W}_{n,m}^\delta &= \cup_{\theta \in [\theta_{n-1},\theta_n]\cap\mathbb{Q}} \cup_{\gamma \in [\theta_{m-1},\theta_m]\cap\mathbb{Q}}\left\{ \bar{w}\in \Sigma: \exists w\in W_{\theta,\gamma}: \|\bar{w}-w\|_\infty\le \delta\right\}.
\end{align*}

Let $\theta=t/T$, $\gamma=\tau/T$, and assume that $\theta\in [\theta_{n-1},\theta_n], \gamma\in [\theta_{m-1},\theta_m]$. Then $\hat{\mu}(t)\in {\cal S}_{{t\over T},{\tau\over T}}$ implies that $\hat{\mu}(\theta_n)\in \bar{S}_{n,m}^\delta$ from Lemma \ref{lem:theta delta mu}. Similarly, $\om (t)\in W_{\frac{t}{T},\frac{\tau}{T}}$ implies that $\om(\theta_n)\in \bar{W}_{n,m}^\delta$ from Lemma \ref{lem:theta delta}. We conclude that: for all $t,\tau$ such that $t/T\in [\theta_{n-1}T,\theta_n T]$ and $\tau/T\in [\theta_{m-1}T,\theta_m T]$,
$$
\PP_{\bm}[\hat{\bm}(t)\in \mathcal{S}_{\frac{t}{T},\frac{\tau}{T}},\om (t)\in W_{\frac{t}{T},\frac{\tau}{T}}]\le \PP_{\bm}[\hat{\bm}(\theta_n T)\in \bar{S}_n^\delta,\om (\theta_n T)\in \bar{W}_{n,m}^\delta].
$$

\underline{Step 3. Upper bound on $\PP_{\bm}[\mathcal{E}]$.} We denote by $p_n$ the number of instants $t$ such that ${t\over T}\in [\theta_{n-1}T,\theta_n T]$. From the above inequality, we conclude that:
\begin{align*}
\PP_{\bm}[\mathcal{E}] &\le \sum_{n=1}^{N_\delta} \sum_{m=1}^{N_\delta} p_np_m \PP_{\bm}[\hat{\bm}(\theta_n T)\in \bar{S}_{n,m}^\delta,\om (\theta_n)\in \bar{W}_{n,m}^\delta],\\
&\le (\sum_{n=1}^{N_\delta}p_n) (\sum_{m=1}^{N_\delta}p_m) \max_{n,m\in \{1,\ldots,N_{\delta}\}}  \PP_{\bm}[\hat{\bm}(\theta_n T)\in \bar{S}_{n,m}^\delta,\om (\theta_n T)\in \bar{W}_{n,m}^\delta].
\end{align*}
We note that $\sum_{n=1}^{N_\delta}p_n$ is roughly equal to $T$, and always smaller than $T+2N_{\delta}$. Taking the logarithm, dividing by $-T$, and the liminf of the above inequality, we get:
\begin{align*}
\lowlim_{T\to\infty} {1\over T}\log{1\over \PP_{\bm}[\mathcal{E}]} &\ge  \lowlim_{T\to\infty} {1\over T}\min_{n,m\in \{1,\ldots,N_{\delta}\} }\log{1\over \PP_{\bm}[\hat{\bm}(\theta_n T)\in \bar{S}_{n,m}^\delta,\om (\theta_n T)\in \bar{W}_{n,m}^\delta]},\\
& = \min_{n\in \{1,\ldots,N_{\delta}\} }\lowlim_{T\to\infty} {1\over T}\log{1\over \PP_{\bm}[\hat{\bm}(\theta_n T)\in \bar{S}_{n,m}^\delta,\om (\theta_n T)\in \bar{W}_{n,m}^\delta]},\\
&\ge \min_{n,m\in \{1,\ldots,N_{\delta}\} }\theta_n\inf_{\om\in \cl(\bar{W}_{n,m}^\delta)} \max\{ F_{\bar{S}_{n,m}^\delta}(\om),I_{\theta_n}(\om)\},
\end{align*}
where the last inequality follows from Theorem \ref{thm:Varadaham bandits} with $\mathcal{S}=\bar{S}_{n,m}^\delta$ and $W=\bar{W}_{n,m}^\delta$.

\underline{Step 4. Continuity arguments.} The last step consists in proving that:
\begin{align*}
\limsup_{\delta\to 0} \min_{n,m\in \{1,\ldots,N_{\delta}\} } & \inf_{\om\in \cl(\bar{W}_{n,m}^\delta)}\theta_n \max\{ F_{\bar{S}_{n,m}^\delta}(\om),I_{\theta_n}(\om)\}\\ 
& \ge \inf_{\theta,\gamma\in [\theta_0,1]\cap \QQ}\inf_{\om\in \cl(W_{\theta,\gamma})}\theta\max\{F_{\mathcal{S}_{\theta,\gamma}}(\om),I_\theta (\om)\}.
\end{align*}

We first state two uniform continuity results, proved in Lemma \ref{lem:21}: For any $\epsilon>0,\,\theta,\gamma\in [\theta_0,1]$, there exists $\delta>0$ such that
\begin{align}
\forall \om, \forall n,m=1,\ldots,N_\delta, \quad  & F_{\bar{S}_{n,m}^\delta}(\om) \ge F_{{\cal S}_{\theta_n,\theta_m}}(\om)-\epsilon, \label{eq:unif1g} \\
 \forall \om,\om': \|\om-\om'\|_\infty\le \delta, \quad & F_{{\cal S}_{\theta,\gamma}}(\om') \ge  F_{{\cal S}_{\theta,\gamma}}(\om) -\epsilon,\ \ I_\theta(\om')\ge I_\theta(\om)-\epsilon.\label{eq:unif2g}
\end{align}
(\ref{eq:unif1g}) is the consequence of (iii) and Lemma~\ref{lem:21}. (\ref{eq:unif2g}) immediately follows from the compactness of $\Sigma$, lower semi-continuity of $I_{\theta},$ and $F_{\mathcal{S}_{\theta,\gamma}}$ (see Lemma~\ref{lem:Fcont} and (ii)).
We fix such a $\delta$, and consider any convergent sequence $(n_k,m_k,\om_k)_k$ with values in $\{1,\ldots,N_\delta\}^2\times \Sigma$ such that if $n_k=n,m_k=m$ then $\om_k\in \bar{W}_{n,m}^\delta$. Denote by $(n,m,\om_0)$ its limit. We let:
$$
g^\star = \lowlim_{k\to\infty}\theta_{n_k} \max(F_{\bar{S}_{n_k,m_k}^\delta}(\om_k), I_{\theta_{n_k}}(\om_k)).
$$
Then, we have:
\begin{align*}
g^\star &\ge \theta_n\max(F_{\bar{S}_{n,m}^\delta}(\om_0), I_{\theta_{n}}(\om_0)) -\epsilon\\ 
&\ge \theta_n\max(F_{{\cal S}_{\theta_n,\theta_m}}(\om_0), I_{\theta_{n}}(\om_0)) -2\epsilon\\ 
&\ge\theta_n \max(F_{{\cal S}_{\theta_n,\theta_m}}(\om), I_{\theta_{n}}(\om)) - 3\epsilon,
\end{align*}
for some $\theta\in [\theta_{n-1},\theta_n],\,\gamma\in [\theta_{m-1},\theta_{m}],\,\om\in \cl(W_{\theta,\gamma})$. The first inequality is due to (\ref{eq:unif2g}), the second to (\ref{eq:unif1g}), and the third to the fact that $\omega_0\in \bar{W}_{n,m}^\delta$ and (\ref{eq:unif2g}). We conclude that:
$$
g^\star \ge  \inf_{\theta,\gamma\in [\theta_0,1]\cap \QQ}\inf_{\om\in \cl(W_{\theta,\gamma})}\theta\max\{F_{\mathcal{S}_{\theta,\gamma}}(\om),I_\theta (\om)\} - 3\epsilon.
$$
\end{proof}

\begin{lemma}\label{lem:21}
    Assume $(\mathcal{S}_{\theta,\gamma})_{\theta \in [\theta_0,1],\gamma\in [\tilde{\theta}_0,1]}$ is a collection of nonempty sets in $[0,1]^K$ that satisfies $\forall \delta>0$, there exists $\eta>0$ such that if $\max\{\left|\theta'-\theta\right|,\left|\gamma'-\gamma\right|\}<\eta$, then
    $$
    \mathrm{dist}(\mathcal{S}_{\theta,\gamma},\mathcal{S}_{\theta',\gamma'})=\max\left\{\sup_{s\in {\cal S}_{\theta,\gamma}} \inf_{s'\in {\cal S}_{\theta',\gamma'}}\left\|s-s'\right\|_{\infty},\sup_{s'\in {\cal S}_{\theta',\gamma'}}\inf_{s\in {\cal S}_{\theta,\gamma}}\left\|s'-s\right\|_{\infty}\right\}.
    $$
   
   Then (\ref{eq:unif1g}) holds.
\end{lemma}

\begin{proof}
Recall that $F_{\mathcal{S}_\theta}(\om)=\inf_{\bl\in \cl (\mathcal{S}_\theta)} \Psi(\bl,\om)=\inf_{\bl\in \cl (\mathcal{S}_\theta)}\sum_{k}\omega_k d(\lambda_k,\mu_k)$. $\Psi$ is uniformly continuous in $[0,1]^K\times \Sigma$, and hence:
$$
\forall\epsilon>0, \exists \bar{\delta}: \forall \bl,\bl', \|\bl-\bl'\|_\infty\le 2\bar{\delta} \Rightarrow |\Psi(\bl,\om)-\Psi(\bl',\om)|\le \epsilon.
$$
Based on the assumption in the lemma, there exists $\eta>0$ such that if $\max\{\left|\theta'-\theta\right|,\left|\gamma'-\gamma\right|\}<\eta$, then 
$$
\mathrm{dist} (\mathcal{S}_{\theta,\gamma},\mathcal{S}_{\theta',\gamma'})<\bar\delta.
$$

Select $\delta< \min(\bar\delta, 2\eta)$. Let $n,m \in \{1,\ldots,N_\delta\}$. For $(\theta,\gamma)\in [\theta_{n-1},\theta_n]\times [\theta_{m-1},\theta_m]$, we have $\max ( |\theta-\theta_n|, |\gamma-\theta_m|)\le \delta/2 <\eta$. This implies that:
$$
\mathrm{dist} (\mathcal{S}_{\theta,\gamma},\mathcal{S}_{\theta_n,\theta_m})<\bar\delta.
$$
And hence since $\mathcal{S}_{\theta_n,\theta_m}\subset \bar{S}_{n,m}^\delta$:
$$
\mathrm{dist} (\mathcal{S}_{\theta,\gamma},\bar{S}_{n,m}^\delta)< \bar\delta.
$$
We conclude that $\forall\om$, $\forall \bl \in \mathcal{S}_{\theta,\gamma}$, $\exists \bl'\in \bar{S}_{n,m}^\delta$,
$$
\Psi(\bl,\om) \ge \Psi(\bl',\om) -\epsilon.
$$
This completes the proof.

\end{proof}


\newpage
\section{A simple min-max problem}\label{app:maxmin}

The two following results are used in Appendix \ref{app_cr}.

\begin{lemma}\label{lem:simple sol new2}
	Let $b_1,c_1,b_2,c_2>0$. Introduce $f(x)=-b_1x+c_1$ and $g(x)=b_2x+c_2$, $\forall x\in \RR$. Then
	$$
	\inf_{x \in \RR}\max\{f(x),\,g(x)\}\ge f(x_0),
	$$
	where $x_0$ is the real number s.t. $x_0\ge 0,\, f(x_0)=g(x_0)$
\end{lemma}
\begin{proof}
As $g$ is an increasing function and $f$ is a decreasing function, we deduce that for all $x\ge x_0$, $\max\{f(x),\,g(x)\}\ge g(x)\ge g(x_0)=f(x_0)$. Similarly for all $x\le x_0$, $\max\{f(x),\,g(x)\}\ge f(x)\ge f(x_0)$.

\end{proof}
\begin{lemma}\label{lem:simple sol new1}
	Let $b_1,c_1,b_2,c_2>0$. Introduce $f(x)= -b_1x+c_1$ and $g(x)=[(c_2\sqrt{x}-b_2)_+]^2$ for $x\in \RR_+$. Then
	$$
\inf_{x \in \RR_+}\max\{ f(x), g(x)\} \ge f(x_0),
	$$
	where $x_0$ is the unique real number s.t. $x_0>0$ and $f(x_0)=g(x_0)$.
\end{lemma}
\begin{proof}
	We first prove the uniqueness of $x_0$. Observe that $g$ is an increasing function, whereas $f$ is a strictly decreasing function. From the definition, we can have $f(0)=c_1>0=g(0)$, hence there exists an unique point $x_0>0$ s.t. $f(x_0)=g(x_0)$. Leveraging the convexity of $g$, there exists a linear function $\underline{g}$ s.t. $g(x)\ge \underline{g}(x)$ and $g(x_0)=\underline{g}(x_0)$. The proof then follows from the fact that 
	$$
	\inf_{x \in \RR_+}\max\{ f(x), g(x)\} \ge \inf_{x \in \RR}\max\{ f(x), \underline{g}(x)\} \ge f(x_0),
	$$
	where the last inequality is the application of Lemma \ref{lem:simple sol new2}.
\end{proof}


\newpage
\section{The LDP conjecture and its consequence}\label{app:conjecture}
In this section, we restate the conjecture mentioned in Section \ref{sec:varadaham}, and we discuss how it relates to the conjectured lower bound (\ref{eq:lb}).

\begin{conjecture}\label{conj}
	Assume that under some adaptive sampling algorithm, $\{ \om (t)\}_{t\ge 1}$ satisfies an LDP with rate function $L$. Then we have: for any non-empty subset $\mathcal{S}$ of $\RR^K$ and any subset $W$ of $\Sigma$,  
	$$
	\lim_{t\rightarrow \infty }   \frac{1}{t} \log \frac{1}{\PP_{\bm}\left[  \hat{\bm}(t)\in \cl (\mathcal{S}),\om(t)\in W \right] }= \inf_{\om\in W} \max \left\{ F_{\mathcal{S}}(\om),L(\om)\right\}.
	$$
\end{conjecture}

\medskip
\noindent
For simplicity, we assume that $\Lambda= \{\bm\in \RR^K:\mu_{1(\bm)}>\mu_k,\forall k\neq 1(\bm)\}$ and all the reward distributions are Gaussian.  Introduce the notation:
\begin{align*}
\Psi^\star_{\bm}&=\max_{\om\in \Sigma}\inf_{\bl\in \alt (\bm)} \Psi_{\bm}(\bl,\om),\\
\text{and }\om^\star (\bm)&=\argmax_{\om\in \Sigma}\inf_{\bl\in \alt (\bm)} \Psi_{\bm}(\bl,\om), 	
\end{align*}
where $\Psi_{\bm}(\bl,\om)=\sum_{k=1}^K\omega_k d(\lambda_k,\mu_k)$. Notice that the KL-divergence is symmetric in its arguments if the distributions are Gaussian. Hence the conjectured lower bound (\ref{eq:lb}) is exactly the same as that in the fixed confidence setting. The solution $\om^\star(\bm)$ to the corresponding optimization problem is unique (see Theorem 5 in \cite{garivier2016optimal}).

\medskip
\noindent
We consider the set of algorithms returning the best empirical arm $\hat{\imath} = 1(\hat{\bm}(T))$ and with error probability matching the conjectured lower bound (\ref{eq:lb}). If such an algorithm exists, under the assumption that Conjecture \ref{conj} is true, the rate function for the corresponding process $\{\om(T)\}_{T\ge 1}$ must satisfy: 
\begin{equation}\label{eq:L}
	\inf_{\om\in \Sigma} \max \left\{(\inf_{\bl\in \alt (\bm)}\Psi_{\bm}(\bl,\om)), L_{\bm}(\om)\right\} \ge \Psi^\star_{\bm},
\end{equation}
where $L_{\bm}$ is the rate function of a complete LDP under $\bm$ for the process $\{\om (T)\}_{T\ge 1}$. Lemma \ref{lem:conject} below shows that (\ref{eq:L}) implies the process $\{\om(T)\}_{T\ge 1}$ convergences to the optimal allocation.

\begin{lemma}\label{lem:conject}
	For $\bm\in \Lambda$, if there is a strategy satisfying (\ref{eq:L}), then
	$L_{\bm}(\om)\ge \Psi^\star_{\bm},\forall \om \neq \om^\star(\bm)$ and $L_{\bm}(\om^\star(\bm))=0$.
\end{lemma}
\begin{proof}
	Assume that, on the contrary, there is $\om' \neq \om^\star(\bm)$ s.t. $L_{\bm}(\om')<\Psi^\star_{\bm}$. Together with $\inf_{\bl\in \alt (\bm)}\Psi_{\bm}(\bl,\om')<\Psi_{\bm}^\star$, this implies that:
	\begin{align*}
	\inf_{\om\in \Sigma} \max \left\{(\inf_{\bl\in \alt (\bm)}\Psi_{\bm}(\bl,\om)), L_{\bm}(\om)\right\} 	\le \max \left\{  \inf_{\bl\in \alt (\bm)}\Psi_{\bm}(\bl,\om'), L_{\bm}(\om')\right\}<\Psi^\star_{\bm}.
	\end{align*}
	This contradicts the assumption, (\ref{eq:L}), so we have $L_{\bm}(\om)\ge \Psi^\star_{\bm},\forall \om \neq \om^\star(\bm)$. As for the optimal allocation, $\om^\star (\bm)$, the fact $\PP_{\bm}\left[\om(T)\in \Sigma\right]=1$ implies that
	\[
	0=\lim_{T\rightarrow \infty}\frac{1}{T}\log \frac{1}{\PP_{\bm}\left[ \om(T)\in \Sigma\right]}\ge \inf_{\om\in \Sigma}L_{\bm}(\om),
	\]
	where the last inequality stems from (\ref{eq:rate LB}) in Definition \ref{def:rate}. Since $L_{\bm}(\om)\ge \Psi^\star_{\bm},\forall \om \neq \om^\star(\bm)$, we conclude that $L_{\bm}(\om^\star(\bm))=0$.
\end{proof}

\medskip
\noindent
So far, we have investigated the consequence of matching the lower bound (\ref{eq:lb}) on a single instance (a single parameter $\bm$). Of course, we wish to get an algorithm matching (\ref{eq:lb}) for all instances. The following theorem shows that this is impossible even for two parameters.

\begin{theorem}\label{thm:conject}
	Consider $\bm,\bp\in\Lambda$ s.t. $\om^\star (\bm)\neq \om^\star (\bp)$ and $\max_{k\in [K]}d(\pi_k,\mu_k)<\Psi^\star_{\bm}$, then there is no strategy satisfying (i) and (ii) simultaneously:
\begin{itemize}
	\item[(i)] (\ref{eq:L}) holds for $\bp$
	\item[(ii)] (\ref{eq:L}) holds for $\bm$
\end{itemize}
\end{theorem}
\begin{proof}
Assume that, on the contrary, there is such a strategy. By the assumption on $\bm$ and $\bp$, there will be an open set $\mathcal{O}\subset \Sigma$ s.t. $\om^\star (\bp)\in \mathcal{O}$ but $\om^\star (\bm)\notin \mathcal{O}$. On the one hand, $L_{\bp}(\om^\star (\bp))=0$ by Lemma \ref{lem:conject} and (i). Recalling the LDP lower bound (\ref{eq:rate LB}) in Definition \ref{def:rate}, $L_{\bp}(\om^\star (\bp))=0$ directly implies that: 
	\begin{equation}\label{eq:contradiction 1}
\lim_{T\rightarrow\infty}	\PP_{\bp}[\om(T)\in\mathcal{O}]= 1.	
	\end{equation}
	
	\noindent
	On the other hand, (ii) and Lemma \ref{lem:conject} imply that $L_{\bm}(\om)\ge \Psi^\star_{\bm}$ if $\om \neq \om^\star (\bm)$, hence
	\begin{equation}\label{eq:contradition 2}
	\lowlim_{T\rightarrow\infty}\frac{1}{T}\log\frac{1}{\PP_{\bm}[\om(T)\in\mathcal{O}]}\ge \Psi^\star_{\bm}.
	\end{equation}
	Now applying a change-of-measure argument (see Lemma 1 in \cite{kaufmann2016complexity} or equation (6) in \cite{garivier2019explore}), one can derive
	\begin{equation}\label{eq:contradiction 3}
			\sum_{k=1}^K\E_{\bp}\left[ \omega_k(T)\right]d(\pi_k,\mu_k) \ge \frac{1}{T}\skl(\PP_{\bp}[\om(T)\in\mathcal{O}],\PP_{\bm}[\om(T)\in\mathcal{O}])
	\end{equation}
	Using the assumption that  $\max_{k\in [K]}d(\pi_k,\mu_k)<\Psi^\star_{\bm}$, the left-hand side of (\ref{eq:contradiction 3}) is strictly smaller $\Psi_{\bm}^\star$. However, by letting $T\rightarrow \infty $ on the r.h.s. of (\ref{eq:contradiction 3}), (\ref{eq:contradiction 1}) and (\ref{eq:contradition 2}) implies the limitinf is larger than $\Psi_{\bm}^\star$. 
This is a contradiction.
\end{proof}

The consequence of Theorem \ref{thm:conject} is that either our conjecture is true and in which case, for any algorithm there are two instances for which it cannot match the error lower bound (\ref{eq:lb}) or the conjecture is not true (the bound provided in Theorem \ref{thm:Varadaham bandits} is not tight). We finally note that recent results presented in \cite{degenne2023,wang2023uniformly} suggest that indeed the lower bound (\ref{eq:lb}) cannot be achieved.

\newpage
\section{Discussion on the conjectured lower bound (\ref{eq:lb})}\label{app:discussion}
In this section, we discuss two points: (i) (\ref{eq:lb}) indeed corresponds to the conjectured lower bound proposed by \cite{garivier2016optimal}, see their Section 7; (ii) however, as far as we know, there is no proof for (\ref{eq:lb}), but one can derive a lower bound by inverting $\max$ and $\inf$ in (\ref{eq:lb}).

\noindent
(i) Without loss of generality, assume that $\mu$ is such that 1 is the best arm. We start from (1) and show that this is equivalent to Garivier-Kaufmann's formula. First, it can be easily checked that in (1), we can replace $\Sigma$ by $\Sigma_{>0}=\{\om \in \Sigma:\omega_k>0,\forall k\in [K]\}$. Then, for any $\omega\in \Sigma_{>0}$, we have 
$$
\inf_{\lambda\in \alt (\bm)}  \sum_k\omega_kd(\lambda_k,\mu_k) =\min_{m\neq 1} \inf_{\mu_m<  x < \mu_1} \omega_1 d(x,\mu_1) + \omega_md(x,\mu_m).
$$
Indeed, we can decompose $\alt (\mu)$ as $\cup_{m\neq 1}\{\lambda\in \Lambda: \lambda_m> \lambda_1 \}$, and thus, we have:
$$
\inf_{\lambda\in \alt (\bm)}  \sum_k\omega_kd(\lambda_k,\mu_k) = 	\min_{m\neq 1}\inf_{\lambda_m>\lambda_1} \sum_k\omega_kd(\lambda_k,\mu_k).
$$
We conclude by observing that 
$$
\inf_{\lambda_m> \lambda_1} \omega_1 d(\lambda_1,\mu_1) + \omega_md(\lambda_m,\mu_m)=  \inf_{\mu_m< x < \mu_1} \omega_1 d(x,\mu_1) + \omega_kd(x,\mu_m),
$$
which holds for all families of distributions such that $x\mapsto d(x,y)$ is monotonic (decreasing before $y$ and increasing after $y$) -- this holds for Bernoulli, Gaussian, etc.

\medskip
\noindent
(ii) Consider a consistent algorithm, and denote by $\omega_k(\bl)$ the expected proportion of rounds where the algorithm selects arm $k$ under the probability $\PP_{\bl}$. Using the classical change-of-measure arguments, we get:
$$
\limsup_{T\rightarrow\infty}\log \frac{1}{{\PP}_{\bm}[\hat{\imath}\neq 1]} \le T \sum_k\omega_k(\bl)d(\lambda_k,\mu_k)\le T \max_{\om\in\Sigma} \sum_k\omega_kd(\lambda_k,\mu_k).
$$
We can only deduce that:
$$
\limsup_{T\rightarrow\infty}\frac{1}{T}\log\frac{1}{{\PP}_{\bm}[\hat{\imath}\neq 1]}  \le \inf_{\bl\in \alt(\bm)}\max_{\om\in \Sigma} \sum_k\omega_kd(\lambda_k,\mu_k).
$$
One cannot directly apply Sion's minimax theorem to derive (1) (as $\alt(\bm)$ is not a convex domain).

\newpage
\section{Numerical experiments}\label{app:exp}

We consider various problem instances to numerically evaluate the performance of \sred. In these instances, we vary the number of arms from 5 to 55; we use Bernoulli distributed rewards, and vary the shape of the arm-to-reward mapping. For each instance, we compare \crc and \cra to \sr\citep{audibert2010best}, \sh \citep{karnin2013almost}, and \ugape \citep{gabillon2012best}. As discussed in Section \ref{sec:related}, \ugape requires prior knowledge about a parameter depending on the underlying problem. We hence implement its heuristic version which estimates the parameter on the fly, such modification was suggested in previous works e.g. \citep{audibert2010best,karnin2013almost}. We implement all algorithms in {\tt Julia 1.7.3} and run all experiments on a machine with Apple M1 with 16 GB RAM.\footnote{Our Julia implementation can be found at \href{https://github.com/rctzeng/NeurIPS2023-CR}{\color{blue}https://github.com/rctzeng/NeurIPS2023-CR}.} The error probabilities averaged over $40,000$ independent runs. In all experiments, we set $\theta_0=10^{-5}$ for \sred. 

We vary the shape of the arm-to-reward function and consider one shape in each of the subsections below. We present the error probability of all algorithms in tables and figures. In the latter, the error probability is presented using the log scale, which sometimes makes the curves for some algorithms close to each other. In the tables, we present the error probability for a few budgets only, and there, we can see a clearer separation between the performance of the various algorithms.

Observe that our algorithms, \sred, perform better than the other algorithms for all arm-to-reward function shapes.

\subsection{One group of suboptimal arms}\label{num:1}
This instance is considered by \citep{karnin2013almost}: $\mu_1=0.5$ and $\mu_k=0.45$ for all $k\geq 2$. We can see that the performances of \sr, \crc, and \cra are significantly better than \ugape and \sh. 
\begin{figure}[h!]
		\centering
		\includegraphics[width=0.5\linewidth]{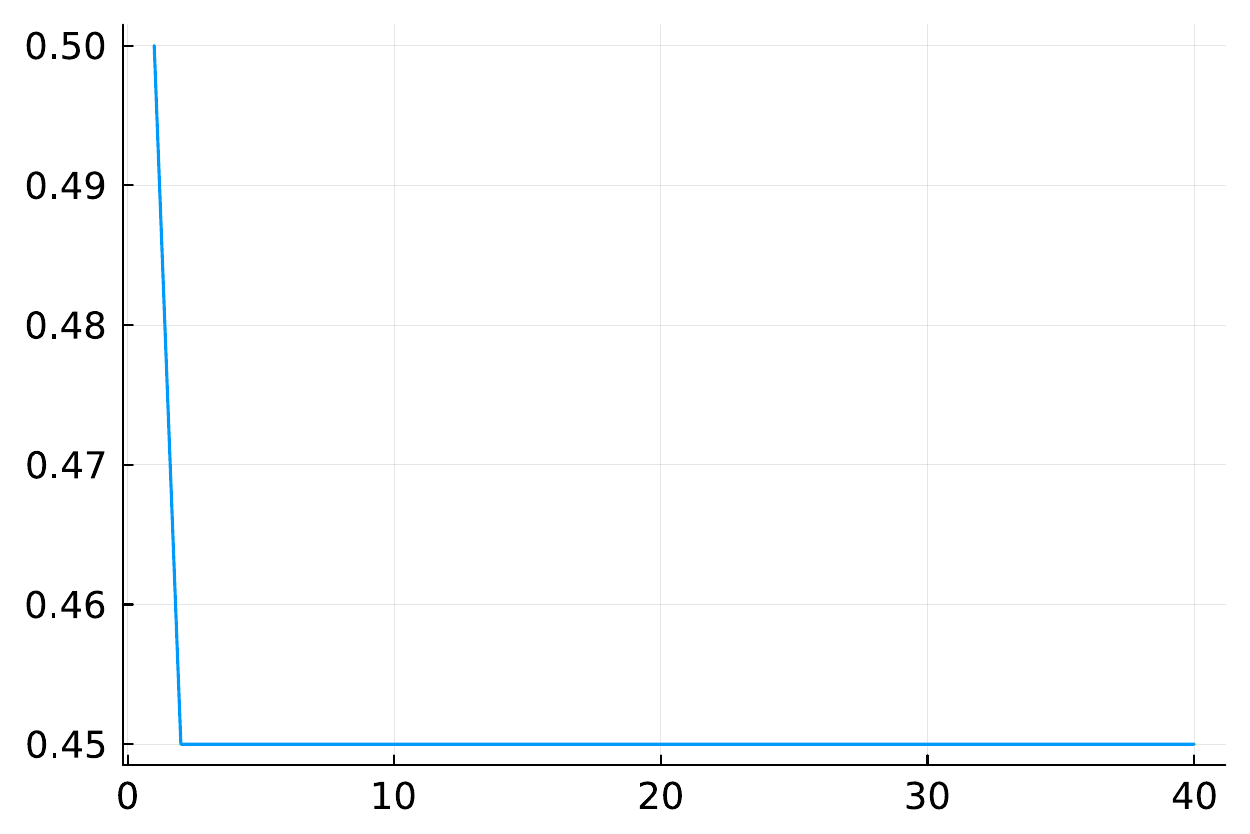}
		\caption{(One group of suboptimal arms) $\bm$ with $K=40$.}\label{fig:1_group-mu}
\end{figure}

\begin{table}[htb!]
\caption{(One group of suboptimal arms) error probabilities (in \%).\label{tab:1_group-mu}}
\centering
\begin{tabular}{lllccc} \toprule
  $K=10$ & & & $T=6,400$ & $T=7,200$ & $T=8,000$\\\midrule
  & \ugape & \citep{gabillon2012best} & $34.66$ & $33.22$ & $32.15$\\\hline
  & \sh &\cite{karnin2013almost} & $19.64$ & $16.85$ & $14.46$\\\hline
  & \sr &\cite{audibert2010best} & $7.86$ & $5.86$ & $4.29$\\\hline
  & \crc &(this paper) & $\mathbf{7.29}$ & $\mathbf{5.47}$ & $4.17$\\\hline
  & \cra &(this paper) & $7.37$ & $5.52$ & $\mathbf{4.07}$\\\bottomrule
\end{tabular}
\begin{tabular}{lllccc} \toprule
  $K=20$ & & & $T=12,000$ & $T=14,000$ & $T=16,000$\\\midrule
  & \ugape & \citep{gabillon2012best} & $42.77$ & $39.93$ & $38.62$\\\hline
  & \sh &\cite{karnin2013almost} & $25.17$ & $20.86$ & $17.94$\\\hline
  & \sr &\cite{audibert2010best} & $9.43$ & $6.59$ & $4.35$\\\hline
  & \crc &(this paper) & $\mathbf{8.39}$ & $\mathbf{5.92}$ & $\mathbf{4.08}$\\\hline
  & \cra &(this paper) & $8.80$ & $6.16$ & $4.42$\\\bottomrule
\end{tabular}
\begin{tabular}{lllccc} \toprule
  $K=40$ & & & $T=30,000$ & $T=35,000$ & $T=40,000$\\\midrule
  & \ugape & \citep{gabillon2012best} & $42.84$ & $40.46$ & $38.11$\\\hline
  & \sh &\cite{karnin2013almost} & $23.26$ & $19.24$ & $16.12$\\\hline
  & \sr &\cite{audibert2010best} & $6.51$ & $4.48$ & $3.14$\\\hline
  & \crc &(this paper) & $\mathbf{5.96}$ & $\mathbf{3.99}$ & $\mathbf{2.89}$\\\hline
  & \cra &(this paper) & $6.56$ & $4.27$ & $3.06$\\\bottomrule
\end{tabular}
\end{table}

\begin{figure}[h!]
	\centering
	\subcaptionbox{$K=10$ \label{fig:1_group-K10}}{
		\includegraphics[width=0.49\textwidth]{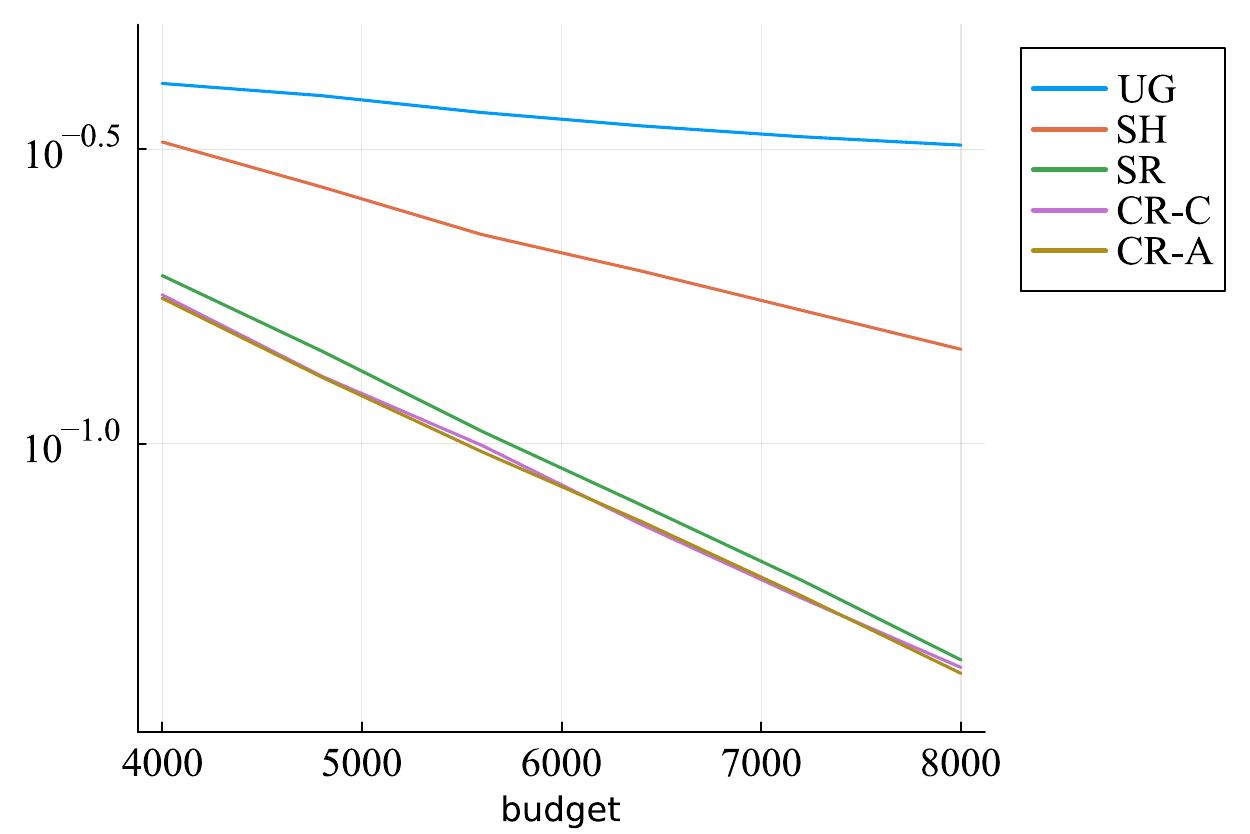}}
	\subcaptionbox{$K=20$ \label{fig:1_group-K20}}{
		\includegraphics[width=0.49\textwidth]{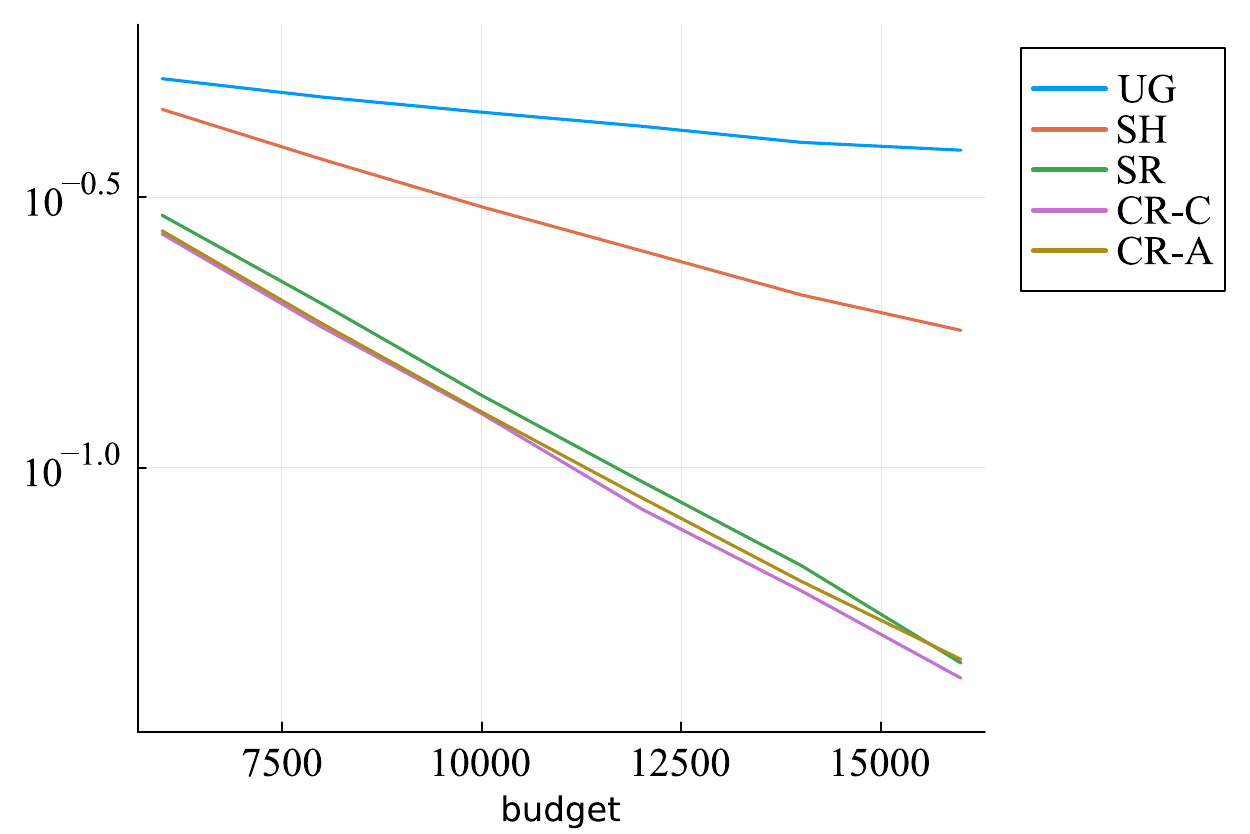}}
	\subcaptionbox{$K=40$ \label{fig:1_group-K40}}{
		\includegraphics[width=0.49\textwidth]{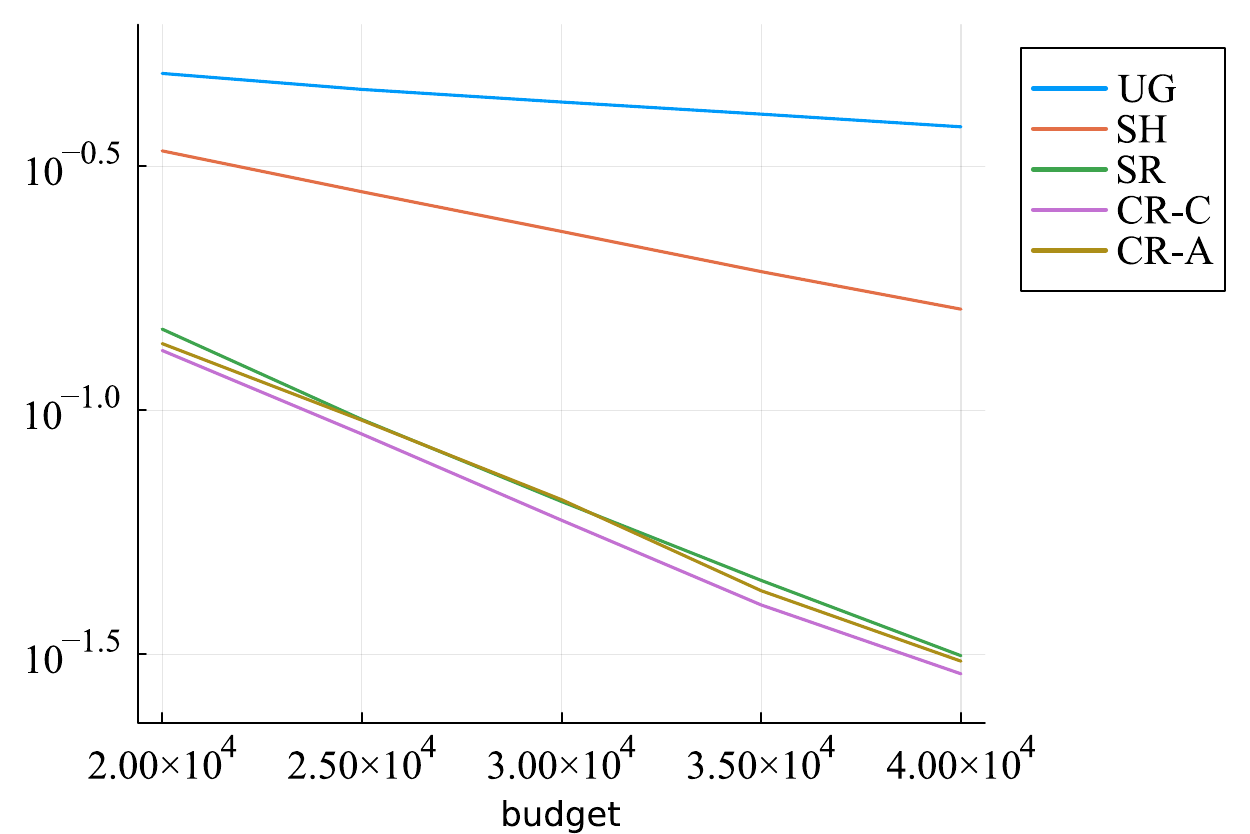}}
	\caption{(One group of suboptimal arms) error probabilities averaged over $40,000$ independent runs.}
\end{figure}

\clearpage\pagebreak
\subsection{Two groups of suboptimal arms}\label{num:2}
In this instance, we set $\mu_1=0.5$, $\mu_k=0.45$ for $k=2,\cdots,\lfloor \frac{K-1}{2}\rfloor$, and $\mu_k=0.4$ for $k=\lfloor \frac{K-1}{2}\rfloor+1,\cdots,K$. Compared to \ref{num:1}, where \crc is always the best, \cra becomes relatively better here.

\begin{figure}[h!]
	\centering
	\includegraphics[width=0.5\linewidth]{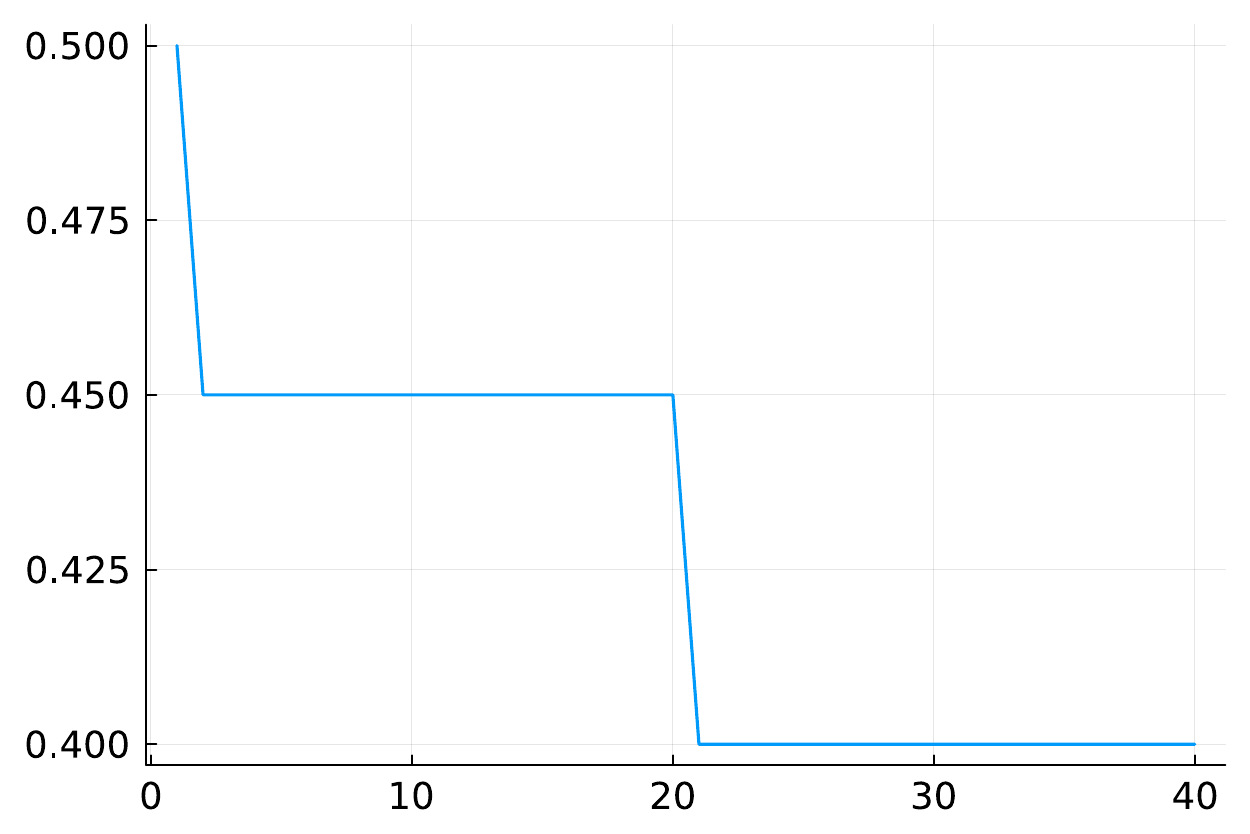}
	\caption{(Two groups of suboptimal arms) $\bm$ with $K=40$.}\label{fig:2_group-mu}
\end{figure}

\begin{table}[htb!]
\caption{(Two groups of suboptimal arms) error probabilities (in \%).
\label{tab:2_group-mu}}
\centering
\begin{tabular}{lllccc} \toprule
  $K=10$ & & & $T=5,600$ & $T=6,800$ & $T=8,000$\\\midrule
  & \ugape & \citep{gabillon2012best} & $25.50$ & $23.02$ & $20.97$\\\hline
  & \sh &\cite{karnin2013almost} & $11.12$ & $7.71$ & $5.35$\\\hline
  & \sr &\cite{audibert2010best} & $4.05$ & $2.30$ & $1.20$\\\hline
  & \crc &(this paper) & $3.68$ & $2.05$ & $1.13$\\\hline
  & \cra &(this paper) & $\mathbf{3.44}$ & $\mathbf{1.88}$ & $\mathbf{0.99}$\\\bottomrule
\end{tabular}
\begin{tabular}{lllccc} \toprule
  $K=20$ & & & $T=4,000$ & $T=7,000$ & $T=10,000$\\\midrule
  & \ugape & \citep{gabillon2012best} & $48.49$ & $40.68$ & $36.12$\\\hline
  & \sh & \cite{karnin2013almost} & $42.62$ & $25.69$ & $16.15$\\\hline
  & \sr &\cite{audibert2010best} & $28.26$ & $12.03$ & $5.03$\\\hline
  & \crc &(this paper) & $26.54$ & $10.62$ & $4.52$\\\hline
  & \cra &(this paper) & $\mathbf{26.00}$ & $\mathbf{10.61}$ & $\mathbf{4.27}$\\\bottomrule
\end{tabular}
\begin{tabular}{lllccc} \toprule
  $K=40$ & & & $T=15,000$ & $T=20,000$ & $T=25,000$\\\midrule
  & \ugape & \citep{gabillon2012best} & $46.04$ & $41.80$ & $38.30$\\\hline
  & \sh &\cite{karnin2013almost} & $27.24$ & $18.98$ & $13.82$\\\hline
  & \sr &\cite{audibert2010best} & $9.97$ & $5.03$ & $2.43$\\\hline
  & \crc &(this paper) & $\mathbf{9.12}$ & $4.49$ & $\mathbf{2.24}$\\\hline
  & \cra &(this paper) & $9.26$ & $\mathbf{4.78}$ & $2.38$\\\bottomrule
\end{tabular}
\end{table}

\begin{figure}[h!]
	\centering
	\subcaptionbox{$K=10$ \label{fig:2_group-K10}}{
		\includegraphics[width=0.49\textwidth]{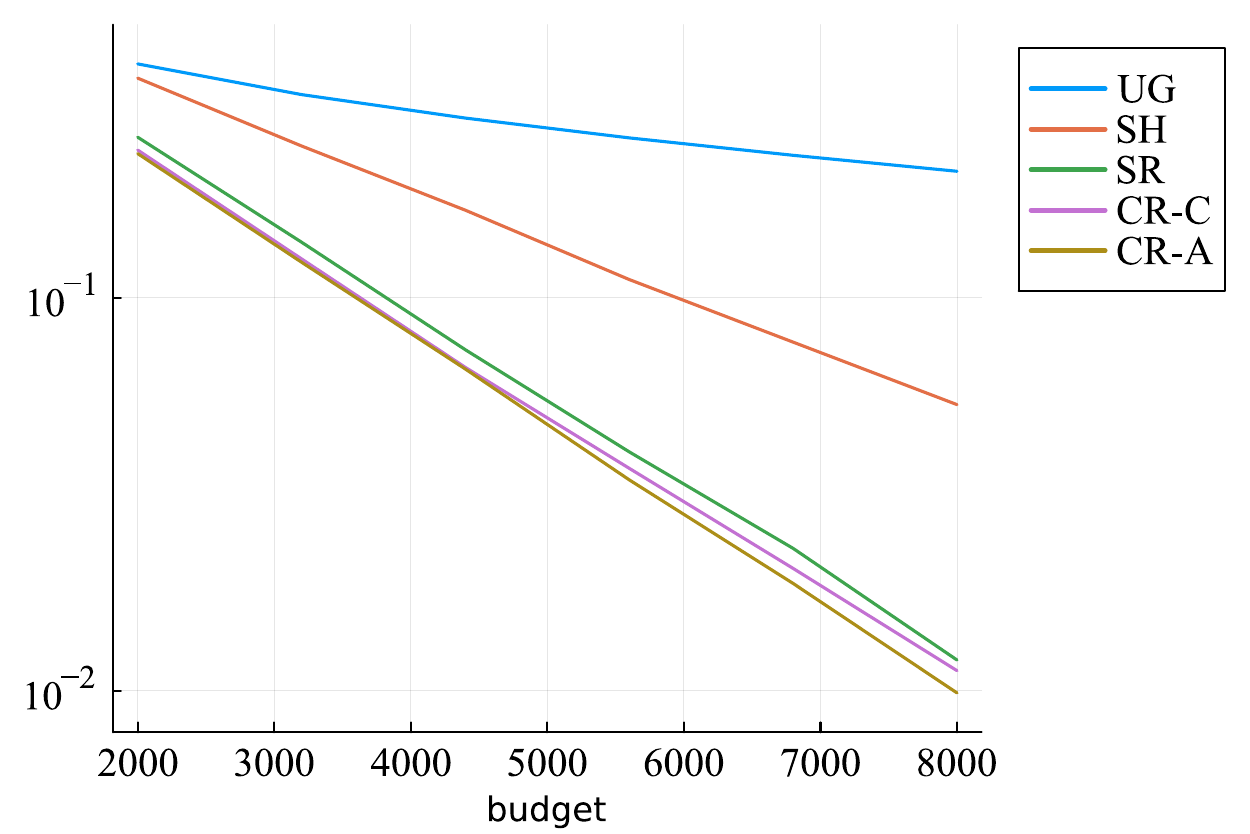}}
	\subcaptionbox{$K=20$ \label{fig:2_group-K20}}{
		\includegraphics[width=0.49\textwidth]{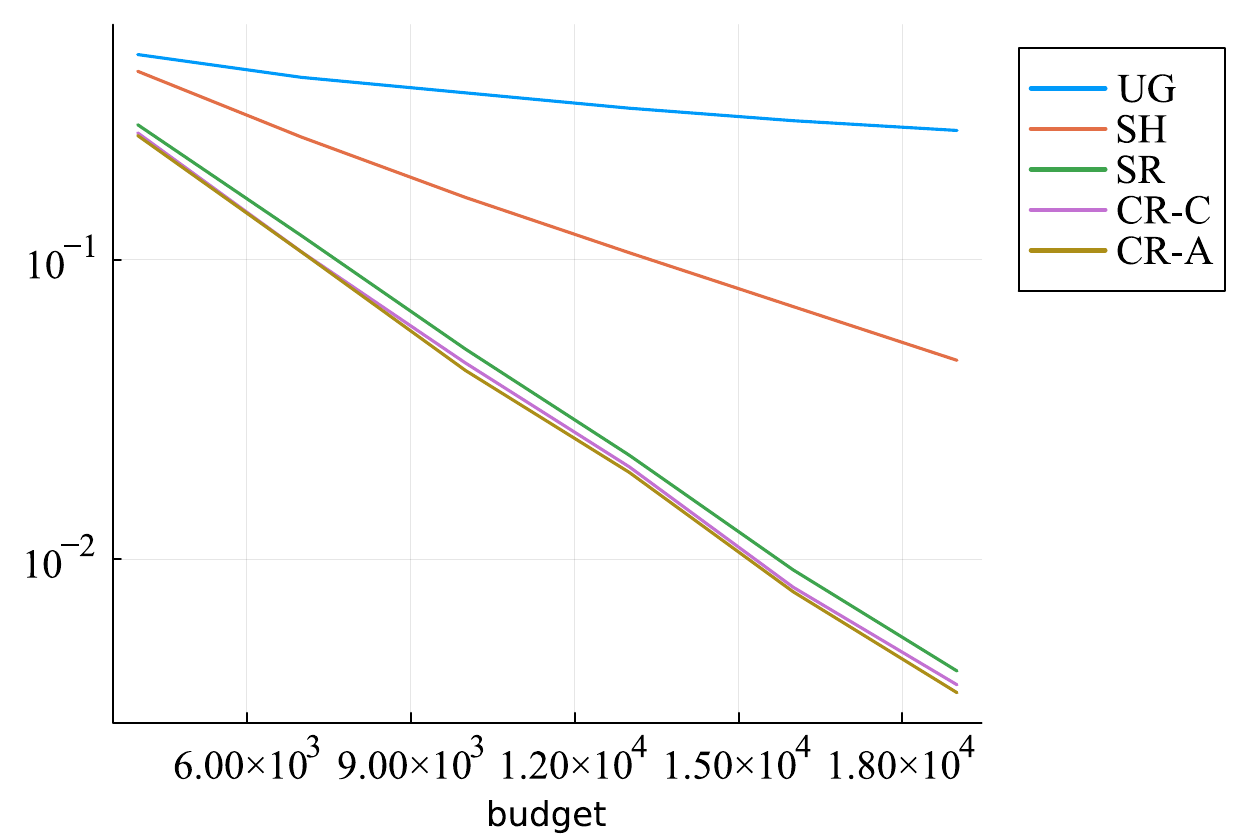}}
	\subcaptionbox{$K=40$ \label{fig:2_group-K40}}{
		\includegraphics[width=0.49\textwidth]{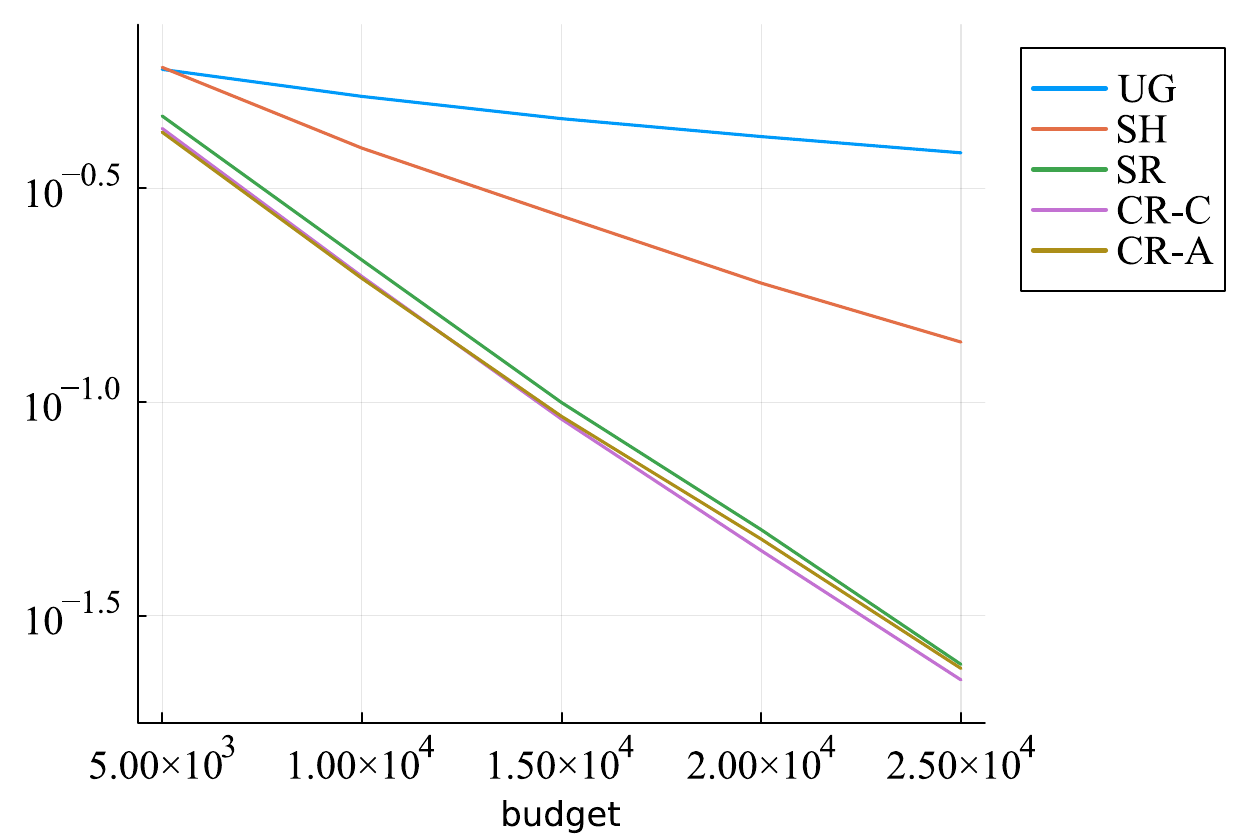}}
	\caption{(Two groups of suboptimal arms) error probabilities averaged over $40,000$ independent runs.}
\end{figure}

\clearpage\pagebreak
\subsection{Linear arm-to-reward function}\label{num:3}
In this instance, we set $\mu_k=\frac{3}{4}-\frac{k-1}{2K}$ for $k=1,\cdots,K$. Here \cra does the best. 

\begin{figure}[h!]
	\centering
	\includegraphics[width=0.5\linewidth]{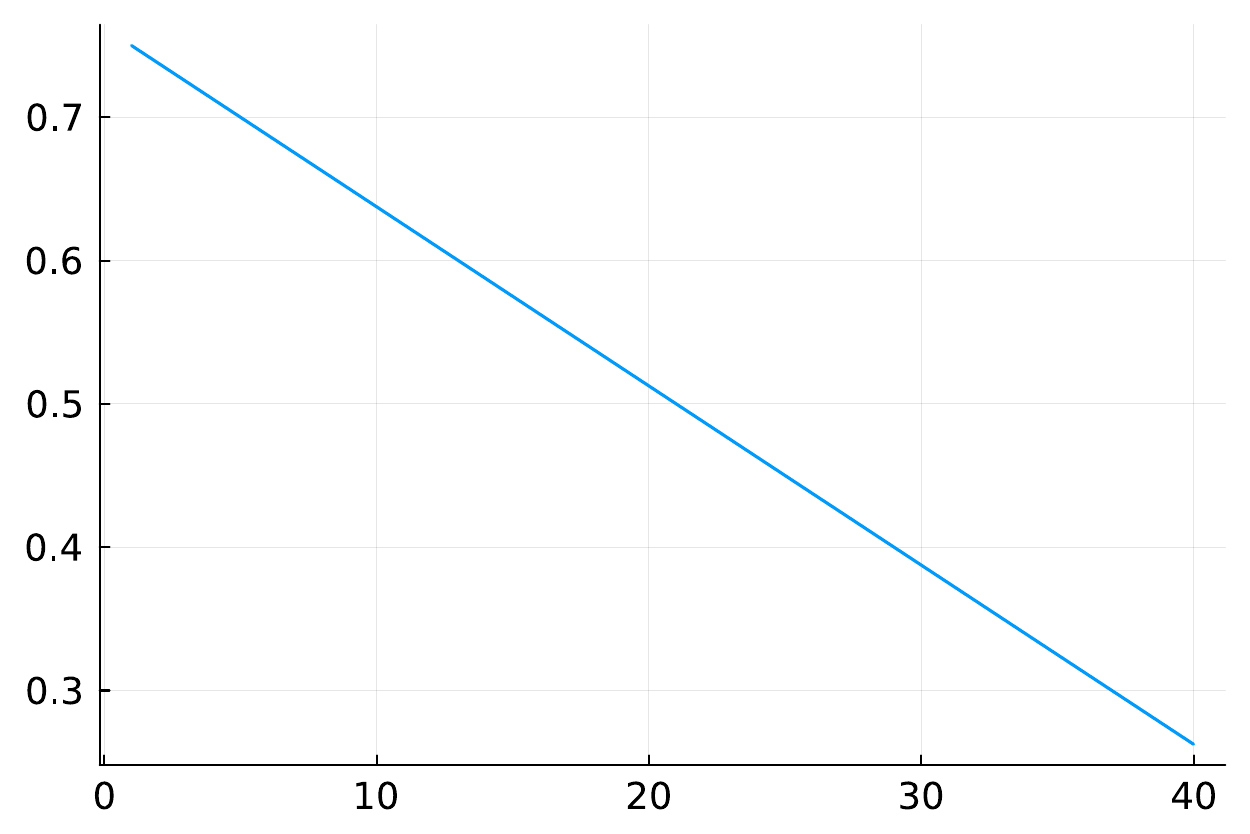}
	\caption{(Linear arm-to-reward function) $\bm$ with $K=40$.}\label{fig:liner-mu}
\end{figure}

\begin{table}[htb!]
\caption{(Linear arm-to-reward function) error probability (in \%).
\label{tab:linear-mu}}
\centering
\begin{tabular}{lllccc} \toprule
  $K=10$ & & & $T=3,200$ & $T=3,600$ & $T=4,000$\\\midrule
  & \ugape & \citep{gabillon2012best} & $4.96$ & $4.22$ & $3.24$\\\hline
  & \sh &\cite{karnin2013almost} & $4.97$ & $4.21$ & $3.49$\\\hline
  & \sr &\cite{audibert2010best} & $2.09$ & $1.53$ & $1.03$\\\hline
  & \crc &(this paper) & $1.59$ & $1.04$ & $0.83$\\\hline
  & \cra &(this paper) & $\mathbf{1.20}$ & $\mathbf{0.81}$ & $\mathbf{0.55}$\\\bottomrule
\end{tabular}
\begin{tabular}{lllccc} \toprule
  $K=20$ & & & $T=6,000$ & $T=8,000$ & $T=10,000$\\\midrule
  & \ugape & \citep{gabillon2012best} & $15.49$ & $11.20$ & $8.78$\\\hline
  & \sh &\cite{karnin2013almost} & $16.00$ & $12.66$ & $9.70$\\\hline
  & \sr &\cite{audibert2010best} & $10.76$ & $7.57$ & $5.24$\\\hline
  & \crc &(this paper) & $10.03$ & $6.96$ & $4.68$\\\hline
  & \cra &(this paper) & $\mathbf{8.78}$ & $\mathbf{5.72}$ & $\mathbf{3.73}$\\\bottomrule
\end{tabular}
\begin{tabular}{lllccc} \toprule
  $K=40$ & & & $T=15,000$ & $T=20,000$ & $T=25,000$\\\midrule
  & \ugape & \citep{gabillon2012best} & $25.09$ & $20.21$ & $16.44$\\\hline
  & \sh &\cite{karnin2013almost} & $25.70$ & $21.19$ & $17.92$\\\hline
  & \sr &\cite{audibert2010best} & $20.29$ & $15.86$ & $13.07$\\\hline
  & \crc &(this paper) & $19.93$ & $15.56$ & $12.30$\\\hline
  & \cra &(this paper) & $\mathbf{17.99}$ & $\mathbf{13.74}$ & $\mathbf{10.67}$\\\bottomrule
\end{tabular}
\end{table}

\begin{figure}[h!]
	\centering
	\subcaptionbox{$K=10$ \label{fig:linear-K10}}{
		\includegraphics[width=0.49\textwidth]{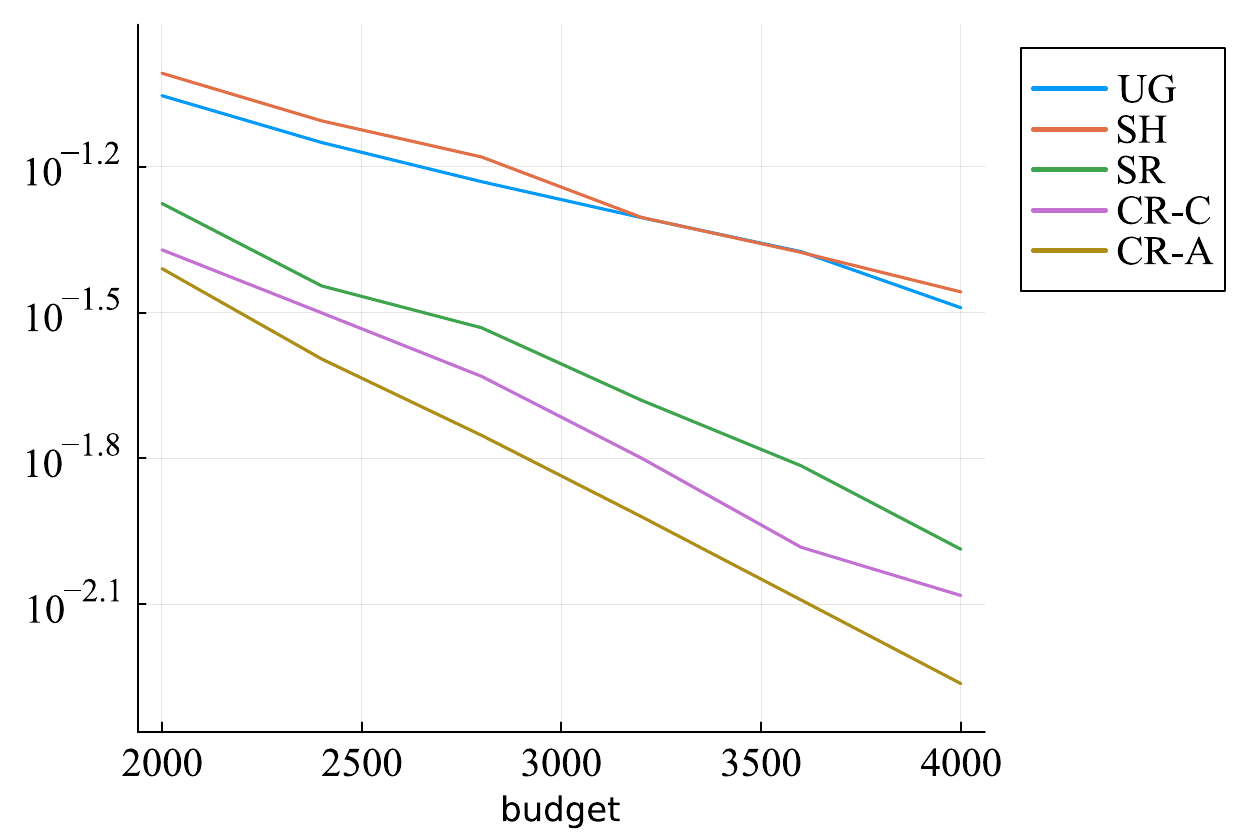}}
	\subcaptionbox{$K=20$ \label{fig:linear-K20}}{
		\includegraphics[width=0.49\textwidth]{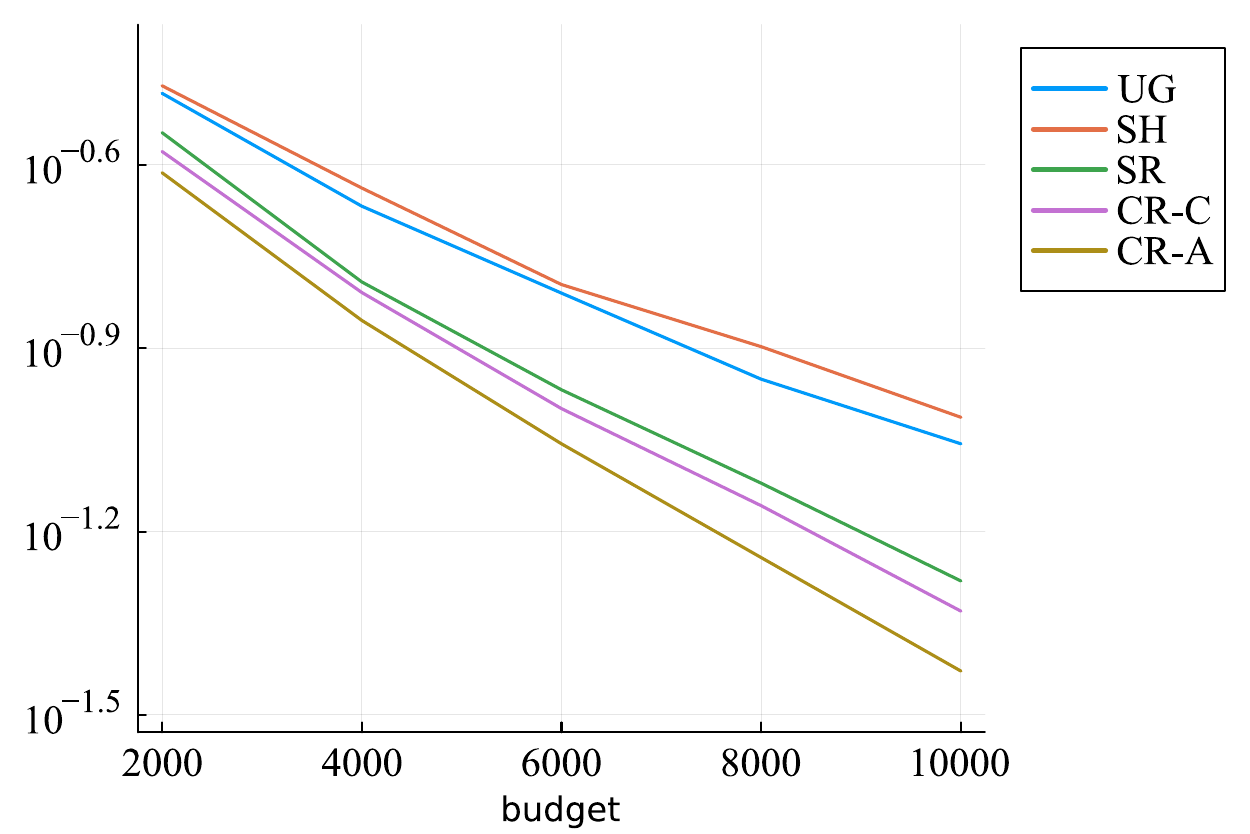}}
	\subcaptionbox{$K=40$ \label{fig:linear-K40}}{
		\includegraphics[width=0.49\textwidth]{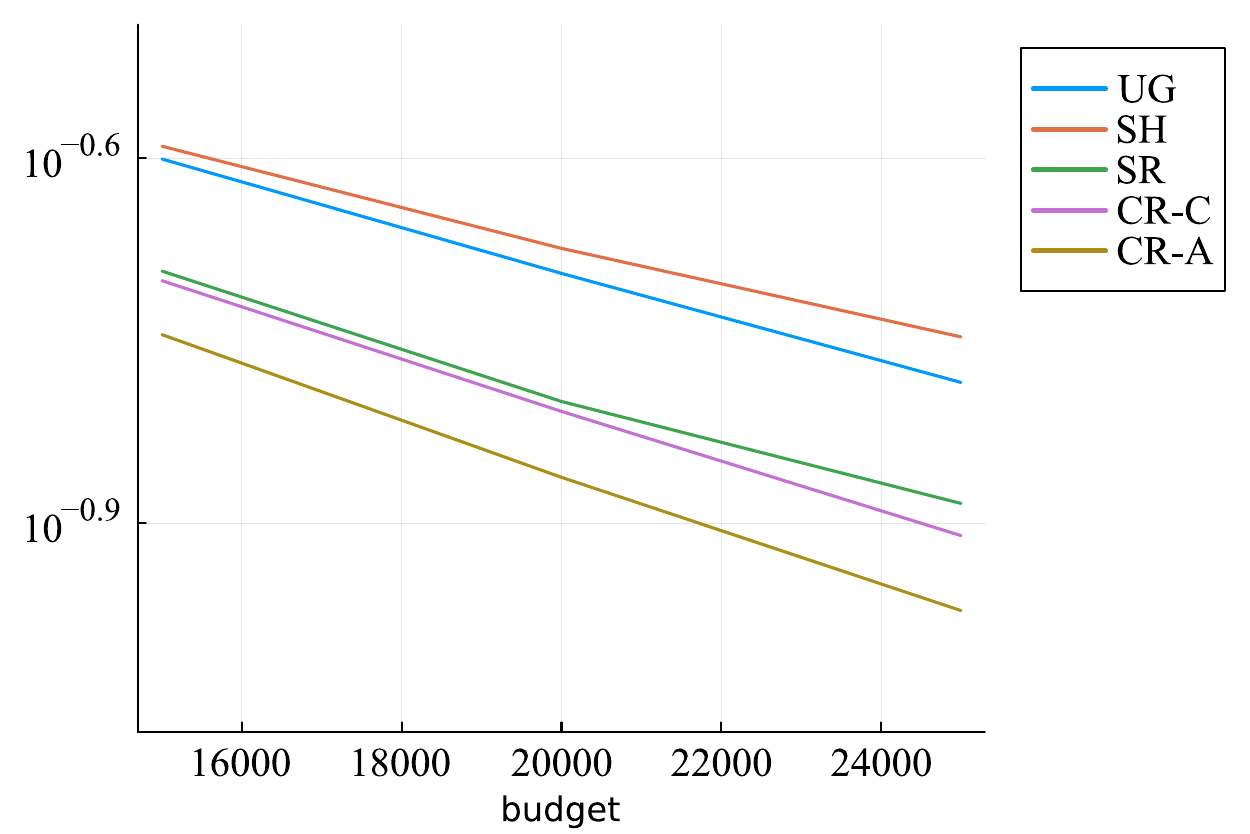}}
	\caption{(Linear arm-to-reward function) error probabilities averaged over $40,000$ independent runs.}
\end{figure}

\clearpage\pagebreak
\subsection{Concave arm-to-reward function}\label{num:4}
In this instance, we set $\mu_1=\sin(\frac{(K-1)\pi}{2K})$ and $\mu_k=\sin(\frac{9\pi(K-k+1)}{20K})$ for $k=2,\cdots,K$. \cra does the best in this instance all the time.
\begin{figure}[h!]
	\centering
	\includegraphics[width=0.5\linewidth]{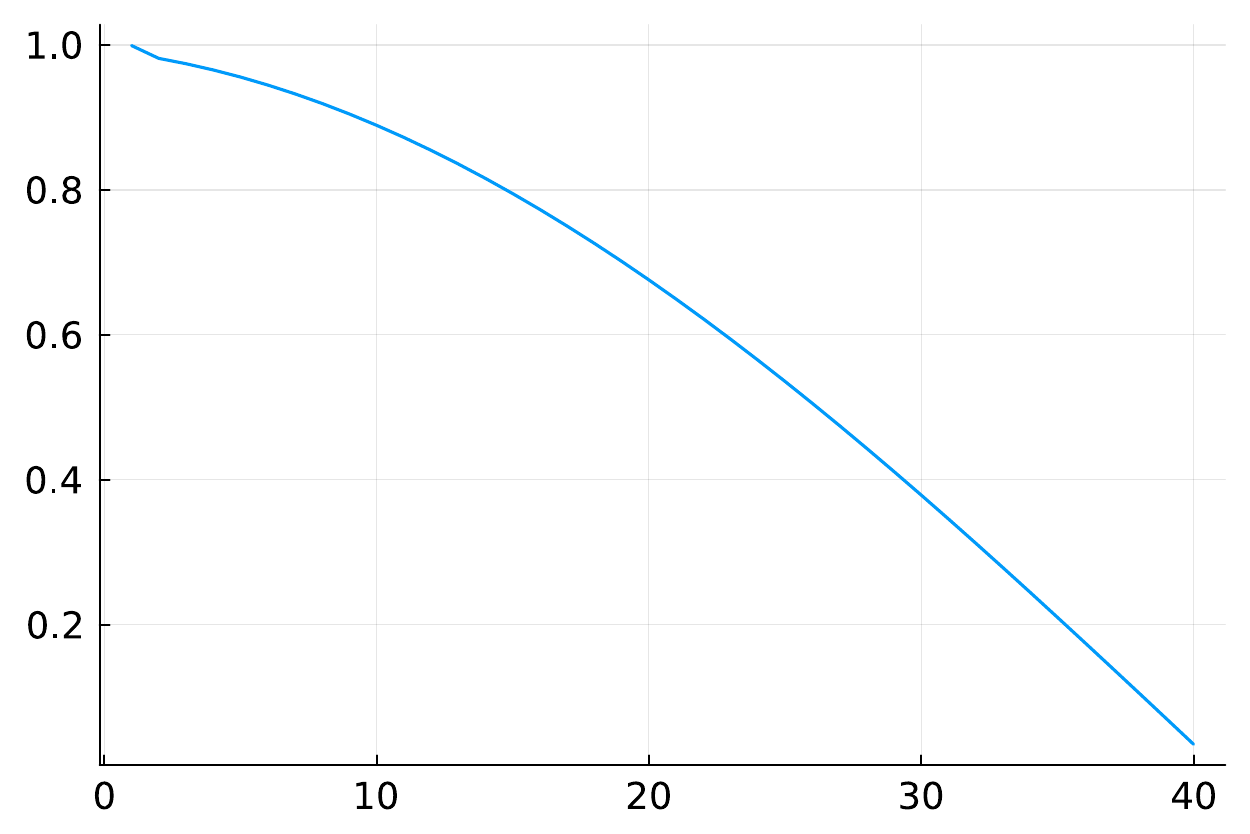}
	\caption{(Concave arm-to-reward function) $\bm$ with $K=40$.}\label{fig:concave-mu}
\end{figure}

\begin{table}[htb!]
\caption{(Concave arm-to-reward function) error probability (in \%).
\label{tab:concave-mu}}
\centering
\begin{tabular}{lllccc} \toprule
  $K=10$ & & & $T=900$ & $T=1,400$ & $T=1,900$\\\midrule
  & \ugape & \citep{gabillon2012best} & $2.23$ & $1.36$ & $0.94$\\\hline
  & \sh &\cite{karnin2013almost} & $4.68$ & $2.15$ & $0.89$\\\hline
  & \sr &\cite{audibert2010best} & $1.85$ & $0.54$ & $0.20$\\\hline
  & \crc &(this paper) & $1.24$ & $0.35$ & $0.10$\\\hline
  & \cra &(this paper) & $\mathbf{0.94}$ & $\mathbf{0.21}$ & $\mathbf{0.04}$\\\bottomrule
\end{tabular}
\begin{tabular}{lllccc} \toprule
  $K=20$ & & & $T=900$ & $T=1,400$ & $T=1,900$\\\midrule
  & \ugape & \citep{gabillon2012best} & $2.44$ & $1.85$ & $1.59$\\\hline
  & \sh &\cite{karnin2013almost} & $6.62$ & $2.62$ & $1.28$\\\hline
  & \sr &\cite{audibert2010best} & $2.81$ & $0.86$ & $0.31$\\\hline
  & \crc &(this paper) & $1.87$ & $0.47$ & $0.14$\\\hline
  & \cra &(this paper) & $\mathbf{1.36}$ & $\mathbf{0.36}$ & $\mathbf{0.09}$\\\bottomrule
\end{tabular}
\begin{tabular}{lllccc} \toprule
  $K=40$ & & & $T=2,400$ & $T=2,800$ & $T=3,200$\\\midrule
  & \ugape & \citep{gabillon2012best} & $1.03$ & $0.98$ & $0.94$\\\hline
  & \sh &\cite{karnin2013almost} & $1.26$ & $0.60$ & $0.35$\\\hline
  & \sr &\cite{audibert2010best} & $0.23$ & $0.10$ & $0.02$\\\hline
  & \crc &(this paper) & $0.18$ & $0.08$ & $0.04$\\\hline
  & \cra &(this paper) & $\mathbf{0.08}$ & $\mathbf{0.03}$ & $\mathbf{0.02}$\\\bottomrule
\end{tabular}
\end{table}

\begin{figure}[h!]
	\centering
	\subcaptionbox{$K=10$ \label{fig:concave-K10}}{
		\includegraphics[width=0.48\textwidth]{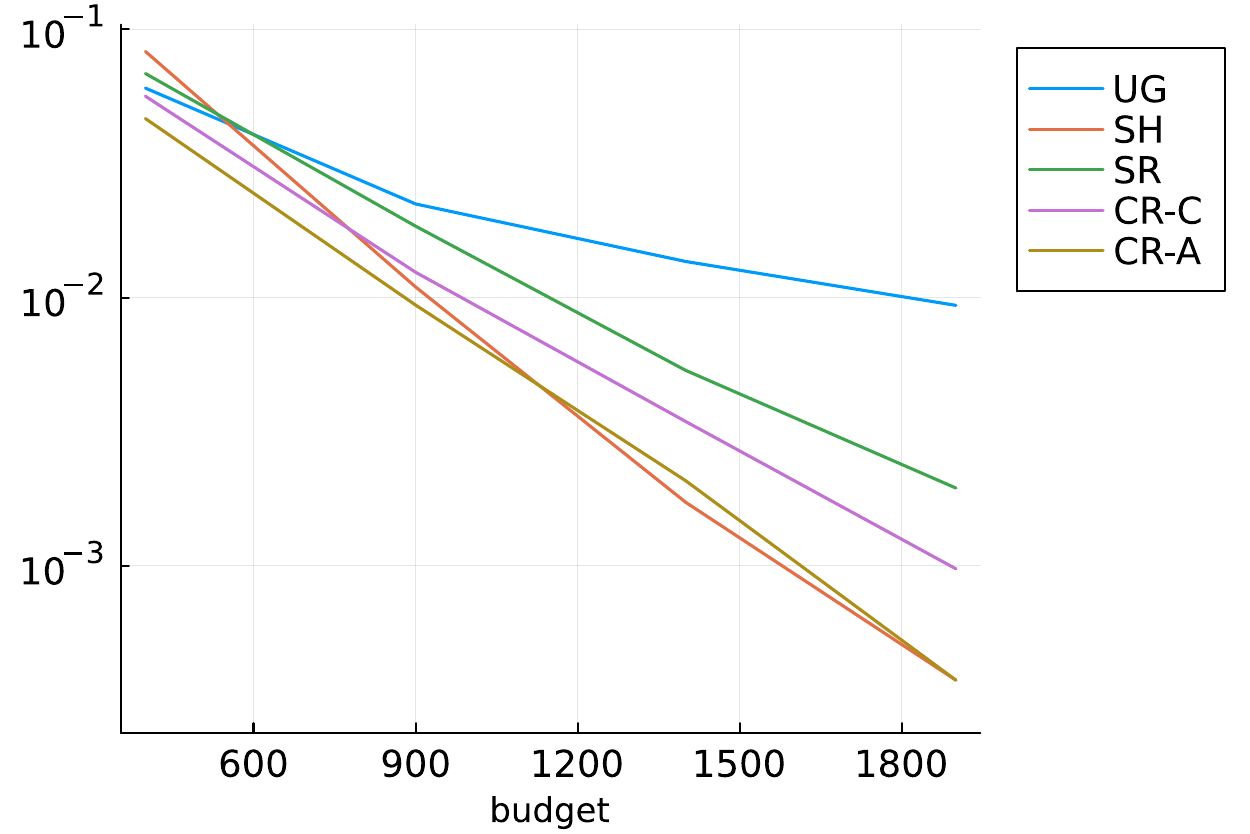}}
	\subcaptionbox{$K=20$ \label{fig:concave-K20}}{
		\includegraphics[width=0.48\textwidth]{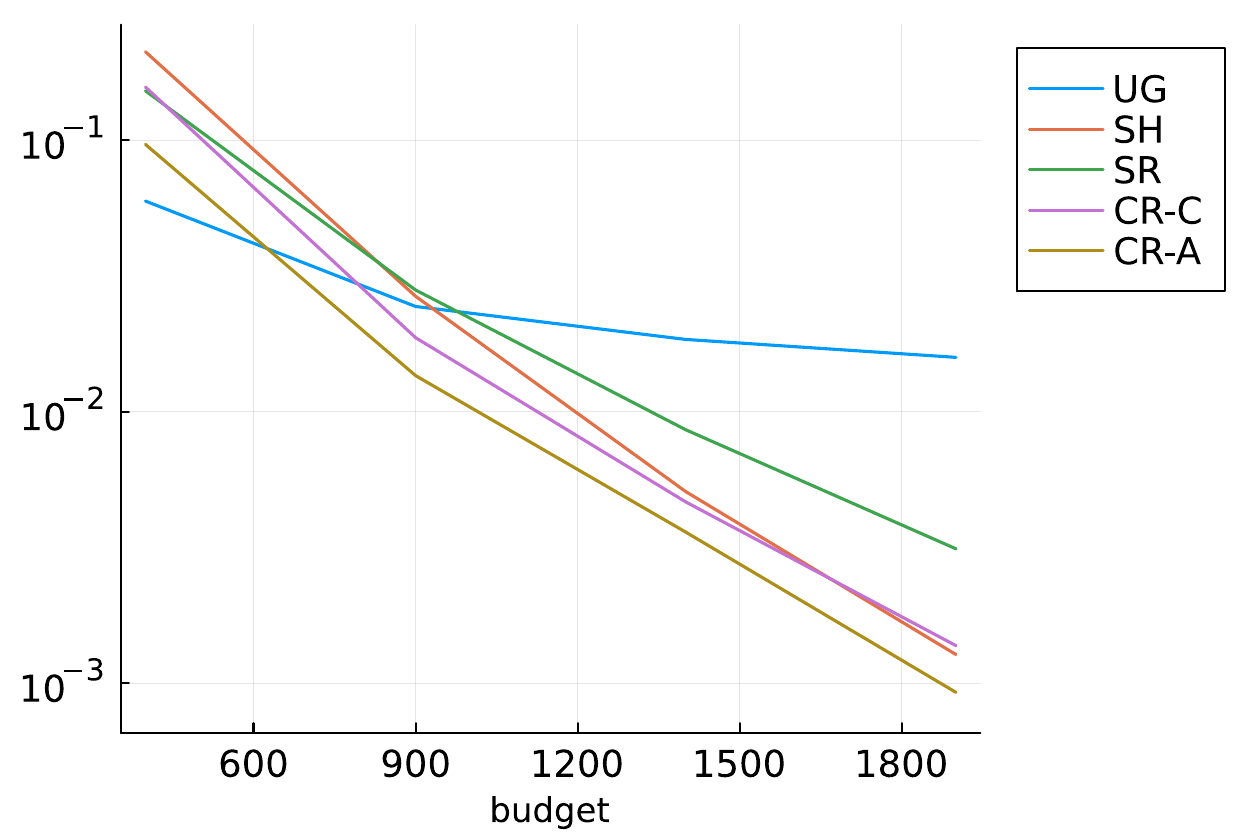}}
	\subcaptionbox{$K=40$ \label{fig:concave-K40}}{
		\includegraphics[width=0.48\textwidth]{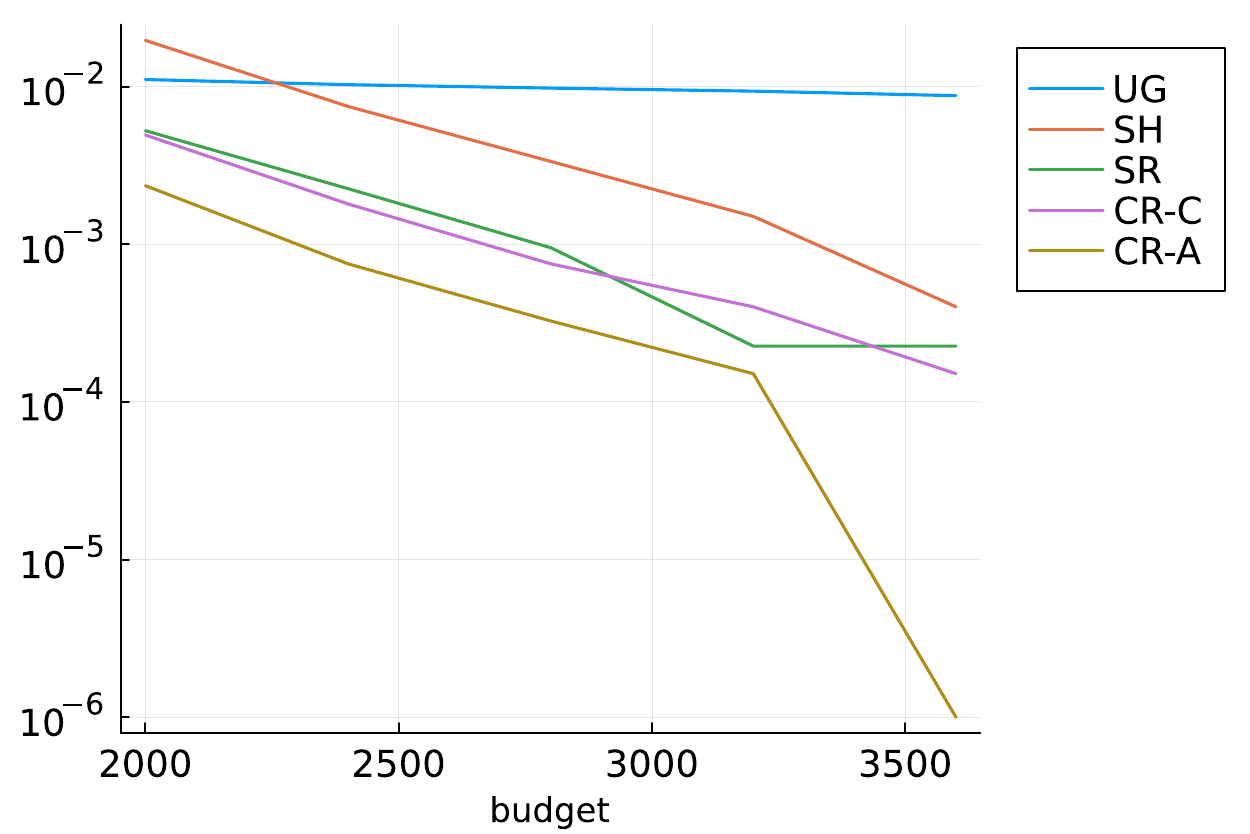}}
	\caption{(Concave arm-to-reward function) error probabilities averaged over $40,000$ independent runs.}
\end{figure}


\clearpage\pagebreak
\subsection{Convex arm-to-reward function}\label{num:5} 
In this instance, we set $\mu_k=\frac{3}{10(k+1)}$ for $k=1,\cdots,K$. Although \sr sometimes does better than \crc, \crc becomes better than \sr when there is more budget given. This confirms our theoretical analysis for \crc (see Theorem \ref{thm:Cj}).

\begin{figure}[h!]
	\centering
	\includegraphics[width=0.5\linewidth]{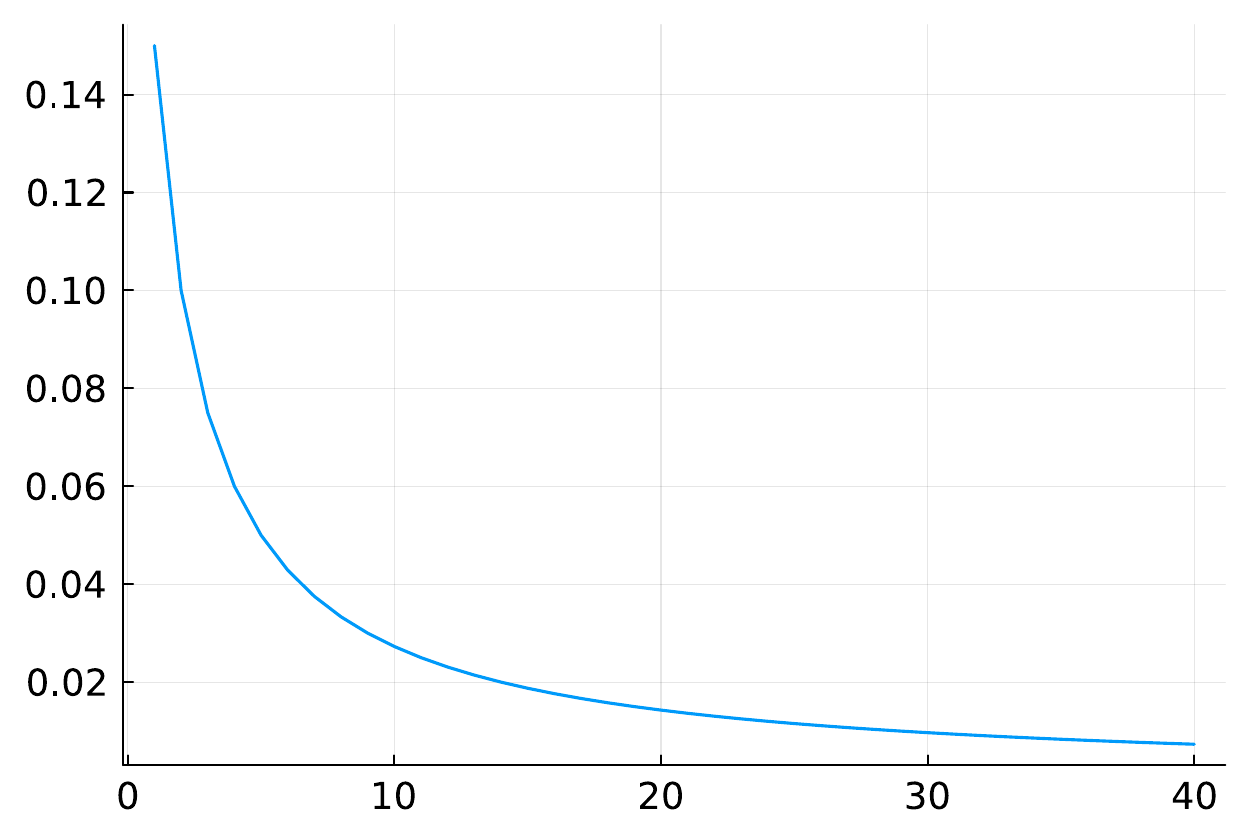}
	\caption{(Convex arm-to-reward function) $\bm$ with $K=40$.}\label{fig:convex-mu}
\end{figure}

\begin{table}[htb!]
\caption{(Convex arm-to-reward function) error probability (in \%).
\label{tab:convex-mu}}
\centering
\begin{tabular}{lllccc} \toprule
  $K=10$ & & & $T=1,500$ & $T=2,000$ & $T=2,500$\\\midrule
  & \ugape & \citep{gabillon2012best} & $13.48$ & $10.08$ & $7.77$\\\hline
  & \sh &\cite{karnin2013almost} & $10.20$ & $5.94$ & $3.84$\\\hline
  & \sr &\cite{audibert2010best} & $3.15$ & $1.45$ & $0.83$\\\hline
  & \crc &(this paper) & $3.12$ & $1.47$ & $0.65$\\\hline
  & \cra &(this paper) & $\mathbf{2.99}$ & $\mathbf{1.27}$ & $\mathbf{0.62}$\\\bottomrule
\end{tabular}
\begin{tabular}{lllccc} \toprule
  $K=20$ & & & $T=3,000$ & $T=3,500$ & $T=4,000$\\\midrule
  & \ugape & \citep{gabillon2012best} & $10.67$ & $8.95$ & $7.76$ \\\hline
  & \sh &\cite{karnin2013almost} & $6.60$ & $4.56$ & $2.78$ \\\hline
  & \sr &\cite{audibert2010best} & $0.96$ & $0.55$ & $0.29$ \\\hline
  & \crc &(this paper) & $0.93$ & $0.56$ & $0.28$ \\\hline
  & \cra &(this paper) & $\mathbf{0.79}$ & $\mathbf{0.39}$ & $\mathbf{0.21}$ \\\bottomrule
\end{tabular}
\begin{tabular}{lllccc} \toprule
  $K=40$ & & & $T=4,400$ & $T=5,200$ & $T=6,000$\\\midrule
  & \ugape & \citep{gabillon2012best} & $12.16$ & $9.95$ & $8.28$\\\hline
  & \sh &\cite{karnin2013almost} & $7.51$ & $4.76$ & $2.96$\\\hline
  & \sr &\cite{audibert2010best} & $\mathbf{1.21}$ & $0.58$ & $0.30$\\\hline
  & \crc &(this paper) & $1.28$ & $\mathbf{0.49}$ & $\mathbf{0.25}$\\\hline
  & \cra &(this paper) & $1.34$ & $0.60$ & $0.29$\\\bottomrule
\end{tabular}
\end{table}

\begin{figure}[h!]
	\centering
	\subcaptionbox{$K=10$ \label{fig:convex-K10}}{
		\includegraphics[width=0.48\textwidth]{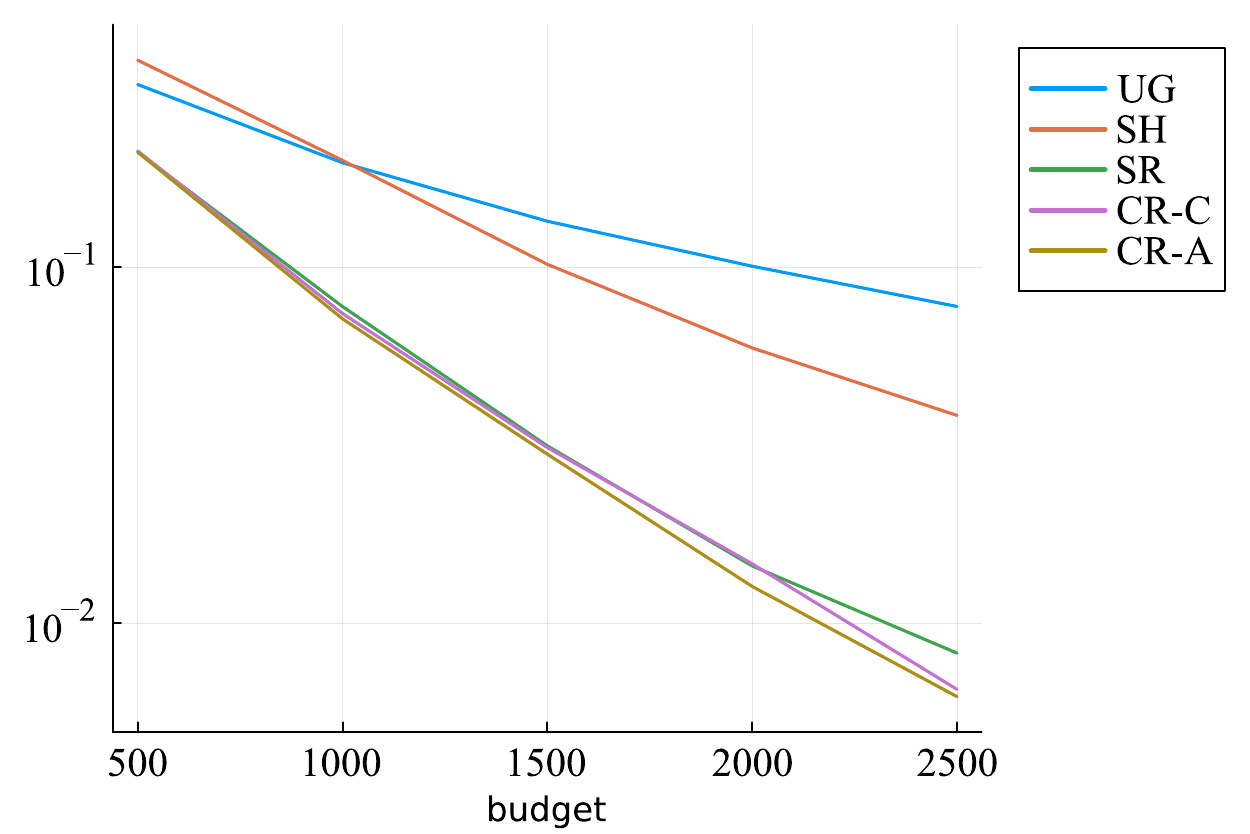}}
	\subcaptionbox{$K=20$ \label{fig:convex-K20}}{
		\includegraphics[width=0.48\textwidth]{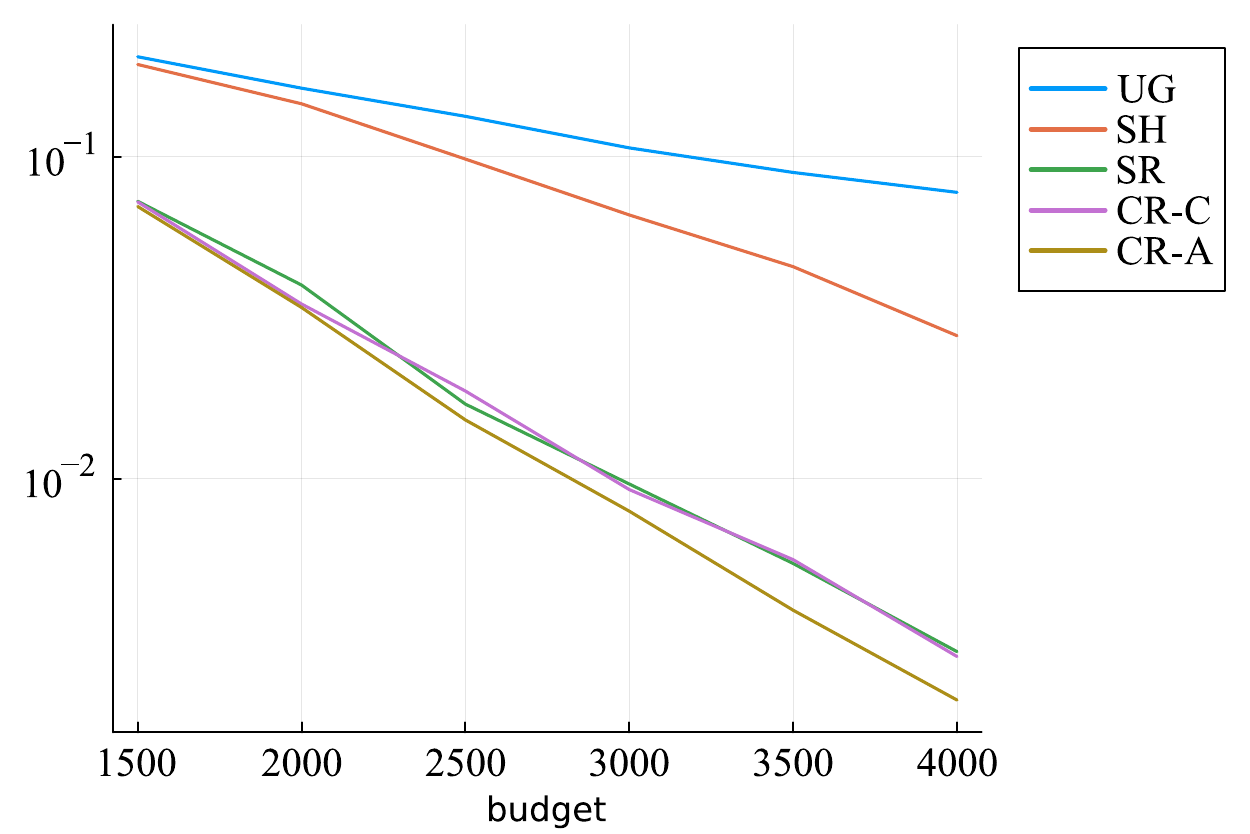}}
	\subcaptionbox{$K=40$ \label{fig:convex-K40}}{
		\includegraphics[width=0.48\textwidth]{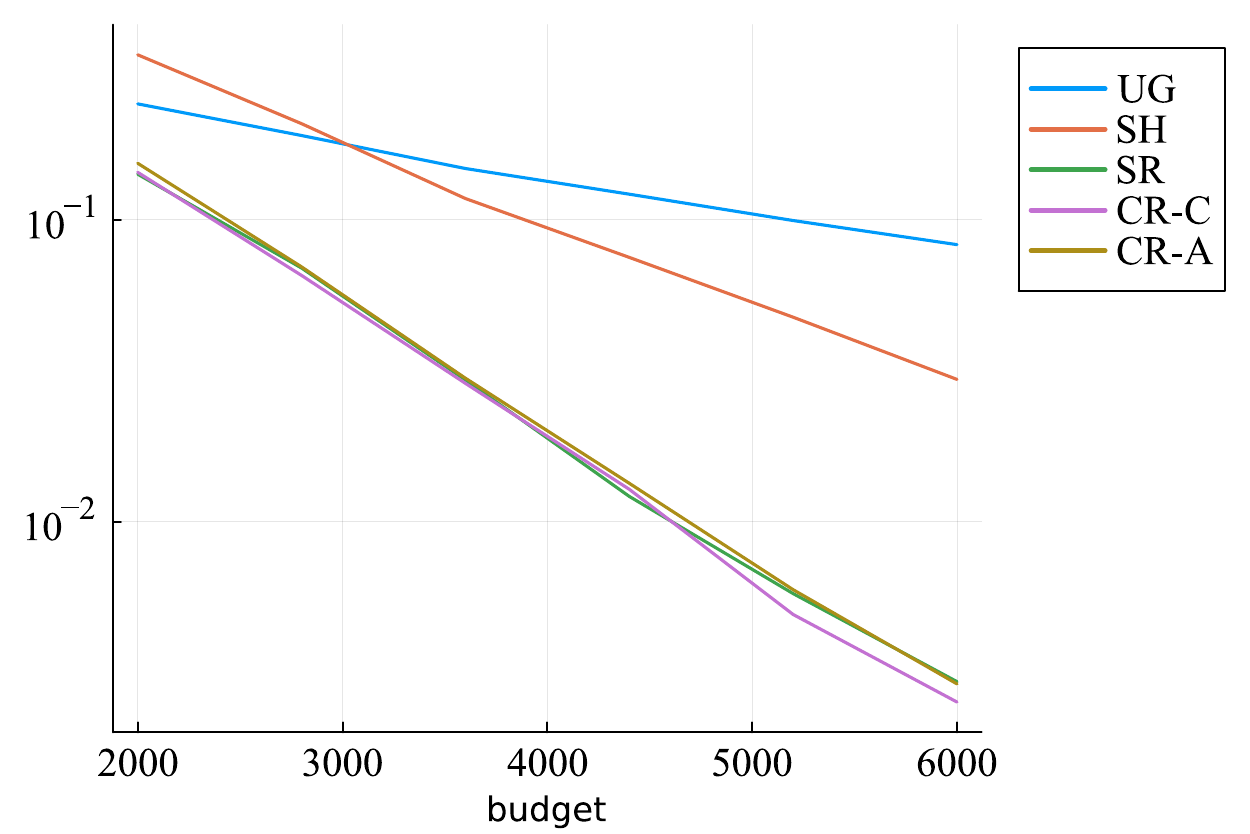}}
	\caption{(Convex arm-to-reward function) error probabilities averaged over $40,000$ independent runs.}
\end{figure}

\clearpage\pagebreak
\subsection{Stair arm-to-reward function}\label{num:6}
In this instance, we consider $M\in \{5,6,10\}$ and a $M(M+1)/2$-dimensional vector $\bm$. For each $M$, we define $\bm$ as: for all positive integers $m$ smaller than $ M$, there are $m$ arms on the same level with value, $\frac{3}{4}\cdot 3^{-\frac{m}{M}}$. For example, we plot the values for $M=10$ (hence $K=55$) in Figure \ref{fig:stair-mu}. One can see in this instance, our algorithms are by far stronger than the others.

\begin{figure}[h!]
    \centering
    \includegraphics[width=0.5\linewidth]{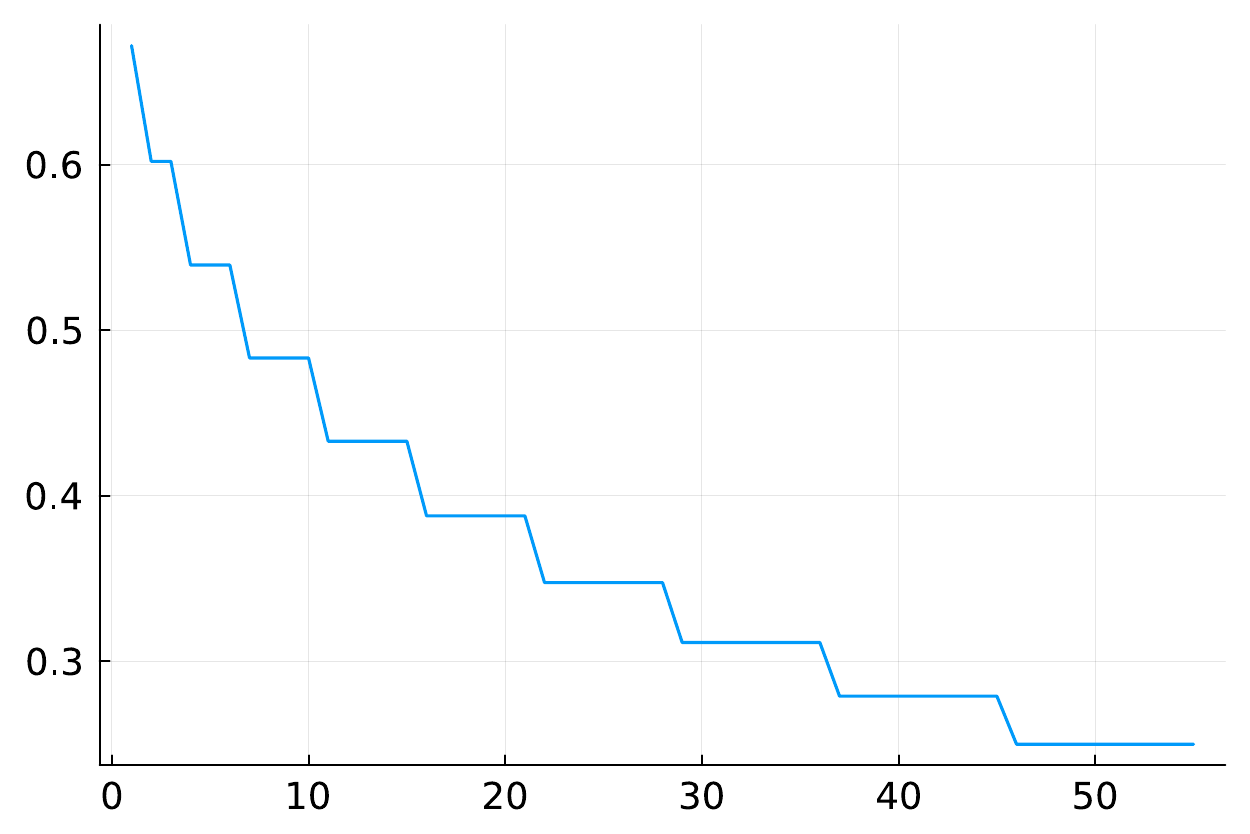}
    \caption{(Stair arm-to-reward function) $\bm$ with $K=55$.}\label{fig:stair-mu}
\end{figure}

\begin{table}[htb!]
\caption{(Stair arm-to-reward function) error probability (in \%).}\label{tab:stair-mu-detailed}
\centering
\begin{tabular}{lllccc} \toprule
  $K=15$ & & & $T=1,600$ & $T=2,000$ & $T=2,400$\\\midrule
  & \ugape & \citep{gabillon2012best} & $12.00$ & $9.89$ & $8.47$\\\hline
  & \sh &\cite{karnin2013almost} & $2.18$ & $1.07$ & $0.52$\\\hline
  & \sr &\cite{audibert2010best} & $0.43$ & $0.19$ & $0.06$\\\hline
  & \crc &(this paper) & $0.36$ & $0.10$ & $\mathbf{0.02}$\\\hline
  & \cra &(this paper) & $\mathbf{0.20}$ & $\mathbf{0.07}$ & $0.04$\\\bottomrule
\end{tabular}
\begin{tabular}{lllccc} \toprule
  $K=21$ & & & $T=1,500$ & $T=2,000$ & $T=2,500$\\\midrule
  & \ugape & \citep{gabillon2012best} & $17.44$ & $13.97$ & $12.08$\\\hline
  & \sh &\cite{karnin2013almost} & $7.78$ & $3.75$ & $1.78$\\\hline
  & \sr &\cite{audibert2010best} & $2.27$ & $0.92$ & $0.32$\\\hline
  & \crc &(this paper) & $1.68$ & $0.55$ & $0.24$\\\hline
  & \cra &(this paper) & $\mathbf{1.14}$ & $\mathbf{0.35}$ & $\mathbf{0.09}$\\\bottomrule
\end{tabular}
\begin{tabular}{lllccc} \toprule
  $K=55$ & & & $T=3,000$ & $T=4,000$ & $T=5,000$\\\midrule
  & \ugape & \citep{gabillon2012best} & $24.74$ & $21.29$ & $18.91$\\\hline
  & \sh &\cite{karnin2013almost} & $13.09$ & $7.76$ & $4.56$\\\hline
  & \sr &\cite{audibert2010best} & $5.55$ & $2.80$ & $1.26$\\\hline
  & \crc &(this paper) & $7.10$ & $2.58$ & $1.05$\\\hline
  & \cra &(this paper) & $\mathbf{4.70}$ & $\mathbf{1.62}$ & $\mathbf{0.57}$\\\bottomrule
\end{tabular}
\end{table}

\begin{figure}[h!]
    \centering
    \subcaptionbox{$K=15$ \label{fig:stair-K15}}{
	\includegraphics[width=0.48\textwidth]{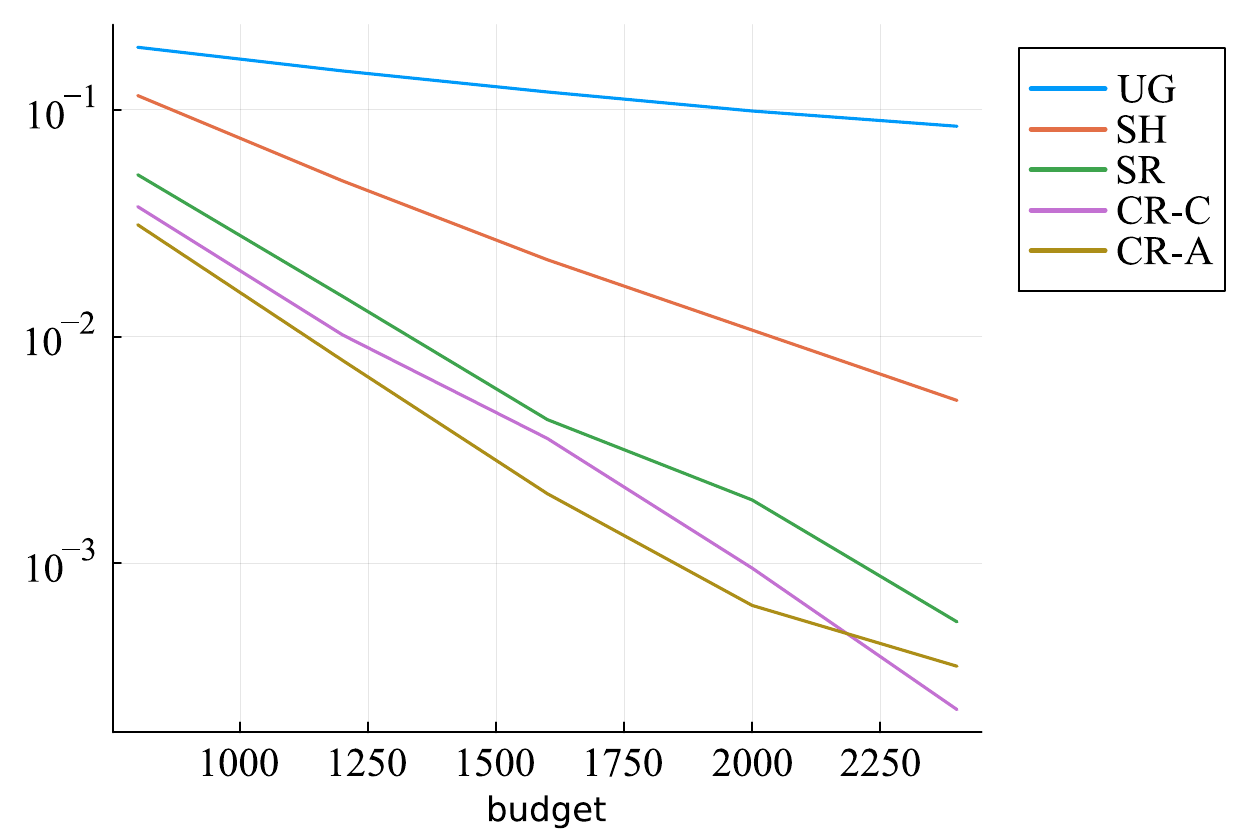}}
    \subcaptionbox{$K=21$ \label{fig:stair-K21}}{
	\includegraphics[width=0.48\textwidth]{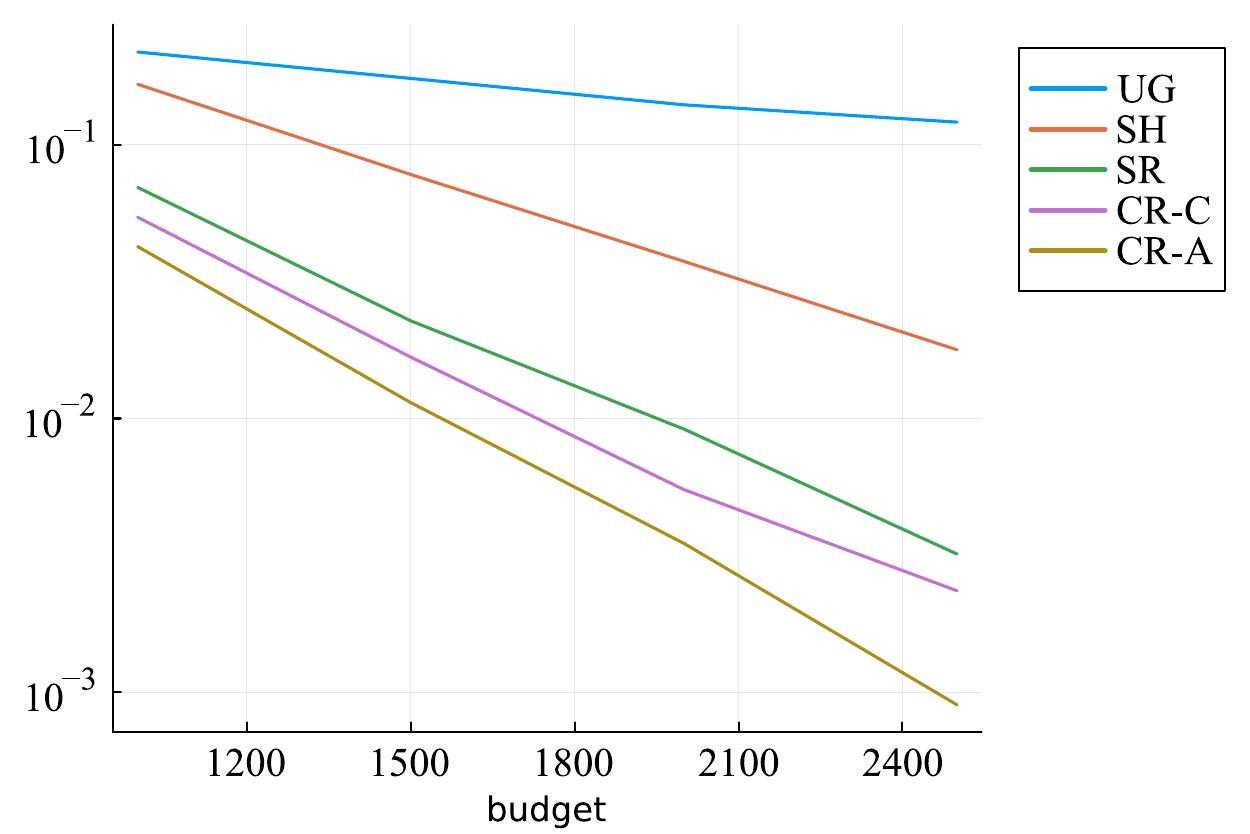}}
    \subcaptionbox{$K=55$ }{
	\includegraphics[width=0.48\textwidth]{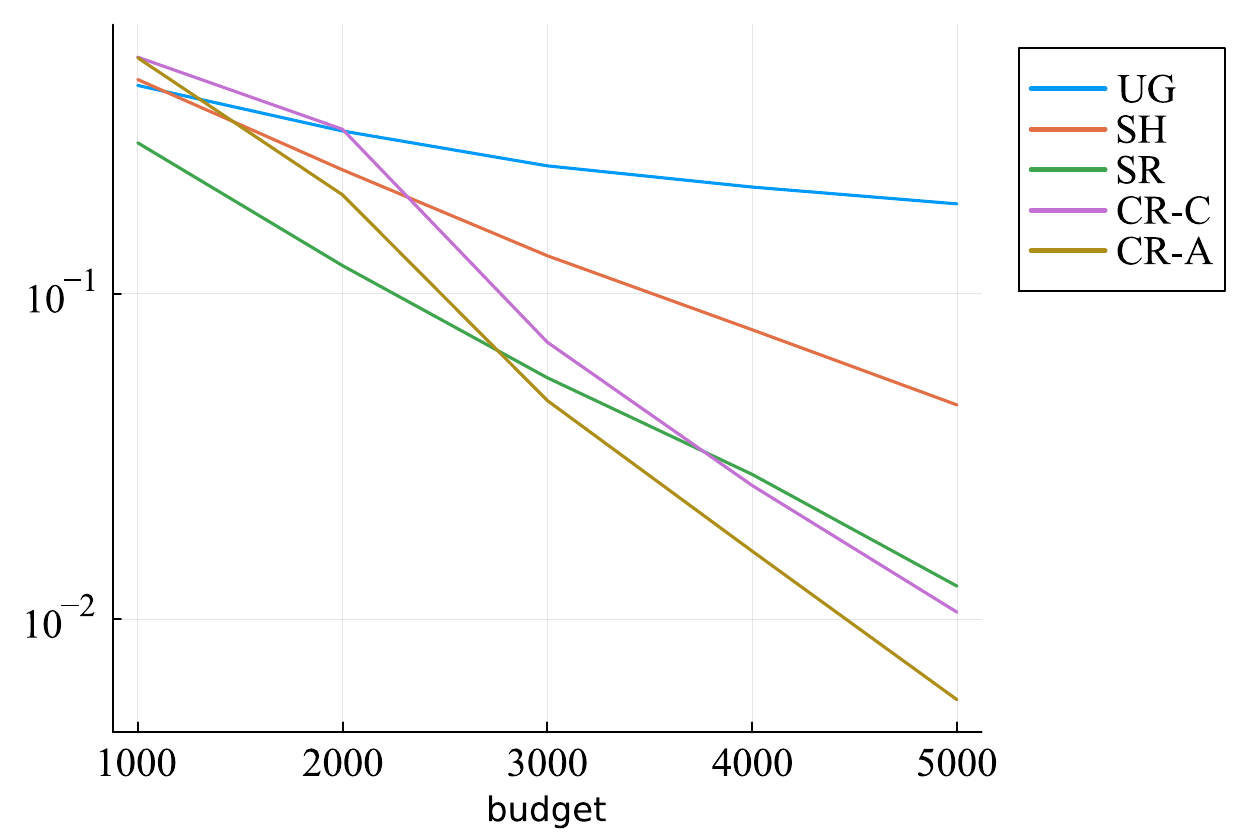}}
    \caption{(Stair arm-to-reward function) error probabilities averaged over $40,000$ independent runs.}
\end{figure}


\end{document}